%% file: main.tex
\def\isarxiv{1} 

\ifdefined\isarxiv
\documentclass[11pt]{article}

\usepackage[numbers]{natbib}

\else
\documentclass{article}
\usepackage{iclr2024_conference}
\usepackage{times} 
\iclrfinalcopy
\fi

\usepackage{amsmath}
\ifdefined\isarxiv
\usepackage{amsthm} 
\else
\fi
\usepackage{amsmath}
\usepackage{amsthm}
\usepackage{amssymb}

\usepackage{algorithm}
\usepackage{subfig}
\usepackage{algpseudocode}
\usepackage{graphicx}
\usepackage{grffile}
\usepackage{wrapfig,epsfig}
\usepackage{url}
\usepackage{xcolor}
\usepackage{color}
\usepackage{epstopdf}

\usepackage{bbm}
\usepackage{dsfont}
\allowdisplaybreaks

\ifdefined\isarxiv

\usepackage{tikz}
\usepackage{hyperref}  
\hypersetup{colorlinks=true,citecolor=blue,linkcolor=blue} 
\usetikzlibrary{arrows}
\usepackage[margin=1in]{geometry}

\else

\usepackage{microtype}
\usepackage{hyperref}
\definecolor{mydarkblue}{rgb}{0,0.08,0.45}
\hypersetup{colorlinks=true, citecolor=mydarkblue,linkcolor=mydarkblue}

\fi

\ifdefined\isarxiv
\newtheorem{theorem}{Theorem}[section]
\newtheorem{lemma}[theorem]{Lemma}
\newtheorem{definition}[theorem]{Definition}
\newtheorem{notation}[theorem]{Notation}
\newtheorem{proposition}[theorem]{Proposition}
\newtheorem{corollary}[theorem]{Corollary}
\newtheorem{conjecture}[theorem]{Conjecture}
\newtheorem{assumption}[theorem]{Assumption}
\newtheorem{observation}[theorem]{Observation}
\newtheorem{fact}[theorem]{Fact}
\newtheorem{remark}[theorem]{Remark}
\newtheorem{claim}[theorem]{Claim}
\newtheorem{example}[theorem]{Example}
\newtheorem{problem}[theorem]{Problem}
\newtheorem{open}[theorem]{Open Problem}
\newtheorem{property}[theorem]{Property}
\newtheorem{hypothesis}[theorem]{Hypothesis}
\else
\newtheorem{theorem}{Theorem}[section]
\newtheorem{lemma}[theorem]{Lemma}
\newtheorem{definition}[theorem]{Definition}

\newtheorem{assumption}[theorem]{Assumption}

\newtheorem{fact}[theorem]{Fact}
\newtheorem{remark}[theorem]{Remark}
\newtheorem{claim}[theorem]{Claim}

\newtheorem{problem}[theorem]{Problem}

\fi

\newcommand{\wh}{\widehat}
\newcommand{\wt}{\widetilde}
\newcommand{\ov}{\overline}

\newcommand{\R}{\mathbb{R}}

\newcommand{\ose}{\mathrm{ose}}

\DeclareMathOperator*{\E}{{\mathbb{E}}}
\DeclareMathOperator*{\var}{\mathrm{Var}}

\DeclareMathOperator{\OPT}{OPT}

\DeclareMathOperator{\poly}{poly}

\DeclareMathOperator{\nnz}{nnz}

\DeclareMathOperator{\dist}{dist}

\DeclareMathOperator{\vect}{vec}

\makeatletter
\newcommand*{\RN}[1]{\expandafter\@slowromancap\romannumeral #1@}
\makeatother

\usepackage{lineno}

\title{Low Rank Matrix Completion via Robust Alternating Minimization in Nearly Linear Time}

\begin{document}

\ifdefined\isarxiv

\date{}

\author{
Yuzhou Gu\thanks{\texttt{yuzhougu@ias.edu}. Institute of Advanced Study.}
\and
Zhao Song\thanks{\texttt{zsong@adobe.com}. Adobe Research.}
\and 
Junze Yin\thanks{\texttt{junze@bu.edu}. Boston University}
\and 
Lichen Zhang\thanks{\texttt{lichenz@mit.edu}. Massachusetts Institute of Technology.}
}

\else
\author{Yuzhou Gu \\
Institute of Advanced Study \\
\texttt{yuzhougu@ias.edu}
\And
Zhao Song \\
Adobe Research \\
\texttt{zsong@adobe.com}
\And 
Junze Yin \\
Boston University \\
\texttt{junze@bu.edu}
\And 
Lichen Zhang \\
MIT CSAIL \\
\texttt{lichenz@csail.mit.edu}
}

\maketitle 
\fi

\ifdefined\isarxiv
\begin{titlepage}
  \maketitle
  \begin{abstract}
\input{abstract}

  \end{abstract}
  \thispagestyle{empty}
\end{titlepage}

{\hypersetup{linkcolor=black}
\tableofcontents
}
\newpage

\else

\begin{abstract}
\input{abstract}
\end{abstract}

\fi


\input{intro} 
\input{related}
\input{preli_intro}
\input{tech}
\input{conclusion}

\input{ack}

\ifdefined\isarxiv
\bibliographystyle{alpha}
\bibliography{ref}
\else
\bibliographystyle{iclr2024_conference}
\bibliography{ref}

\fi

\newpage
\onecolumn
\appendix
\section*{Appendix}


\input{app_preli}

\input{sketch}

\input{init}

\input{induction}
\input{rankk}

\input{stronger_core}




\end{document}

%% file: abstract.tex
Given a matrix $M\in \mathbb{R}^{m\times n}$, the low rank matrix completion problem asks us to find a rank-$k$ approximation of $M$ as $UV^\top$ for $U\in \mathbb{R}^{m\times k}$ and $V\in \mathbb{R}^{n\times k}$ by only observing a few entries specified by a set of entries $\Omega\subseteq [m]\times [n]$. In particular, we examine an approach that is widely used in practice --- the alternating minimization framework. Jain, Netrapalli, and Sanghavi~\cite{jns13} showed that if $M$ has incoherent rows and columns, then alternating minimization provably recovers the matrix $M$ by observing a nearly linear in $n$ number of entries. While the sample complexity has been subsequently improved~\cite{glz17}, alternating minimization steps are required to be computed exactly. This hinders the development of more efficient algorithms and fails to depict the practical implementation of alternating minimization, where the updates are usually performed approximately in favor of efficiency.

In this paper, we take a major step towards a more efficient and error-robust alternating minimization framework. To this end, we develop an analytical framework for alternating minimization that can tolerate a moderate amount of errors caused by approximate updates. Moreover, our algorithm runs in time $\widetilde O(|\Omega| k)$, which is nearly linear in the time to verify the solution while preserving the sample complexity. This improves upon all prior known alternating minimization approaches which require $\widetilde O(|\Omega| k^2)$ time.

%% file: intro.tex
\section{Introduction}

Matrix completion is a well-studied problem both in theory and practice of computer science and machine learning. Given a matrix $M\in \R^{m\times n}$, the matrix completion problem asks to recover the matrix by observing only a few (random) entries of $M$. This problem originally appears in the context of collaborative filtering~\cite{rs05} with the most notable example being the Netflix Challenge. Since then, it has found various applications in signal processing~\cite{llr95,sy05} and traffic engineering~\cite{gc12}. In theory, one requires additional structural assumptions on the matrix $M$ in order to obtain provable guarantees, with the most natural and practical assumption being the matrix $M$ is a rank-$k$ low rank matrix. In tasks such as collaborative filtering, the matrix one faces is often low rank~\cite{cr12}. Another popular assumption is the $M$ has incoherent rows and columns. This intuitively eliminates the degenerate case where $M$ only has a few large entries, making it imperative to observe them. Matrices in practice are usually incoherent as well~\cite{amt10,mt11,kmt12}.

Under these assumptions, a variety of algorithms based on convex relaxation have been derived~\cite{ct10,r11,cr12}. These algorithms relax the problem as a trace-norm minimization problem that can be solved with semidefinite programs (SDP). Unfortunately, solving SDP is inherently slow and rarely used in practice: even the state-of-the-art SDP solver would require $O(n^\omega)$ time\footnote{$\omega$ is the exponent of matrix multiplication. Currently, $\omega<2.258$ over every field of characteristic zero~\cite[Theorem~1.1 and Corollaries~14.1--14.2]{OpenAIMM26}.} to solve the program~\cite{jklps20,hjstz21}. Heuristics such as alternating minimization and gradient descent are much preferred due to their efficiency.~\cite{jns13} first showed that under the standard low rank and incoherent settings, as long as we are allowed to observe $\wt O(\kappa^4 \mu^2 n k^{4.5}\log(1/\epsilon))$ entries\footnote{In this paper, we assume $n\geq m$ and we use $\wt O(\cdot)$ to hide polylogarithmic factors in $n$, $k$, $\kappa$, $\|M\|_F$ and $1/\epsilon$.} (where $\kappa$ is the condition number of matrix $M$ and $\mu$ is its incoherence parameter from Definition~\ref{def:mu_k_incoherent}), then alternating minimization provably recovers the matrix $M$ up to $\epsilon$ error in terms of Frobenius norm. Subsequent works further unify different approaches by treating matrix completion as a non-convex optimization problem~\cite{zwl15} and improve the sample complexity~\cite{glz17}. A remarkable advantage of alternating minimization is its efficiency, as each iteration, the two alternating updates can be implemented in $O(|\Omega| k^2)$ time. 

Despite various advantages, the series of theoretical papers analyzing alternating minimization fail to capture a crucial component of many practical implementations as the updates are usually computed approximately to further speed up the process. In contrast, the analysis pioneered in~\cite{jns13} and followups~\cite{zwl15,glz17,llr16} crucially relies on the exact formulation of the updates, making it difficult to adapt and generalize to the approximate updates setting. On the other hand, developing an alternating minimization framework would also enable us to utilize faster, approximate solvers to implement updates that leads to \emph{theoretical speedup} of the algorithm.

In this paper, we take the first major step towards an error-robust\footnote{We note that in the context of matrix completion, robustness is often used with respect to noise --- i.e., we observe a noisy, higher rank matrix $M=M^*+N$, where $M^*$ is the low rank ground truth and $N$ is the higher rank noise matrix. In this paper, we focus on the setting where the ground truth is noiseless, and we use robust referring to robustness against the errors caused by computing updates approximately.} analysis for alternating minimization. Specifically, we show that the alternating updates can be formulated as two multiple response regressions and fast, high accuracy solvers can be utilized to solve them in nearly linear time. Coupling our robust alternating minimization framework with our faster multiple response regression solver, we derive an algorithm that solves the matrix completion problem in time $\wt O(|\Omega|k)$. We note that this runtime is nearly linear in terms of the time to \emph{verify the solution}: Given matrices $U\in \R^{m\times k}$ and $V\in \R^{n\times k}$, it takes $O(k)$ time to verify a single entry and a total of $O(|\Omega| k)$ time to verify all entries in $\Omega$. To the best of our knowledge, this is \emph{the first algorithm based on alternating minimization, that achieves a nearly linear time in verification.} We remark that a recent work~\cite{kllst23} achieves a similar runtime behavior of $\wt O(|\Omega|k)$: their algorithm has an improved sample complexity of $|\Omega|=\wt O(nk^{2+o(1)})$ and runtime of $\wt O(nk^{3+o(1)})$. However, their algorithm heavily relies on the tools developed in~\cite{kllst22} and mainly serves as a proof-of-concept for matrix completion rather than efficient practical implementation.\footnote{We note another work that also claims to obtain a runtime of $\wt O(|\Omega|k)$~\cite{cgj17}, but as pointed out in~\cite{kllst23} their runtime is only achievable if certain linear algebraic operations can be performed exactly in sublinear time, which is unknown to this day.} In contrast, our algorithm adopts an off the shelf regression solver that is fast and high accuracy with good practical performances~\cite{amt10,msm14}. 

We state an informal version of our main result as follows:

\begin{theorem}[Informal version of Theorem~\ref{thm:main}]
\label{thm:main_informal}
Let $M\in \R^{m\times n}$ be a matrix that is rank-$k$, is $(\mu,k)$-incoherent (Definition~\ref{def:mu_k_incoherent}) and entries can be sampled independently. Then, there exists a randomized algorithm that takes the incoherence parameter $\mu$, the condition number $\kappa$ of $M$ and $\|M\|_F$, or upper bounds on them, as input, samples $|\Omega|=O(\kappa^4\mu^2 nk^{4.5}\log n\log(k\|M\|_F/\epsilon))$ entries, and with high probability, outputs a pair of matrices $\wh U\in \R^{m\times k}, \wh V\in \R^{n\times k}$ such that
\begin{align*}
    \|M-\wh U\wh V^\top\|_F \leq & ~ \epsilon,
\end{align*}
and the algorithm runs in time $\wt O(|\Omega| k)$.
\end{theorem}

\paragraph{Roadmap.}

In Section~\ref{sec:related_work}, we introduce the related works which include matrix completion and applying sketching matrices to solve optimization problems. In Section~\ref{sec:prelim}, we present the notations, definitions, and lemmas that we use for the later sections. In Section~\ref{sec:technical_overview}, we present the significant findings and discuss the techniques utilized. In Section~\ref{sec:conclusion}, we provide some concluding remarks and future directions.

%% file: related.tex
\section{Related Work}
\label{sec:related_work}

Low rank matrix completion is a fundamental problem in machine learning. Practical applications including recommender systems, with the most notable one being collaborative filtering~\cite{rs05} and the Netflix Challenge \cite{k09,kbv09}. It also has a wide range of applications in computer visions~\cite{cr12} and signal processing~\cite{llr95,sy05,clmw11}. For more comprehensive surveys, we refer readers to~\cite{j90,nks19}. Algorithms for matrix completion can be roughly divided into two categories --- convex relaxation and non-convex heuristics. Candes and Recht~\cite{cr12} prove the first sample complexity for low rank matrix completion under convex relaxation and the bound is subsequently improved by~\cite{ct10}. In practice, heuristic methods based on non-convex optimizations are often preferred due to their simplicity and efficiency. One notable approach is gradient descent~\cite{kom09,zwl15}. Alternating minimization is also a popular alternative that is widely applied in practice.~\cite{jns13} provides a provable guarantee on the convergence of alternating minimization when the matrix-to-recover is low rank and incoherent. Subsequently, there have been a long line of works analyzing the performance of non-convex heuristics under standard matrix completion setting and under the noise or corrupted-entries setting~\cite{gagg13,h14,hmrw14,hw14,jn15,sl16,glz17,cgj17}. Recently, the work~\cite{kllst23} provides an algorithm with improved sample complexity and runtime when stronger subspace regularity assumptions are imposed. 

In addition to the high accuracy randomized solver we utilize in this paper, sketching has a variety of applications in numerical linear algebra and machine learning, such as linear regression, low rank approximation~\cite{cw13,nn13}, matrix CUR decomposition \cite{bw14,swz17,swz19}, weighted low rank approximation \cite{rsw16}, entrywise $\ell_1$ norm low-rank approximation \cite{swz17,swz19_neurips_l1}, general norm column subset selection \cite{swz19_neurips_general}, tensor low-rank approximation \cite{swz19}, and tensor regression \cite{dssw18,djs+19,rsz22,swyz21}. 

A recent trend in the sketching community is to apply them as a means to reduce the iteration cost of optimization algorithms. Applications such as linear programming~\cite{sy21}, empirical risk minimization~\cite{lsz19,qszz23}, approximating the John ellipsoid~\cite{ccly19}, Frank-Wolfe algorithm~\cite{xss21} and linear MDPs~\cite{ssx23}.

\paragraph{Subsequent work.} After our work, Abbasi and Vaswani~\cite{av25} analyzed the alternating minimization update of~\cite{jns13} with fresh samples, the same update we analyze, started from their own initialization. In the same model, they need $O(\kappa^4\mu^2 nk^2\log n\log(1/\epsilon))$ samples, where $\kappa$ is the condition number of $M$ and $\mu$ is its incoherence (Definition~\ref{def:mu_k_incoherent}). This improves the factor $k^{4.5}$ in Theorem~\ref{thm:main} to $k^2$ and partially answers our open question on its sample complexity. Their updates are exact and take $O(|\Omega|k^2)$ time per iteration, where $\Omega$ is the set of observed entries. Subsequently, Cai, Lu and You~\cite{cly26} showed in a recent preprint that Riemannian gradient descent recovers a square $n\times n$ matrix exactly from $O(\mu^2 nk\log n\log(n\kappa))$ samples, in our normalization of $\mu$, and reaches relative error $\epsilon$ in time $O(|\Omega|k(\log n+\log(1/\epsilon)))$. This removes a factor $\kappa^4k^{3.5}$ and the dependence on $\epsilon$ from our sample complexity, and the runtime is also nearly linear in the verification time. Finally,~\cite{kllst24} studies the harder semi-random model, where each entry is revealed with an unknown probability at least $p$. It allows bounded entrywise noise and rectangular matrices, the latter by a reduction to the symmetric case. It runs in nearly linear time with no dependence on $\kappa$, but with a larger polynomial dependence on $k$ in the sample complexity.

%% file: preli_intro.tex
\section{Preliminary}
\label{sec:prelim}
In this section, we provide necessary background on matrix completion and assumptions.

\paragraph{Notations.}

We use $[n]$ for a positive integer $n$ to denote the set $\{1,2,\ldots,n\}$. Given a vector $x$, we use $\|x\|_2$ to denote its $\ell_2$ norm. Given a matrix $A$, we use $\|A\|$ to denote its spectral norm and $\|A\|_F$ to denote its Frobenius norm. Given a matrix $A$, we use $A_{i,*}$ to denote its $i$-th row and $A_{*,j}$ to denote its $j$-th column. We use $\|A\|_0$ to denote the $\ell_0$ semi-norm of $A$ that measures the number of nonzero entries in $A$.

Given an orthonormal basis $A\in \R^{n\times k}$, we use $A_\perp\in \R^{n\times (n-k)}$ to denote an orthonormal basis to $A$'s orthogonal complement, such that $AA^\top+A_\perp A_\perp^\top=I_n$.

Given a rank-$k$, $m\times n$ real matrix $A$, we use $\kappa(A)$ to denote its condition number: $\kappa(A)=\frac{\sigma_1(A)}{\sigma_k(A)}$ where $\sigma_1(A),\ldots,\sigma_k(A)$ are singular values of $A$ sorted in magnitude. When $A$ is clear from context, we often use $\kappa$ directly.

\subsection{Angles and Distances Between Subspaces}
A key quantity that captures the closeness of the subspaces by two matrices is the distance, defined as follows:

\begin{definition}[Distance between general matrices]\label{def:dist}

    Given two matrices $\wh{U}, \wh{W} \in \R^{m \times k}$, the (principal angle) distance between the subspaces spanned by the columns of $\wh{U}$ and $\wh{W}$ is given by:
    \begin{align*}
        \dist(\wh{U}, \wh{W}) := \|U_\bot^\top W\| = \|W_\bot^\top U\|
    \end{align*}
    where $U$ and $W$ are orthonormal bases of the spaces $\mathrm{span}(\wh{U})$ and $\mathrm{span}(\wh{W})$, respectively. Similarly, $U_\bot$ and $W_\bot$ are any orthonormal bases of the orthogonal spaces $\mathrm{span}(U)^\bot$ and $\mathrm{span}(W)^\bot$, respectively.
\end{definition}

One can further quantify the geometry of two subspaces by their principal angles:

\begin{definition}\label{def:angle_and_distance}
Let $U \in \R^{n \times k}$ and $V \in \R^{n \times k}$.

For any matrix $U$, and for orthogonal matrix $V$ ($V^\top V = I_k$) we define
\begin{itemize}
    \item $\tan \theta(V,U) := \| V_{\bot}^\top U ( V^\top U )^{-1} \|$.
\end{itemize}
For orthogonal matrices $V$ and $U$ ($V^\top V = I_k$ and $U^\top U = I_k$), we define
\begin{itemize}
    \item $\cos \theta (V,U) := \sigma_{\min} (V^\top U)$. 
    \begin{itemize} 
        \item $\cos \theta (V,U) = 1/ \| (V^\top U)^{-1} \|$ and $\cos \theta (V,U) \leq 1$.
    \end{itemize}
    \item $\sin \theta(V,U) := \| (I - V V^\top) U \|$.
    \begin{itemize} 
        \item $\sin \theta(V,U) = \| V_{\bot} V_{\bot}^\top U \| = \| V_{\bot}^\top U \|$ and $\sin \theta(V,U) \leq 1$.
    \end{itemize}
\end{itemize}
\end{definition}

Standard trigonometry equalities and inequalities hold for above definitions, see, e.g.,~\cite{zk13}. 

\subsection{Background on Matrix Completion}

We introduce necessary definitions and assumptions for low rank matrix completion. Given a set of indices $\Omega$, we define a linear operator that selects the corresponding entries:

\begin{definition}\label{def:operator_P_Omega}
Let $S \in \R^{m \times n}$ denote a matrix. Let $\Omega \subset [m] \times [n]$.
We define operator $P_{\Omega}$ as follows:
    \begin{align*}
        P_\Omega(S)_{i,j} = 
        \begin{cases}
        S_{i,j}, & \mathrm{if} ~(i,j)\in\Omega; \\
        0, & \mathrm{otherwise}.
        \end{cases}
    \end{align*}
\end{definition}

A crucial assumption for low rank matrix completion is incoherence. We start by defining the notion below.

\begin{definition}\label{def:U_is_mu}
For an orthonormal basis $U \in \R^{m \times k}$, we say $U$ is $\mu$-incoherent, if
\begin{align*}
\| u_i \|_2 \leq \frac{\sqrt{k}}{\sqrt{m}} \cdot \mu.
\end{align*}
Here $u_i$ is the $i$-th row of $U$, for all $i \in [m]$. 
\end{definition}

We say a general rank-$k$ matrix is incoherent if both its row and column spans are incoherent:

\begin{definition}[$(\mu,k)$-incoherent]\label{def:mu_k_incoherent}
     A rank-$k$ matrix $M \in \R^{m \times n}$ is $(\mu, k)$-incoherent if: 
    \begin{itemize}
        \item $\|u_i \|_2 \leq \frac{\sqrt{k}}{\sqrt{m}} \cdot \mu$, $ \forall i\in [m]$
        \item $\|v_j \|_2 \leq \frac{\sqrt{k}}{\sqrt{n}} \cdot \mu$, $ \forall j\in [n]$
    \end{itemize} 
     where $M = U\Sigma V^\top$ is the SVD of M and $u_i$, $v_j$ denote the $i$-th row of $U$ and the $j$-th row of $V$ respectively. 
\end{definition}

We are now in the position to state the set of assumptions for low rank matrix completion problem. 

\begin{assumption}\label{assump:matrix_completion}
Let $M\in \R^{m\times n}$ and we assume $M$ satisfies the following property:
\begin{itemize}
    \item $M$ is a rank-$k$ matrix;
    \item $M$ is $(\mu, k)$-incoherent;
    \item Each entry of $M$ is sampled independently with equal probability. In other words, the set $\Omega \subseteq [m]\times [n]$ is formed by sampling each entry independently according to a certain distribution. 
\end{itemize}
\end{assumption}

We note that the above assumptions are all standard in the literature.

Now we are ready to state the low rank matrix completion problem.

\begin{problem}
Let $M\in \R^{m\times n}$ satisfy Assumption~\ref{assump:matrix_completion}. The low rank matrix completion problem asks to find a pair of matrices $U\in \R^{m\times k}$ and $V\in \R^{n\times k}$ such that
\begin{align*}
    \|M-UV^\top\|_F \leq & ~ \epsilon.
\end{align*}
\end{problem}

%% file: tech.tex
\section{Technical Overview}
\label{sec:technical_overview}

Before diving into the details of our algorithm and analysis, let us review the alternating minimization proposed in~\cite{jns13}. The algorithm can be described pretty succinctly: given the sampled indices $\Omega$, the algorithm starts by partitioning $\Omega$ into $2T+1$ groups, denoted by $\Omega_0,\ldots,\Omega_{2T}$. The algorithm first computes a top-$k$ SVD of the matrix $\frac{1}{p}P_{\Omega_0}(M)$ where $p$ is the sampling probability. It then proceeds to trim all rows of the left singular matrix $U$ with large row norms (this step is often referred to as clipping). It then optimizes the factors $U$ and $V$ alternatively. At iteration $t$, the algorithm first fixes $U$ and solves for $V$ with a multiple response regression using entries in $\Omega_{t+1}$, then it fixes the newly obtained $V$ and solves for $U$ with entries in $\Omega_{T+t+1}$. After iterating over all groups in the partition, the algorithm outputs the final factors $U$ and $V$ as its output. Here, we use independent samples across different iterations to ensure that each iteration is independent of priors (in terms of randomness used). Since the parts of a partition are not independent, our algorithm draws these samples as $2T+1$ independent subsamples of $\Omega$, each containing every entry independently with probability $p/(2T+1)$ (Algorithm~\ref{alg:alternating_minimization_completion}). If one would like to drop the uniform and independent samples across iterations (see, e.g.,~\cite{lly17}), the convergence has only been shown under additional assumptions and in terms of critical points to certain non-convex program.

From a runtime complexity perspective, the most expensive steps are solving two multiple response regressions per iteration, which would take a total of $O(|\Omega| k^2)$ time. On the other hand, if the sum of the size of all the regressions across $T$ iterations is only $O(|\Omega| k)$, inspiring us to consider efficient, randomized solvers that can solve the regression in time nearly linear in the instance size.\footnote{We note that our solver would incur an extra total cost of $\wt O(nk^3)$. Under the standard low rank and incoherent assumptions, the state-of-the-art non-convex optimization-based approach~\cite{kllst23} would require $\wt O(nk^{2+o(1)})$ samples and the $\wt O(nk^3)$ time is subsumed by the $\wt O(|\Omega|k)$ portion. This term is hence ignored.} This in turn produces an algorithm with a nearly linear runtime in terms of verifying all entries in $\Omega$. 

Two problems remain to complete our algorithm: 1).\ What kind of multiple response regression solvers will run in nearly linear time and produce high accuracy solutions and 2).\ How to prove the convergence of the alternating minimization in the presence of the errors caused by approximate solvers. While the second question seems rather intuitive that the~\cite{jns13} algorithm should tolerate a moderate amount of errors, the analysis becomes highly nontrivial when one examines the argument in~\cite{jns13} and followups, as all of them crucially rely on the fact that the factors $U$ and $V$ are solved exactly and hence admit closed form. This motivates us to develop an analytical framework for alternating minimization, that admits approximate solutions for each iteration. 

Our argument is inductive in nature. Let $M=U^*(V^*)^\top$ be optimal rank-$k$ factors, we build up our induction by assuming the approximate regression solutions $\wh U_{t-1}, \wh V_{t-1}$ are close to $U^*, V^*$ and proving for $\wh U_t, \wh V_t$. The loss of structure due to solving the regressions approximately forces us to find alternatives to bound the closeness, as the argument of~\cite{jns13} relies on the exact formulation of the closed-form solution. Our strategy is to instead consider a sequence of imaginary, exact solutions to the multiple response regressions and prove that our approximate solutions are close to those exact solutions. In particular, we show that if $\wh U_t, \wh V_t$ are close to their exact counterpart in the spectral norm sense, then it suffices to prove additional structural conditions on $U_t$ and $V_t$ to progress the induction. Specifically, we show that as long as $U_{t-1}$ is incoherent and $\dist(\wh U_{t-1},U^*)\leq 1/2$, we will have $\dist(\wh V_t, V^*)\leq \max\{\frac{1}{2}\dist(\wh U_{t-1}, U^*),\, 2\epsilon_{\mathrm{sk}}\}$, where $\epsilon_{\mathrm{sk}}$ is the error contributed by the approximate solve, and the \emph{exact solution at time $t$}, $V_t$ is incoherent. We can then progress the induction by using the small distance between $\wh V_t$ and $V^*$ in conjunction with the incoherence of $V_t$ to advance on $\dist(\wh U_t, U^*)$ and the incoherence of $U_t$. As the gap decreases geometrically, this leads to an overall $O(\log(k\|M\|_F/\epsilon))$ iterations to converge, as desired.

It remains to develop a perturbation theory for alternating minimization and our key innovation is a perturbation argument on the incoherence bound under small spectral norm perturbation. We show that if two matrices $A$ and $B$ satisfying $\|A-B\|\leq \epsilon$ and $\epsilon\leq \frac{1}{2}\sigma_{\min}(A)$, then the change of incoherence between $A$ and $B$ is also very small. Our main strategy involves reducing incoherence to statistical leverage score, a critical numerical quantity that measures the importance of rows~\cite{ss11}. The problem then reduces to showing that if two matrices are close in spectral norm, their leverage scores are also similar. By utilizing the leverage score formulation and a perturbation bound for the orthogonal projections onto the two column spans, we prove that it is indeed the case they have similar leverage scores. This concludes the perturbation argument on matrix incoherence.

However, this alone is insufficient for our inductive argument to progress, as we must also quantify the proximity between the approximate and exact solutions. To accomplish this, we employ a novel approach based on the condition number of the matrices. Specifically, given the exact matrix $A$ and the approximate matrix $B$, we demonstrate that the distance between $A$ and $B$ is relatively small as long as $\kappa(B)$ is bounded by $\kappa(A)$. To obtain a good bound on $\kappa(B)$, we need to solve the approximate regression to high precision (inverse polynomial proportionally to $\kappa(A)$). We achieve this objective through a sketching-based preconditioner, which is the first problem we would like to address.

Sketching-based preconditioning is a widely used and lightweight approach to solving regression problems with high accuracy. The fundamental concept involves reducing the dimensionality of the regression problem by utilizing a sketching matrix and creating an approximate preconditioner based on the sketch. This preconditioner accelerates the convergence of an iterative solver, such as the conjugate gradient or Generalized Minimal RESidual method (GMRES), employed in the reduced system. With this preconditioner, we obtain an $O(\log(1/\epsilon))$ convergence for an $\epsilon$-approximate solution. Beyond its theoretical soundness and efficiency, these preconditioners also exhibit good performances when used in practice~\cite{amt10,msm14}.

An important feature of our robust alternating minimization framework coupled with the fast regression solvers is that we preserve the sample complexity of~\cite{jns13}: by picking the sampling probability $p=O(\kappa^4  \mu^2 k^{4.5} \log n\log(k\|M\|_F/\epsilon)/m)$, we are required to sample $|\Omega|=O(\kappa^4  \mu^2  nk^{4.5} \log n\log(k\|M\|_F/\epsilon))$ entries. By a clever thresholding approach,~\cite{glz17} further improves the sample complexity to $O(\kappa^2 \mu^4nk^4 \log(1/\epsilon))$ samples in our normalization of $\mu$. Their bound assumes an extra condition, which they call the probabilistic generalized restricted isometry condition (their Assumption 2), and a different sampling model, a union of random regular bipartite graphs for square $M$. Their error guarantee is the relative error $\epsilon\|M\|_F$, and their analysis treats $k$, $\kappa$ and $\mu$ as constants as $n\to\infty$. Under incoherence alone, the degree of each of their random regular bipartite graphs needs an extra additive term $O(\mu^2k\log(1/\epsilon))$ in our normalization of $\mu$, so they use $O((\kappa^2\mu^4k^4+\mu^2k\log(1/\epsilon))\,n\log(1/\epsilon))$ samples. In both cases their algorithm runs for $O(\log(1/\epsilon))$ rounds. Their algorithm is also alternating minimization in nature. We expect that our robust framework and nearly linear time regression solvers could be combined with it, which would reduce the sample complexity and runtime by a factor of about $\kappa^2 k^{0.5}\log n/\mu^2$ when $\mu^2<\kappa^2k^{0.5}\log n$. We do not prove this here, because their algorithm thresholds the singular values in the least-squares problem of a row whenever that problem is ill-conditioned, and rescales rows of large norm after each iteration, and our analysis does not cover these steps.

In the rest of this section, we will review our approaches and clarify how we employ various strategies to achieve our goal.

    \begin{algorithm}[!ht]
\caption{Alternating minimization for matrix completion. The \textsc{Init} procedure clips the rows with large norms, then performs a Gram-Schmidt process. Here $\ov U_t$ and $\ov V_{t+1}$ denote orthonormal factors. In the text of this section, for $t\geq 1$, $U_t$ and $V_t$ denote the exact solutions of the regressions that $\wh U_t$ and $\wh V_t$ approximate, and $U_0$ denotes the orthonormal factor of $\wh U_0$. Sections~\ref{sec:main_induction} and~\ref{sec:update_rules_and_notations} use $U_t$ and $V_{t+1}$ for orthonormal factors.} \label{alg:alternating_minimization_completion} 
\begin{algorithmic}[1]
    \Procedure{FastMatrixCompletion}{$\Omega \subset [m] \times [n]$, $P_\Omega(M)$}  \Comment{Theorem~\ref{thm:main}}\label{line:completion_procdure_begin}
    \State \Comment{Draw independent $\Omega_0, \cdots, \Omega_{2T} \subseteq \Omega$, each containing every entry independently with probability $p_t := p/(2T+1)$} \label{line:completion_partition_omega}
    \State \Comment{This is possible since $1-(1-p_t)^{2T+1} \leq p$, and the sets may overlap} \label{line:completion_element_of_omega} 
    \State $U_{\phi} = \mathrm{SVD}(\frac{1}{p} P_{\Omega_0} (M), k)$ \Comment{computed approximately by subspace iteration}
    \label{line:completion_widehatU}
    \State $\wh{U}_0 \gets \textsc{Init}(U_\phi)$ \Comment{Algorithm~\ref{alg:init}}
    \label{line:completion_clipping}

    \For{$t = 0, \cdots, T-1$} \label{line:completion_for_begin}
        \State $\ov U_t \gets$ orthonormal factor of $\wh U_t$ (QR)
        \State Obtain $\wh{V}_{t + 1} $ with Lemma~\ref{lem:weighted_regression_standard_error_informal} by solving  $\min_{V \in \R^{n \times k}} \|P_{\Omega_{t + 1}} (\ov U_t V^\top - M)\|_F^2$ \label{line:completion_Vt+1_gets}
        \State $\ov V_{t+1} \gets$ orthonormal factor of $\wh V_{t+1}$ (QR)
        \State Obtain $\wh{U}_{t + 1}$ with Lemma~\ref{lem:weighted_regression_standard_error_informal} by solving $ \min_{U \in \R^{m \times k}} \|P_{\Omega_{T + t + 1}} (U \ov V_{t + 1}^\top - M)\|_F^2$ \label{line:completion_Ut+1_gets}
    \EndFor \label{line:completion_for_end}
    \State \Return $\wh U_T \ov V_T^\top$
    \label{line:completion_return}
\EndProcedure \label{line:completion_procedure_end}
\end{algorithmic}
\end{algorithm}

\subsection{Subspace Approaching Argument}
\label{sub:technical_overview:induction_hypothesis}

The basis of the analysis of both~\cite{jns13} and ours is an inductive argument that shows the low rank factors obtained in each iteration approach the optimum as their distance shrinks by a constant factor and we call this \emph{subspace approaching argument}, as we iteratively refine subspaces that approach the optimum. The major difference is that the induction of~\cite{jns13} can be based on the exact solutions solely --- they inductively prove the distances shrink by a constant factor, and the low rank factors are $\mu_2$-incoherent for $\mu_2$ to be specified. It is tempting to replace these exact solutions with the approximate solutions we obtain. Unfortunately, the~\cite{jns13} heavily exploits the closed-form formulation of the solutions so that they can easily decompose the matrices into terms whose spectral norms can be bounded in a straightforward manner. Our approximate solutions on the other hand cannot be factored in such a fashion. 

To circumvent this issue, we instead keep the exact solutions $U_t, V_t$ as a sequence of references and inductively bound the incoherence of these exact solutions. More specifically, for any $t\in\{0,\ldots,T-1\}$, we assume $U_t$ is $(\mu_2/2)$-incoherent and $\dist(\wh U_t,U^*)\leq 1/2$, then we show that
\begin{itemize}
    \item $\dist(\wh V_{t+1},V^*)\leq \max\{\frac{1}{2}\dist(\wh U_t,U^*),\, 2\epsilon_{\mathrm{sk}}\}\leq 1/2$, where $\epsilon_{\mathrm{sk}}$ is the per-step perturbation term of Lemma~\ref{lem:induction_control_eps_sketch} (the exact solution $V_{t+1}$ of the regression with design $\ov U_t$ satisfies $\dist(V_{t+1},V^*)\leq \frac{1}{4}\dist(\wh U_t,U^*)$), and
    \item $V_{t+1}$ is $(\mu_2/2)$-incoherent.
\end{itemize}

We can then proceed to show similar conditions hold for $\wh U_{t+1}$ and $U_{t+1}$, therefore advance the induction. At first glance, our induction conditions seem rather far-off from what we want. The regression for $\wh U_{t+1}$ uses $\ov V_{t+1}$, the orthonormal factor of $\wh V_{t+1}$, as its design, so to show that $\wh U_{t+1}$ and $U^*$ are close, one would need to show that $\wh V_{t+1}$ has small incoherence, while our induction only controls the incoherence of the exact solution $V_{t+1}$. We alternatively adapt a perturbation argument on incoherence, specifically, note that $\wh V_{t+1}$ can be treated as blending in a small perturbation to $V_{t+1}$ in spectral norm, if we can manage to prove under such perturbations, the incoherence of $\wh V_{t+1}$ is still close to the incoherence of $V_{t+1}$, then our induction condition on the incoherence of $V_{t+1}$ effectively translates to the incoherence of $\wh V_{t+1}$. In the next section, we illustrate how to develop such a perturbation theory.

\subsection{A Perturbation Theory for Matrix Incoherence}
\label{sub:technical_overview:robust_perturbation_error}

Our main goal is to understand how perturbations to rows of a matrix change the incoherence. Let us consider two matrices $A, B\in \R^{m\times k}$ that are close in the spectral norm: $\|A-B\|\leq \epsilon_0$ for some small error $\epsilon_0$, as the closeness in spectral norm implies \emph{all} row norms of the difference matrix $A-B$ are small. Our next step will be putting the rows of $A$ and $B$ to isotropic position\footnote{We say a collection of vectors $\{v_1,\ldots,v_m\}\subseteq \R^k$ are in isotropic position if $\sum_{i=1}^m v_iv_i^\top=I_k$.}, by mapping $a_i\mapsto (A^\top A)^{-1/2}a_i$ and $b_i\mapsto (B^\top B)^{-1/2}b_i$. The main motivation is the equivalence between matrix incoherence and statistical leverage score~\cite{ss11}, which measures the $\ell_2$ importance of rows by putting the matrix into isotropic position. This provides us with an alternative way to measure the matrix incoherence directly instead of working with the singular value decomposition. It remains to bound the difference $|\|(A^\top A)^{-1/2}a_i\|_2-\|(B^\top B)^{-1/2}b_i\|_2 |$, without loss of generality, assume $\|(A^\top A)^{-1/2}a_i\|_2\geq \|(B^\top B)^{-1/2}b_i\|_2$, then we can equivalently express it using difference of squares formula: 
\begin{align*}
    \|(A^\top A)^{-1/2}a_i\|_2-\|(B^\top B)^{-1/2}b_i\|_2= & ~\frac{\|(A^\top A)^{-1/2}a_i\|_2^2-\|(B^\top B)^{-1/2}b_i\|_2^2}{\|(A^\top A)^{-1/2}a_i\|_2+\|(B^\top B)^{-1/2}b_i\|_2}
\end{align*}
where the numerator is the difference between leverage scores. A single row can have a tiny leverage score, so the denominator has no useful lower bound in general. Instead, we use the identity $\|(A^\top A)^{-1/2}a_i\|_2=\|P_Ae_i\|_2$, where $P_A$ is the orthogonal projection onto the column span of $A$, and likewise for $B$. Hence the difference is at most $\|(P_A-P_B)e_i\|_2\leq \|P_A-P_B\|$. A short perturbation argument for the two projections, in the spirit of the perturbation bounds for pseudo-inverses in~\cite{w73}, bounds $\|P_A-P_B\|$ by $3\epsilon_0/\sigma_{\min}(A)$. 

Combining these bounds, we conclude that
\begin{align*}
    |\|(A^\top A)^{-1/2}a_i\|_2-\|(B^\top B)^{-1/2}b_i\|_2 | \leq & ~ O(\epsilon_0 \sigma_{\min}(A)^{-1} \kappa(A)^4).
\end{align*}
To put this bound into perspective, we will set $A$ to be the exact update matrix $U$ and $B$ to be the approximate update matrix $\wh U$. In our induction framework, we further make sure that the condition number of $U$ is well-controlled, therefore we can safely suffer error up to $\poly(\kappa(U))$, as we can set our final precision to be inverse polynomially in $\kappa(U)$. Since our algorithm converges in $O(\log(k\|M\|_F/\epsilon))$ iterations, and our regression solver has an error dependence $\log(1/\epsilon)$, this only blows up the runtime by a factor of $O(\log(n\kappa k\|M\|_F/\epsilon))$, as desired.

\subsection{Nearly Linear Time Solve via Sketching}
\label{sub:technical_overview:sketching_to_speedup}

Now that we have a framework that tolerates errors caused by approximate solves, we are ready to deploy regression solvers that can actually compute the factors in nearly linear time. A popular and standard approach is through the so-called sketch-and-solve paradigm --- given an overconstrained regression problem $\min_x \|Ax-b\|_2$ with $A\in \R^{n\times k}$ and $n\gg k$, one samples a random sketching matrix $S\in \R^{m\times n}$ with $m\ll n$ and instead solves the sketched regression $\|SAx-Sb\|_2$. As the new regression problem has much fewer row count, solving it would take nearly linear time as long as $SA$ can be quickly applied. In addition to its simplicity, efficiency and good error guarantees, if one picks $S$ to be a dense subsampled randomized Hadamard transform ({\sf SRHT})~\cite{ldfu13}, then the regression \emph{solution} is also close to the exact solution in an $\ell_\infty$ sense~\cite{psw17}. While appealing, sketch-and-solve has the deficiency of \emph{low accuracy} --- if one wants a solution with regression cost at most $(1+\epsilon)$ times the optimal cost, then the algorithm would take time proportional to $\epsilon^{-2}$. As we would set the $\epsilon$ to be polynomially small to incorporate the error blowups due to perturbations to incoherence and the error conversion from regression cost to solution, an inverse polynomial dependence on $\epsilon$ would significantly slow down the algorithm and ultimately lead to an even slower runtime than solving the regressions exactly.

We resort to a sketching-based approach that enjoys high accuracy as the solver converges in $\log(1/\epsilon)$ iterations and each iteration can be implemented in $\wt O(nk+k^3)$ time. The idea is to pick a sketching matrix that produces a subspace embedding of $O(1)$-distortion\footnote{We say a matrix $S$ produces a subspace embedding of $\epsilon$-distortion if for any fixed orthonormal basis $U\in \R^{n\times k}$, the singular values of $SU$ lie in $[1-\epsilon,1+\epsilon]$ with high probability.} and compute a QR decomposition of matrix $SA=QR^{-1}$. The matrix $R$ is then used as a preconditioner and we solve the preconditioned regression $\min_x\|ARx-b\|_2$ via gradient descent. 

Even though the solver runs in nearly linear time in the instance size, we still need to make sure that the \emph{sparsity} of $\Omega$ is utilized, as $|\Omega|\ll mn$. To demonstrate how to achieve a runtime of $\wt O(|\Omega|k)$ instead of $\wt O(mnk)$, let us consider the following multiple response regression: $\min_{V\in \R^{n\times k}}\|UV^\top-M\|_F$. Our strategy is to solve $n$ least-square regression. Consider the $i$-th regression problem: $\min_v \|P_\Omega(Uv)-P_{\Omega}(M_{*,i})\|_2$, note that $v$ only contributes to the $i$-th column of $UV^\top$ which correlates to the $i$-th column of the operator $P_\Omega$, therefore we only need to solve a regression with the rows of $U$ selected by the $i$-th column of $P_\Omega$. Let ${\cal K}_i$ be the corresponding set of nonzero indices, we select the rows of $U$ which is of size $U_{{\cal K}_i,*}\in \R^{|{\cal K}_i|\times k}$ and the resulting regression is 
    $\min_{v\in \R^k} \|U_{{\cal K}_i,*}v-M_{{\cal K}_i,i}\|_2$. Sum over all $T$ iterations, the total size of the multiple response regression is $O(\sum_{i=1}^n|{\cal K}_i|k)=O(|\Omega|k)$, and our regression solver takes $\wt O(|\Omega |k+nk^3)=\wt O(|\Omega| k)$ time, as desired.

We summarize the result in the following lemma.

\begin{lemma}\label{lem:weighted_regression_standard_error_informal}
Let $\Omega\subseteq [m]\times [n]$, $M\in \R^{m\times n}$ and $U\in \R^{m\times k}$. Let $\epsilon,\delta\in (0,1)$ be precision and failure probability, respectively. For $j\in[n]$, let $\Omega_{*,j}:=\{i\in[m]:(i,j)\in\Omega\}$, assume that each $U_{\Omega_{*,j},*}$ has full column rank, and let $\kappa_0 := \max\{1, \|M\| / \min_{j\in[n]} \sigma_{\min}(U_{\Omega_{*,j},*})\}$. There exists an algorithm that takes 
\begin{align*}
    O((|\Omega|k\log n+nk^3\log^2(n/\delta))\cdot \log(n\kappa_0/\epsilon))
\end{align*}
time to output a matrix $\wh V$ such that $\|\wh V-V^*\|\leq \epsilon$ where $V^*=\arg\min_{V\in \R^{n\times k}}\|P_\Omega(UV^\top)-P_\Omega(M)\|_F$. The algorithm succeeds with probability at least $1-\delta$.
\end{lemma}

\subsection{Putting Things Together}
\label{sub:technical_overview:our_robust_error}

Given our fast regression solver and robust analytical framework that effectively handles the perturbation error caused by approximate solve, we are in the position to deliver our main result.

\begin{theorem}[Main result, formal version of Theorem~\ref{thm:main_informal}]
\label{thm:main}
Let $M\in \R^{m\times n}$ be a matrix satisfying Assumption~\ref{assump:matrix_completion}. Let operator $P_\Omega$ be defined as in Definition~\ref{def:operator_P_Omega}. Then, there exists a randomized algorithm (Algorithm~\ref{alg:alternating_minimization_completion}) with the following properties:
\begin{itemize}
    \item It samples $|\Omega|=O(\kappa^4 \mu^2 nk^{4.5}\log n\log(k\|M\|_F/\epsilon))$ entries;
    \item It runs in $\wt O(|\Omega| k) $ time.
\end{itemize}
The algorithm outputs a pair of matrices $\wh U\in\R^{m\times k}$ and $\wh V\in \R^{n\times k}$ such that
\begin{align*}
    \|M-\wh U\wh V^\top\|_F \leq & ~ \epsilon
\end{align*}
holds with high probability.
The algorithm takes the condition number $\kappa$ of $M$ and $\|M\|_F$, or upper bounds on them, as input, since its formal version (Algorithm~\ref{alg:alternating_minimization_completion_formal}) sets the number of iterations $T$ and the accuracy $\epsilon_0$ of its regression solves from them.
\end{theorem}

Let us pause and make some remarks regarding the above result. On the sample complexity front, we attain the result achieved in~\cite{jns13}, but we would also like to point out that we expect it can be further improved using the approach developed in~\cite{glz17}, although we do not prove this here. Our robust alternating minimization framework indicates that as long as the error caused by the approximate solver can be polynomially bounded, then the convergence of the algorithm is preserved (up to a factor of $O(\log(n\kappa k\|M\|_F/\epsilon))$ in the running time). We expect that our framework can be integrated into other alternating minimization-based algorithms for matrix completion, but we prove this only for the update of~\cite{jns13} with fresh samples.

\paragraph{Comparison with~\cite{kllst23}.} The runtime of our algorithm is nearly linear in the verification time --- given a set of observed entries $\Omega$, it takes $O(k)$ time to verify an entry of $P_\Omega(\wh U\wh V^\top)$, hence requires a total of $O(|\Omega|k)$ time.~\cite{kllst23} achieves a similar runtime behavior with an improved sample complexity of $|\Omega|=\wt O(nk^{2+o(1)})$. It is worth noting that most popular \emph{practical algorithms} for matrix completion are based on either alternating minimization or gradient descent, since they are easy to implement and certain steps can be sped up via fast solvers. In contrary, the machinery of~\cite{kllst23} is much more complicated. In short, they need to decompose the update into a ``short'' progress matrix and a ``flat'' noise component whose singular values are relatively close to each other. To achieve this goal, their algorithm requires complicated primitives, such as approximating singular values and spectral norms~\cite{mm15}, Nesterov's accelerated gradient descent~\cite{n83} (which is known to be hard to realize for practical applications) and a complicated post-process procedure. While it is totally possible that these subroutines can be made practically efficient, empirical studies seem to be necessary to justify its practical performance. In contrast, our algorithm could be interpreted as providing a theoretical foundation on why \emph{fast} alternating minimization works so well in practice. As most of the fast alternating minimization implementations rely on quick, approximate solvers (for instance,~\cite{lv17,lsb+20}) but most of their analyses assume every step of the algorithm is computed exactly. From this perspective, one can view our robust analytical framework as a step toward ``completing the picture'' for these variants of alternating minimization, although our proof covers only the update of~\cite{jns13} with fresh samples. Moreover, if one can further sharpen the dependence on $k$ and condition number $\kappa$ in the sample complexity for alternating minimization, matching that of~\cite{kllst23}, we expect to obtain an algorithm with the same (asymptotic) complexity of their algorithm. We leave improving the sample complexity of alternating minimization as a future direction.

%% file: conclusion.tex
\section{Conclusion}
\label{sec:conclusion}
In this paper, we develop a nearly linear time algorithm for low rank matrix completion via two ingredients: a sketching-based preconditioner for solving multiple response regressions, and a robust framework for alternating minimization. 

Our robust alternating minimization framework narrows the gap between theory and practice. All prior theoretical analyses of alternating minimization require exact solves of the multiple response regressions, but in practice, fast and cheap approximate solvers are preferred. Our analysis covers the update of~\cite{jns13} with fresh samples, and we expect it to extend to other variants. Our algorithm also has the feature that it runs in time nearly linear in verifying the solution, as given a set of entries $\Omega$, it takes $O(|\Omega|k)$ time to compute the corresponding entries.

%% file: ack.tex
\section*{Acknowledgement}

We would like to thank Jonathan Kelner for helpful discussions and pointing out many references, as well as comments from anonymous reviewers that significantly improve the presentation of this paper.
Yuzhou Gu is supported by the National Science Foundation under Grant No.~DMS-1926686.
Lichen Zhang is supported by NSF CCF-1955217 and NSF DMS-2022448.
The corrections in this version were found and written with the help of Claude Opus 5.5. The authors have verified the corrections and take full responsibility for the content of this paper.

%% file: app_preli.tex
\paragraph{Roadmap.}

In Section~\ref{sec:app_preli}, we introduce more fundamental lemmas and facts. In Section~\ref{sec:high_accuracy_solver}, we provide more details about our sketching-based solver. In Section~\ref{sec:init_conditions}, we support the main init by giving and explaining more definitions and lemmas. In Section~\ref{sec:main_induction}, we explain our main induction hypothesis. In Section~\ref{sec:update_rules_and_notations}, we analyze the general case of the update rules and notations. In Section~\ref{sec:distance_shrinking}, we give the lemmas for distance shrinking and their proofs: we upper bound different terms by distance. In Section~\ref{sec:purterbation}, we analyze the distance and incoherence under row perturbations.

\section{More Preliminary}\label{sec:app_preli}

In this section, we display more fundamental concepts. In Section~\ref{sub:app_preli:algebra}, we introduce algebraic properties for matrices. In Section~\ref{sub:app_preli:angle}, we analyze the properties of angles and distances. In Section~\ref{sub:app_preli:tools}, we provide the tools which are used in previous works. In Section~\ref{sub:app_preli:probability}, we state several well-known probability tools.

\subsection{Basic Matrix Algebra}
\label{sub:app_preli:algebra}

We state several standard norm inequalities here.
\begin{fact}[Norm inequalities]\label{fac:norm_inequality}
For any matrix $A,B$
\begin{itemize}
\item Part 1. $\| A B \| \leq \| A \| \cdot \| B \|$.
\item Part 2. If $B$ is square, $\| A B \| \geq \| A \| \cdot \sigma_{\min}(B)$.
\item Part 3. $\| A + B\| \leq \| A \| + \| B \|$.
\item Part 4. if $B^\top A = 0$, $\| A + B \| \geq \| A \|$ .
\end{itemize}
For any matrix $A$ and vector $x$
\begin{itemize}
    \item Part 5. $\| A x \|_2 \leq \| A \| \cdot \| x \|_2$.
    \item Part 6. $\| A \| \leq \| A \|_F$.
    \item Part 7. $\| A \|_F \leq \sqrt{k} \| A \|$ if $A$ is rank-$k$.
\end{itemize}
For any square matrix $A$
\begin{itemize}
    \item Part 8. If $A$ is invertible, we have $\| A \| = 1/\sigma_{\min}(A^{-1})$.
    \item Part 9. For vector $x$ with $\| x \|_2 = 1$, we have $\| A x \|_2 = \max_{y: \| y \|_2=1} y^\top A x$.
\end{itemize}
For any matrix $A$, square invertible matrix $B$
\begin{itemize}
    \item Part 10. $\| A \| \leq \sigma_{\max}(B) \cdot \| A B^{-1} \|$.
\end{itemize}
For any matrix $A$, and square diagonal invertible matrix $B$
\begin{itemize}
    \item Part 11. $\sigma_{\min}(B A) \geq \sigma_{\min}(B) \cdot \sigma_{\min}(A)$.
\end{itemize}
\end{fact}
We omit their proofs, as they are quite standard. We also state Weyl's inequality for singular values:
\begin{lemma}[\cite{w12}]
\label{lem:weyl}
Let $A, B\in \R^{n\times k}$ where $n\geq k$, then we have for any $i\in [k]$, 
\begin{align*}
    |\sigma_i(A)-\sigma_i(B)| \leq & ~ \|A-B\|.
\end{align*}
\end{lemma}

The next fact bounds the spectral norm of a matrix after applying a unitary transformation.

\begin{fact}\label{fac:U_shrink_spectral_norm}
Let $m \geq k$. 
Let $U \in \R^{m \times k}$ denote a matrix that has an orthonormal basis. 
\begin{itemize}
    \item Part 1. For any matrix $A \in \R^{m \times n}$, we have
\begin{align*}
 \| U^\top A \| \leq \| A \|.
\end{align*}
\item Part 2. For any matrix $B \in \R^{k \times d}$, we have
\begin{align*}
\| U B \| = \| B \|.
\end{align*}
\end{itemize}
\end{fact}
\begin{proof}
{\bf Part 1.} 
By the property of $U$, we know that $U U^\top \preceq I_m $.

Thus, for any vector $x$, we have
\begin{align*}
x^\top A^\top U U^\top A x \leq x^\top A^\top A x.
\end{align*}
Thus, $\| U^\top A \| \leq \|A \|$.

{\bf Part 2.} 
By property of $U$, we have $U^\top U = I_k$.

Thus, for any vector $x$, we have
\begin{align*}
x^\top B^\top U^\top U B x = x^\top B^\top B x.
\end{align*}

Thus, $\| U B \| = \| B \|$.
\end{proof}

The following is a collection of simple algebraic facts.
\begin{fact}\label{fac:x_vs_sqrt_1_minus_x^2}

Let $y \in (0,0.1)$. We have
\begin{itemize}
    \item Part 1. If $x \in (0,1/2)$, then $\sqrt{1-x^2} - y x \geq 1/2 $.
    \item Part 2. If $x \geq 1/2$, then $x - y \sqrt{1-x^2} \geq \frac{1}{2}x$.
    \item Part 3. If $x \in [0,1]$, then $1-\sqrt{1-x^2} \leq x^2$.
\end{itemize}
\end{fact}
\begin{proof}
{\bf Proof of Part 1.} We have 
\begin{align*}
  \sqrt{1-x^2} -  y x 
\geq & ~ \sqrt{1-x^2} - \frac{1}{5} x \\
\geq & ~ 1 - \frac{1}{2} x - \frac{1}{5} x\\
\geq & ~ \frac{1}{2},
\end{align*}
where the first step follows from $y \in (0,0.1)$, the second step follows from $\sqrt{1-x^2} \geq 1-\frac{1}{2}x$ for all $x \in [0,4/5]$, and the last step follows from $x \leq 1/2$.

{\bf Proof of Part 2.} We know that
\begin{align}\label{eq:y_leq_x/2}
y \sqrt{1-x^2} \leq y \leq 0.1 \leq x/2,
\end{align}
where the first step follows from $0 \leq \sqrt{1-x^2} \leq 1$, the second step follows from $y \in (0,0.1)$, and the last step follows from $x \geq 1/2$.

Thus, 
\begin{align*}
x - y \sqrt{1-x^2} 
\geq & ~ x - x/2 \\
= & ~ x/2,
\end{align*}
where the first step follows from Eq.~\eqref{eq:y_leq_x/2} and the second step follows from simple algebra.

{\bf Proof of Part 3.} We know that
\begin{align*}
1-x^2 \leq \sqrt{1-x^2}
\end{align*}
for all $x$ in the domain of $\sqrt{1-x^2}$.

Then, the statement is true.
\end{proof}

\subsection{Properties of Angles and Distances}
\label{sub:app_preli:angle}

We explore some more relations for angles and distances between subspaces.

Let $U, V\in \R^{n\times k}$ be two matrices with orthonormal columns, we use $\dist_c(V, U)$ to denote the following minimization problem:

\begin{align*}
    \dist_c(V, U) = & ~ \min_{Q\in O_k} \|VQ-U\|,
\end{align*}
where $O_k \subset \R^{k\times k}$ is the set of $k\times k$ orthogonal matrices.

The next lemma is a simple application of fundamental subspace decomposition.

\begin{lemma}\label{lem:AAtop_I}
Let $A\in \R^{n\times k}$ be an orthonormal basis, then there exists a matrix $A_{\perp}\in \R^{n\times (n-k)}$ with $AA^\top+A_{\perp}A_{\perp}^\top=I_n$.
\end{lemma}

\begin{proof}
We know that the column space of $A$ is orthogonal to the null space of $A^\top$, and the null space of $A^\top$ has dimension $n-k$. 

Let $A_{\perp}$ be an orthonormal basis of the null space of $A^\top$. 

We have 
\begin{align}\label{eq:Atop_Abot_0}
    A^\top A_{\perp}=0.
\end{align}

It remains to show that 
\begin{align*}
    AA^\top+A_{\perp}A_{\perp}^\top=I_n.
\end{align*}

Let $z\in \R^n$ be any vector.

Since $\R^n$ is the direct sum of the column space of $A$ and the null space of $A^\top$, we can write $z = Ay + A_\perp w$. By linearity, it suffices to verify the identity for $z=Ay$ and for $z=A_\perp w$. 

{\bf Case 1: $z$ in the column space of $A$.} In this case, we can write $z=Ay$.

Then,
\begin{align*}
    AA^\top z+A_\perp A_\perp^\top z = & ~ AA^\top Ay+A_\perp A_\perp^\top Ay \\
    = & ~ Ay+A_\perp A_\perp^\top Ay \\
    = & ~ Ay+A_\perp (A^\top A_\perp)^\top y \\
    = & ~ Ay+0 \\
    = & ~ z,
\end{align*}
where the first step follows from $z=Ay$, the second step follows from $A$ is orthogonal that $A^\top = A^{-1}$, the third step follows from $(AB)^\top = B^\top A^\top$ for all matrices $A$ and $B$, the fourth step follows from Eq.~\eqref{eq:Atop_Abot_0}, and the last step follows from $z=Ay$.

{\bf Case 2: $z$ in the null space of $A^\top$.} In this case, we know that $A^\top z=0$ and $z=A_\perp y$, so
\begin{align*}
    AA^\top z+A_\perp A_\perp^\top z 
    = & ~ 0 + A_\perp A_\perp^\top A_\perp y \\
    = & ~ A_\perp y \\
    = & ~ z,
\end{align*}
where the first step follows from $A^\top z=0$ and $z=A_\perp y$, the second step follows from $A_{\bot}^\top A_{\bot} = I$, and the third step follows from $z=A_\perp y$.

Thus, we have shown $AA^\top+A_{\perp}A_{\perp}^\top=I_n$.
\end{proof}

The following lemma presents a simple inequality for orthogonal complements.

\begin{lemma}[Structural lemma for matrices with orthonormal columns]
\label{lem:trig_structural}
Let $U, V\in \R^{n\times k}$ be matrices with orthonormal columns. Then
\begin{align*}
    (V^\top U)_\bot = & ~ V_\bot^\top U.
\end{align*}
\end{lemma}

\begin{proof}
    Let us first compute the Gram matrix of $V^\top U$, which is
\begin{align*}
    U^\top VV^\top U 
    = & ~ U^\top (I-V_\bot V_\bot^\top) U \\
    = & ~ U^\top U-U^\top V_\bot V_\bot^\top U \\
    = & ~ I_k-U^\top V_\bot V_\bot^\top U,
\end{align*}
where the first step follows from $V_\bot V_\bot^\top + VV^\top = I$, the second step follows from simple algebra, and the last step follows from $U$ has orthonormal columns.

This means that $(V^\top U)_\bot=V_\bot^\top U$.
\end{proof}

The singular vectors can be parametrized by orthogonal complement and inverse, solely.

\begin{lemma}[Orthogonal and inverse share singular vectors]
\label{lem:perp_inv_singular}
Let $A\in \R^{k \times k}$ be non-singular such that there exists $A_{\perp}\in \R^{(n-k)\times k}$ with $A^\top A+A_\perp^\top A_\perp = I$, then the following holds:
\begin{itemize}
    \item Part 1. $A_\bot$ and $A^{-1}$ have the same set of singular vectors.
    \item Part 2. Let $u$ be a singular vector of $A$, if $u$ corresponds to $\sigma_i(A)$, then it corresponds to $\sigma_{k-i+1}(A_\bot)=\sqrt{1-\sigma_i^2(A)}$.
    \item Part 3. $\|A_\bot A^{-1}\|=\|A_\bot\| \|A^{-1}\|$.
\end{itemize}
\end{lemma}

\begin{proof}

{\bf Proof of Part 1.} Let $x\in \R^k$ be the unit eigenvector of $A$ that realizes the spectral norm.

Note that
\begin{align*}
    \|A_\bot x\|_2^2 = & ~ 1-\|A\|^2,
\end{align*}
we argue that $x$ corresponds to the smallest singular value of $A_\bot$ via contradiction. Suppose there exists some unit vector $y$ with 
\begin{align*}
    \|A_\bot y\|_2 < \|A_\bot x\|_2.
\end{align*}
By definition, we know that 
\begin{align*}
    \|A_\bot y\|_2^2+\|Ay\|_2^2=1,
\end{align*}
which means 
\begin{align*}
    \|Ay\|_2>\|Ax\|_2=\|A\|,
\end{align*}
and it contradicts the definition of spectral norm. 

Similarly, if $z$ is the unit vector that realizes the spectral norm of $A_\bot$, then it is also a singular vector corresponds to the smallest singular value of $A$, or equivalently, the spectral norm of $A^{-1}$. Our above argument essentially implies that $A_\bot$ and $A^{-1}$ have the same set of singular vectors. 

Our above argument is choosing $x$ to the eigenvector corresponding to the largest singular values. Similarly, we can choose to 2nd, 3rd, and then prove entire sets.

{\bf Proof of Part 2.} The key of the proof is that for any unit vector $u\in \R^k$, we have
\begin{align*}
    \|Au\|_2^2+\|A_\perp u\|_2^2 = & ~ 1,
\end{align*}
let $\sigma_1(A),\ldots,\sigma_k(A)$ be the singular values of $A$ and $u_1,\ldots,u_k$ be corresponding singular vectors. 

By above definition, we know
\begin{align*}
\| A u_i \|_2 = \sigma_i(A) 
\end{align*}

Then, we know
\begin{align*}
    \sigma_i^2(A)+\|A_\perp u_i\|_2^2 = & ~ 1,
\end{align*}
which means that $\|A_\perp u_i\|_2=\sqrt{1-\sigma_i^2(A)}$. Since $A_\perp^\top A_\perp = I - A^\top A$, the vector $u_i$ is also a right singular vector of $A_\perp$, with singular value $\sqrt{1-\sigma_i^2(A)}$. Note that all singular values of $A_\perp$ are in this form, i.e., we have its singular values being
\begin{align*}
    \sqrt{1-\sigma_1^2(A)},\ldots,\sqrt{1-\sigma_k^2(A)},
\end{align*}
as $\sigma_1(A)\geq \ldots \geq \sigma_k(A)$, and the above singular values are in ascending order. 

Thus, we have 
\begin{align*}
    \sigma_{k-i+1}(A_\perp) = & ~ \sqrt{1-\sigma_i^2(A)}.
\end{align*}

{\bf Proof of Part 3.} The proof is then straightforward.

Suppose that $(\lambda,z)$ is the largest singular value and the corresponding unit right singular vector of $A_\bot$ (i.e., $\| A_{\bot} \| = \lambda$). Then, we have
\begin{align}\label{eq:Abot_lambda_z}
    \|A_\bot z\|_2=\lambda
\end{align}
Using Part 2 (by choosing $i=k$), we know that $z$ is the right singular vector of $A$ for its smallest singular value $\sigma_k(A) = 1/\|A^{-1}\|$. Assume that $A^{-1}$ largest singular value is $\mu$ (e.g. $\| A^{-1} \| = \mu$). Then we have
\begin{align}\label{eq:A-1_mu_z}
    A^{-1}y=\mu z, \quad \text{where } y := Az/\|Az\|_2.
\end{align}
Then, we have
\begin{align}\label{eq:rewrite_A_bot_A_inverse_z}
    \|A_\bot A^{-1} y\|_2 
    = & ~ \|A_\bot \mu z\|_2 \notag \\
    = & ~ \mu \|A_\bot z\|_2 \notag \\
    = & ~ \lambda \mu,
\end{align}
where the first step follows from Eq.~\eqref{eq:A-1_mu_z}, the second step follows from $\mu > 0$, and the third step follows from Eq.~\eqref{eq:Abot_lambda_z}.  

From Eq.~\eqref{eq:rewrite_A_bot_A_inverse_z} and $\|y\|_2 = 1$, we know that
\begin{align*}
\| A_{\bot} A^{-1} \| \geq \lambda \mu = \| A_{\bot} \| \cdot \| A^{-1} \|.
\end{align*}

Using norm inequality, we have
\begin{align*}
\| A_{\bot} A^{-1} \| \leq \| A_{\bot} \| \cdot \| A^{-1} \|.
\end{align*}

 Combining the lower bound and upper bound, we have 
 \begin{align*}
 \|A_\bot A^{-1}\|=\|A_\bot\|\|A^{-1}\|,
 \end{align*}
 and we have proved the assertion.
\end{proof}

We prove several fundamental trigonometry equalities for subspace angles.

\begin{lemma}\label{lem:tan_is_sin_cos}
Let $U, V\in \R^{n\times k}$ be orthonormal matrices, then 
\begin{align*}
    \tan \theta(V,U) = \frac{\sin \theta(V,U)}{\cos \theta(V,U)}.
\end{align*}
\end{lemma}

\begin{proof}

We have, 
\begin{align*}
    \tan\theta(V,U)
    = & ~ \| V_{\bot}^\top U ( V^\top U )^{-1} \|\\
    = & ~ \|(V^\top U)_\bot (V^\top U)^{-1}\|\\
    = & ~ \|(V^\top U)_\bot\| \|(V^\top U)^{-1}\|\\
    = & ~ \frac{\|(V^\top U)_\bot\|}{1/\|(V^\top U)^{-1}\|} \\
    = & ~ \frac{\|V_\bot^\top U\|}{1/\|(V^\top U)^{-1}\|} \\
    = & ~ \frac{\sin \theta(V,U)}{\cos \theta(V,U)}, 
\end{align*}
where the first step follows from the definition of $\tan \theta(V,U)$ (see Definition~\ref{def:angle_and_distance}), the second step follows from Lemma~\ref{lem:trig_structural}, the third step follows from the part 3 of Lemma~\ref{lem:perp_inv_singular}, the fourth step follows from simple algebra, the fifth step follows from Lemma~\ref{lem:trig_structural}, and the last step follows from the definition of $\sin \theta(V,U)$ and $\cos \theta(V,U)$ (see Definition~\ref{def:angle_and_distance}).
\end{proof}

\begin{lemma}\label{lem:sin^2_and_cos^2_is_1}
Let $U, V\in \R^{n\times k}$ be orthogonal matrices, then 
\begin{align*}
\sin^2\theta(V, U) + \cos^2\theta(V,U) =1.
\end{align*}
\end{lemma}
\begin{proof}
Recall that $\cos\theta(V,U)=\frac{1}{\|(V^\top U)^{-1}\|}$ and $\sin\theta(V,U)=\|V_\bot^\top U\|$.

By Lemma~\ref{lem:trig_structural}, we know that 
\begin{align*} 
(V^\top U)_\bot=V^\top_\bot U,
\end{align*}
so by the definition of $\sin\theta(V,U)$ (see Definition~\ref{def:angle_and_distance}), we have
\begin{align*}
\sin\theta(V,U)=\|(V^\top U)_\bot \|.
\end{align*}

We define matrix $A \in \R^{k \times k}$ as follows,
\begin{align}\label{eq:define_A}
A:=V^\top U.
\end{align}

Since $U^\top U = U^\top VV^\top U + U^\top V_\bot V_\bot^\top U = I_k$, we have $\sigma_{\min}^2(V^\top U) = \lambda_{\min}(U^\top VV^\top U) = 1-\lambda_{\max}(U^\top V_\bot V_\bot^\top U) = 1-\|V_\bot^\top U\|^2$. This is the claim with $\cos\theta(V,U)=\sigma_{\min}(V^\top U)$, and it also holds when $V^\top U$ is singular. The argument below assumes that $V^\top U$ is invertible and serves only as intuition. By part 1 of Lemma~\ref{lem:perp_inv_singular}, we know that $A_\bot$ and $A^{-1}$ have the same set of singular vectors.

Let $z\in \R^k$ be the unit singular vector with singular value $\|A_\bot\|$. 

This implies  
\begin{align}\label{eq:Abot_z2}
\| A_{\bot } z \|_2 = \| A_{\bot} \|.
\end{align}

 Note that $A_\bot$ and $A^{-1}$ have the same singular vectors implies that the singular vector realizing $\|A_\bot\|$ corresponds to the smallest singular value of $A$, i.e.,  
 \begin{align}\label{eq:Az2_sigma_min_A}
 \| Az\|_2 = \sigma_{\min}(A).
\end{align}

 Then, we have
\begin{align}\label{eq:Abot2_sigmamin_1}
    1 
    = & ~ z^\top z \notag \\
    = & ~ z^\top (A^\top A +  A_\bot^\top A_\bot) z \notag\\
    = & ~ z^\top A^\top Az+z^\top A_\bot^\top A_\bot z \notag\\
     = & ~ \| A z \|_2^2+ \| A_\bot z \|_2^2 \notag\\
    = & ~ \| A z \|_2^2 +  \|A_\bot\|^2  \notag\\
    = & ~\sigma_{\min}^2(A) +  \|A_\bot\|^2,
\end{align}
where the first step follows from $\| z \|_2^2=1$, the second step follows from $A^\top A + A_{\bot}^\top A_{\bot} = I$, the third step follows from simple algebra, the fourth step follows from $x^\top B^\top B x = \| B x \|_2^2$, the fifth step follows from Eq.~\eqref{eq:Abot_z2}, and the sixth step follows from Eq.~\eqref{eq:Az2_sigma_min_A}.  

Also, we get
\begin{align*}
    \|A_\bot\|^2+\sigma_{\min}^2(A)
    = & ~ \|(V^\top U)_\bot\|^2+\sigma_{\min}^2(V^\top U) \\
    = & ~ \|V_\bot^\top U\|^2+\sigma_{\min}^2(V^\top U) \\
    = & ~ \sin^2\theta(V,U) + \cos^2\theta(V,U),
\end{align*}
where the first step follows from Eq.~\eqref{eq:define_A}, the second step follows from Lemma~\ref{lem:trig_structural}, the third step follows from definitions of $\sin\theta$ and $\cos\theta$ (see Def.~\ref{def:angle_and_distance}).  

Therefore, by Eq.~\eqref{eq:Abot2_sigmamin_1}, we have $\sin^2\theta(V,U) + \cos^2\theta(V,U) = 1$. This completes the proof.
\end{proof}

The next lemma demonstrates relationship between several trigonometry definitions.

\begin{lemma}\label{lem:more_properties_angle_distance}
    Let $V, U \in \R^{n \times k}$ be two orthogonal matrices, then 
    \begin{itemize}
    \item Part 1. $\sin \theta (V, U) \leq \tan \theta (V, U)$.
    \item Part 2. $\frac{1 - \cos \theta ( V, U)}{\cos \theta (V, U)} \leq \tan \theta (V, U)$.
    \item Part 3. $\sin \theta (V, U) \leq \dist_c ( V, U )$.
    \item Part 4. $\dist_c ( V, U ) \leq \sin \theta ( V, U ) + \frac{1 - \cos \theta ( V, U )}{\cos \theta ( V, U )}$.
    \item Part 5. $ \dist_c ( V, U ) \leq 2 \tan \theta ( V, U ) $.
    \end{itemize}
\end{lemma}

\begin{proof}

We define
\begin{align*}
Q^* := \arg \min_{Q \in O^{k \times k}} \| V Q - U \|,
\end{align*}
Next, we define matrix $R$ to be
\begin{align}\label{eq:define_R}
R : = U - V Q^*.
\end{align}

Then we have
\begin{align}\label{eq:dist_is_R}
\dist_c (V,U) = \| R \|
\end{align}
and
\begin{align}\label{eq:sin_less_than_R}
            \sin \theta (V, U) 
            = & ~ \| V_{\bot}^\top U \| \notag \\ 
            = & ~ \| V_{\bot}^\top ( V Q^* + R ) \| \notag \\
            = & ~ \| V_{\bot}^\top V Q^* + V_{\bot}^\top R  \| \notag \\
            = & ~ \| V_{\bot}^\top R \| \notag\\
            \leq & ~ \| V_{\bot}^\top\| \| R \| \notag\\
            \leq & ~ \| R \|, 
\end{align}
        where the first step follows from the definition of 
        \begin{align*}
            \sin \theta (V, U),
        \end{align*}
        the second step follows from $U = V Q^* + R$ (see Eq.~\eqref{eq:define_R}), the third step follows from simple algebra, the fourth step follows from $V_{\bot}^\top V = (V^\top V_{\bot})^\top = 0$ (see Lemma~\ref{lem:AAtop_I}), the fifth step follows from $\|AB\| \leq \|A\| \|B\|$, and the last step follows from $\| V_{\bot}^\top \| \leq 1$.

    {\bf Proof of Part 1.}
Note that
\begin{align*}
\sin \theta (V,U) 
= & ~\tan \theta (V,U) \cdot \cos \theta (V,U) \\
\leq & ~ \tan \theta (V,U),
\end{align*}
where the first step follows from Lemma~\ref{lem:tan_is_sin_cos}, the last step follows from $\cos \theta (V,U) \leq 1$.

    {\bf Proof of Part 2.}
    For simplicity, we use $\theta $ to represent $\theta(V,U)$.
    
    The statement is
    \begin{align*}
        \frac{1-\cos \theta}{ \cos \theta } \leq \tan \theta,
    \end{align*}
    which implies
    \begin{align*}
        1 - \cos \theta \leq \sin \theta.
    \end{align*}
    Taking the square on both sides, we get
    \begin{align*}
        1 - 2 \cos \theta + \cos^2 \theta \leq \sin^2 \theta.
    \end{align*}
    Using Lemma~\ref{lem:sin^2_and_cos^2_is_1}, the above equation is further equivalent to
    \begin{align*}
    2 \cos \theta ( \cos \theta -1 ) \leq 0.
    \end{align*}
    This is true forever, since $\cos \theta \in [0,1]$.

    {\bf Proof of Part 3.} Given $\dist_c (V, U) = \| R \|$ (Eq.~\eqref{eq:dist_is_R}) and $\sin \theta ( V, U ) \leq \| R \|$ (Eq.~\eqref{eq:sin_less_than_R}), we have 
        \begin{align*}
        \sin \theta ( V, U ) \leq \dist_c ( V, U ). 
        \end{align*}

    {\bf Proof of Part 4.}
        We define $A,D,B$ as the SVD of matrix $V^\top U$ as follows
        \begin{align*}
            A D B^\top = \mathrm{SVD} ( V^\top U ),
        \end{align*}
        then 
        \begin{align}\label{eq:A_Ytop_V_B_D}
            A = V^\top U B D^{-1}
        \end{align}
        and 
        \begin{align*}
            \sigma_{\min} ( D ) = \sigma_{\min} ( V^\top U ) = \cos \theta ( V, U ).
        \end{align*}
        In addition, 
        \begin{align*}
            ( A B^\top )^\top 
            = & ~ B A^\top \\
            = & ~ ( B^\top )^\top A^\top \\
            = & ~ ( B^\top )^{-1} A^{-1} \\
            = & ~ ( A B^\top )^{-1},
        \end{align*}
        where the first step follows from $(AB)^\top = B^\top A^\top$ for all matrices $A$ and $B$, the second step follows from the simple property of matrix, the third step follows from the fact that $B$ is orthogonal, and the last step follows from simple algebra, i.e., $A B^\top \in O^{k \times k}$. 
        
        For $\dist_c ( V, U )$, we have 
        \begin{align}\label{eq:bound_dist}
             \dist_c ( V, U ) 
            \leq & ~ \| V A B^\top - U \| \notag\\
            = & ~ \| V V^\top U B D^{-1} B^\top - U\| \notag\\
            = & ~ \| V V^\top U B D^{-1} B^\top - V V^\top U + V V^\top U - U \| \notag\\
            \leq & ~ \| V V^\top U B D^{-1} B^\top - V V^\top U \| + \| V V^\top U - U \|,
        \end{align}
        where the first step follows from $AB^\top$ can not provide a better cost than minimizer, the second step follows from Eq.~\eqref{eq:A_Ytop_V_B_D}, the third step follows from simple algebra, and the last step follows from the triangle inequality. 
        
        For the second term of Eq.~\eqref{eq:bound_dist}, namely $\| V V^\top U - U \|$, we have
        \begin{align*}
         \| V V^\top U - U \| 
         = & ~ \| (V V^\top - I) U \| \\
         = & ~ \sin \theta(V,U),
        \end{align*}
        where the first step follows from simple algebra and the second step follows from the definition of $\sin \theta(V,U)$ (see Definition~\ref{def:angle_and_distance}).
        
        For the first term of Eq.~\eqref{eq:bound_dist}, namely $\| V V^\top U B D^{-1} B^\top - V V^\top U \|$, we have
        \begin{align*}
             \| V V^\top U B D^{-1} B^\top - V V^\top U \| 
            = & ~ \| V V^\top U ( B D^{-1} B^\top - I ) \|   \\
            \leq & ~ \| B D^{-1} B^\top - I \|  \\
            = & ~ \| D^{-1} - I \|   \\
            = & ~ \frac{1-\cos \theta(V,U)}{ \cos \theta(V,U) },
        \end{align*}
        where the first step follows from simple algebra, the second step follows from $\| V V^\top U \| \leq 1$, the third step follows from $B$ is orthonormal basis, and the last step follows from $\| D^{-1} - I \| = | \frac{1}{\sigma_{\min} (D) } - 1 | = |\frac{1}{\cos \theta} - 1 |  = \frac{1-\cos\theta}{\cos \theta}$.

{\bf Proof of Part 5.}
We have
\begin{align*}
    \dist_c(V,U) 
    \leq & ~\sin \theta (V,U) + \frac{1-\cos \theta (V,U)}{ \cos \theta(V,U) } \\
    \leq & ~ \sin \theta (V,U) + \tan \theta(V,U) \\
    \leq & ~ 2 \tan \theta (V,U),
\end{align*}
where the first step follows from Part 4, the second step follows from Part 2, and the last step follows from Part 1.
\end{proof}

\subsection{Tools from Prior Works}
\label{sub:app_preli:tools}

The previous paper~\cite{jns13} assumes the entries of $M$ are sampled independently with the following probability. Note that the choice of $p$ also determines the sample complexity of the algorithm.

\begin{definition}[Sampling probability]\label{def:sampling_probability}
Let $C \geq 10$ denote a sufficiently large constant. We define sample probability to be
\begin{align*}
p: = C \cdot ({\sigma_1^*}/{\sigma_k^*})^2 \cdot \mu^2 \cdot k^{2.5} \cdot \log{n} \cdot \log (  k \|M\|_F / \epsilon )  / ( m\delta_{2k}^2 ).
\end{align*}
Then, we sample $\Omega \sim [m] \times [n]$ each coordinately independently according to that probability.
\end{definition}

The parameter $\delta_{2k}$ is a tuning parameter for sampling probability.

\begin{definition}\label{def:delta_2k}
We choose $\delta_{2k}$ as follows 
\begin{align*}
\delta_{2k} := \frac{1}{100} \frac{1}{k} \sigma_k^* / \sigma_1^*.
\end{align*}
We can see that $\delta_{2k} \in (0,0.01]$.
\end{definition}

Given the set of indices, we denote the row and column selection operator as follows.

\begin{definition}\label{def:Omega_*_j}
For each $j \in [n]$, we define set $\Omega_{*,j} \subset [m]$
\begin{align*}
    \Omega_{*,j} : =\{ i \in [m] ~:~ (i,j) \in \Omega\}.
\end{align*}
For each $i \in [m]$, we define set $\Omega_{i,*} \subset [n]$,
\begin{align*}
\Omega_{i,*} := \{ j \in [n] ~:~ (i,j) \in \Omega \}.
\end{align*}
\end{definition}

The next structural lemma bounds the inner product between vectors inside set $\Omega$.

\begin{lemma}[Lemma C.5 in \cite{jns13}]\label{lem:modified_of_lem6.1}
Let $\Omega \subset [m] \times [n]$ be a set $f$ indices sampled uniformly at random from $[m] \times [n]$ with each element of $[m]\times [n]$ sampled independently with probability $p \geq C (\log n) / m$. Let $C>1 $ be a sufficiently large constant.

Then with probability $1-1/\poly(n)$, for all $x \in \R^m$ with $\sum_{i=1}^m x_i = 0$ and for all $y \in \R^n$, we have
\begin{align*}
\sum_{(i,j) \in \Omega} x_i y_j \leq C (p\max\{m,n\})^{1/2} \| x \|_2 \| y\|_2.
\end{align*}
\end{lemma}

The following lemma bounds the spectral gap between two matrix inverses in terms of the larger of their spectral norm squared, times the spectral norm of their differences.
\begin{lemma}[\cite{w73}, see Theorem 1.1 in \cite{mz10} as an example]\label{lem:perturb}
For two conforming real matrices $A$ and $B$, 
\begin{align*}
\| A^\dagger - B^\dagger \| \leq  2 \cdot \max \{ \| A^\dagger \|^2, \| B^\dagger \|^2 \} \cdot \| A - B \|
\end{align*}
\end{lemma}

The next claim bounds the inner product between any unit vector and rows of incoherent matrix.

\begin{claim}\label{cla:sum_Ux_4th_power}
Suppose $U_t$ is $\mu_2$ incoherent, and $U_*$ is $\mu$ incoherent. For any unit vectors $x,y$ we have
\begin{itemize}
    \item Part 1. $\sum_{i \in [m]} \langle x, U_{t,i} \rangle^4 \leq \mu_2^2 k / m$
    \item Part 2. $\sum_{i\in [m]} \langle y, U_{*,i}\rangle^2 \cdot \langle x, U_{t,i} \rangle^2 \leq \mu^2 k / m $.
\end{itemize}
\end{claim}
\begin{proof}

{\bf Proof of Part 1.} Note that $\| U_{t,i} \|_2 \leq 1$, so
\begin{align}\label{eq:x_Uti}
\langle x , U_{t,i} \rangle = & ~ \| U_{t,i} \|_2 \cdot \langle x, U_{t,i} / \| U_{t,i} \|_2 \rangle \notag\\
\leq & ~  \| U_{t,i} \|_2,
\end{align} 
where the second step follows from the inner product between two unit vectors is at most 1. 

Thus, based on Eq.~\eqref{eq:x_Uti},
\begin{align}\label{eq:x_dot_U_t_i}
\langle x , U_{t,i} \rangle^2 
\leq & ~ \| U_{t,i} \|_2^2 \notag \\
\leq & ~ \mu_2^2 k / m.
\end{align}
where the last step follows from $U_t$ is $\mu_2$ incoherent.

We have
\begin{align*}
 \sum_{i=1}^m \langle x, U_{t,i} \rangle^4    
     \leq & ~    \max_{i\in [m]} \langle x, U_{t,i} \rangle^2  \sum_{i=1}^m \langle x, U_{t,i} \rangle^2    \notag\\
     \leq & ~  \frac{\mu_2^2 k }{m} \sum_{i=1}^m \langle x, U_{t,i} \rangle^2  \\
     = & ~  \frac{\mu_2^2 k }{m}  x^\top U_t^\top U_t x \\
    = & ~  \frac{\mu_2^2 k }{m} x^\top x \\
    = & ~ \frac{\mu_2^2 k }{m},
\end{align*}
where the first step follows from $ \sum_{i \in [m]} a_i b_i \leq \max_{i \in [m]} a_i \sum_{i \in [m]} b_i$, the second step follows from Eq.~\eqref{eq:x_dot_U_t_i}, the third step follows from $U_t^\top U_t = \sum_{i=1}^m U_{t,i}^\top U_{t,i}$, the fourth step follows from $U_t^\top U_t = I_k$, and the last step follows from the Lemma statement that $x$ is a unit vector.

{\bf Proof of Part 2.} Similarly as the proof of part 1. We know that
\begin{align}\label{eq:y_dot_U_*_i}
\langle y, U_{*,i} \rangle^2 \leq \| U_{*,i} \|_2^2 \leq \mu^2 k / m.
\end{align}
We can show
\begin{align*}  
\sum_{i=1}^m \langle x, U_{t,i} \rangle^2 \cdot \langle y, U_{*,i}\rangle^2 
\leq & ~  \max_{i \in [m]} \langle y, U_{*,i} \rangle^2 \cdot \sum_{i=1}^m \langle x, U_{t,i} \rangle^2 
  \notag\\
  \leq & ~  \frac{\mu^2 k}{m} \cdot \sum_{i=1}^m \langle x, U_{t,i} \rangle^2 
  \notag\\
= & ~ \frac{\mu^2 k}{m} \cdot x^\top U_t^\top U_t x   \notag\\
= & ~ \frac{\mu^2 k}{m}  x^\top x  \notag\\
= & ~ \frac{\mu^2 k}{ m},
\end{align*}
where the first step follows from $\sum_{i \in [m]} a_i b_i \leq \max_{i \in [m]} a_i \sum_{i \in [m]} b_i$, the second step follows from Eq.~\eqref{eq:y_dot_U_*_i}, the third step follows from $U_t^\top U_t = \sum_{i=1}^m U_{t,i}^\top U_{t,i}$, the fourth step follows from $U_t^\top U_t = I_k$, and the last step follows from the Lemma statement that $x$ is a unit vector.
\end{proof}

\subsection{Probability Tools}\label{sub:app_preli:probability}

We introduce several well-known probability inequalities. 

\begin{lemma}[Chernoff bound]\label{lem:chernoff_bound}
Let $X = \sum_{i=1}^n X_i$, where $X_i=1$ with probability $p_i$ and $X_i=0$  with probability $1-p_i$, and all $X_i$ are independent. Let $\mu = \E[X] = \sum_{i=1}^n p_i$. Then
\begin{itemize}
    \item $\Pr[ X \geq (1+\delta) \mu ] \leq \exp(-\delta^2 \mu /3)$, $\forall \delta \in (0,1]$;
    \item $\Pr[ X \leq (1-\delta) \mu] \leq \exp(-\delta^2 \mu /2) $, $\forall \delta \in (0,1)$.
\end{itemize}
\end{lemma}

\begin{lemma}[Bernstein inequality]\label{lem:bernstein_inequality}
Let $X_1, \cdots, X_n$ be independent zero-mean random variables (i.e., $\E[X_i] = 0$ for all $i \in [n]$). Suppose that $|X_i| \leq M$ almost surely, for all $i \in [n]$. 

Then, for any positive $t$, we have
\begin{align*}
\Pr[ \sum_{i=1}^n X_i > t ] \leq \exp( -  \frac{t^2/2}{ \sum_{j=1}^n \E[X_j^2] + M t/3 } ).
\end{align*}
\end{lemma}

%% file: sketch.tex
\section{High Accuracy Weighted Regression Solver}
\label{sec:high_accuracy_solver}

We demonstrate our high accuracy, iterative solver for weighted multiple response regression in this section.

\subsection{Dense, High Accuracy and Iterative Solver}
\label{sub:preli:dense_high_accuracy}

Here, we mainly focus on the {\sf SRHT} matrix, the algorithm that processes it, and its properties. We start by introducing the definition of the {\sf SRHT} matrix \cite{ldfu13}.

\begin{definition}[Subsampled randomized Hadamard transform ({\sf SRHT})]
\label{def:srht}
The \textsf{SRHT} matrix $S\in\R^{m \times n}$ is defined as $S:=\frac{1}{\sqrt{m}} P H D$, where each row of matrix $P \in \{0, 1\}^{m \times n}$ contains exactly one $1$ at a random position. $H$ is the $n \times n$ Hadamard matrix. $D$ is an $n \times n$ diagonal matrix with each diagonal entry being a value in $\{-1, +1\}$ with equal probability.
\end{definition}

\begin{remark}\label{rem:SA_compute_time}
For an $n\times d$ matrix $A$, $SA$ can be computed in time $O(nd\log n)$.
\end{remark}

The following lemma states that a suitable chosen {\sf SRHT} provides the so-called subspace embedding property.

\begin{lemma}[\cite{t11}]
\label{lem:iter_regression}
Let $S\in \R^{m\times n}$ be an {\sf SRHT} matrix as in Def.~\ref{def:srht}. Let $\epsilon_{\ose},\delta_{\ose}\in (0,1)$ be parameters. 
 Then for any integer $d\leq n$, if $m_{\mathrm{sk}}\geq C\epsilon_{\ose}^{-2}d\log^2(n/\delta_{\ose})$ for a sufficiently large absolute constant $C$, then matrix $S$ is an $(\epsilon_{\ose},\delta_{\ose})$ oblivious subspace embedding, i.e., for any fixed orthonormal basis $U\in \R^{n\times d}$ with probability at least $1-\delta_{\ose}$, the singular values of $SU$ lie in $[1-\epsilon_{\ose}, 1+\epsilon_{\ose}]$.
\end{lemma}

{\sf SRHT} can be further utilized for a high accuracy regression solver, as presented by the following lemma.

\begin{lemma}[Dense and high accuracy regression,~\cite{amt10}]\label{lem:alg_takes_time}
Given a matrix $A\in \R^{n\times d}$ of full column rank and a vector $b\in \R^n$, let $\epsilon_1 \in (0, 0.1)$ and $\delta_1 \in (0,0.1)$, there exists an algorithm that takes time 
\begin{align*}
    O((nd\log n+d^3\log^2(n/\delta_1))\log(1/\epsilon_1))
\end{align*}
and outputs $x' \in \R^d$ such that
\begin{align*}
\| A x' - b \|_2 \leq (1+\epsilon_1) \min_{x \in \R^d} \| A x - b \|_2
\end{align*}
holds with probability $1-\delta_1$.
\end{lemma}

The above result also extends to weighted regression, measured in terms of the $\|v \|_w^2=\sum_{i=1}^n w_i v_i^2$ norm. For the sake of completeness, we include the algorithm and a proof here.

\begin{algorithm}[!ht]
\caption{High precision solver}
\label{alg:iter_regression}
\begin{algorithmic}[1]
\Procedure{HighPrecisionReg}{$A\in \R^{n\times d}, b\in \R^n,n,d, \epsilon_1 \in (0,1), \delta_1 \in (0,1)$} \Comment{Lemma~\ref{lem:alg_takes_time}}
\State $T_1\gets  \Theta( \log(1/\epsilon_1) )$  
\State $\epsilon_{\ose}\gets 0.01$
\State $\delta_{\ose} \gets \delta_1$
\State $m_{\mathrm{sk}} \gets \Theta(  \epsilon_{\ose}^{-2}\cdot d \cdot \log^2(n/\delta_{\ose}) )$
\State Let $S\in \R^{m_{\mathrm{sk}} \times n}$ be an {\sf SRHT} matrix
\State Compute QR decomposition of $SA=QR^{-1}$
\State $x_0\gets \arg\min_{x\in \R^d} \|S ARx-Sb\|_2$
\For{$t=0\to T_1$}
\State $x_{t+1}\gets x_t+R^\top A^\top (b-ARx_t)$
\EndFor
\State \Return $R x_{T_1}$
\EndProcedure
\end{algorithmic}
\end{algorithm}

\begin{proof}[Proof of Lemma~\ref{lem:alg_takes_time}]

Let us analyze Algorithm~\ref{alg:iter_regression}, first on its convergence then on its runtime. 

Note that the $S$ we choose is an $(\epsilon_{\ose}, \delta_{\ose})$-oblivious subspace embedding. We apply it to the column span of $[A,\, b]$, which contains the column span of $A$ and has dimension at most $d+1$. For this, we use Lemma~\ref{lem:iter_regression} with $d+1$ in place of $d$, which only changes the constant in $m_{\mathrm{sk}}$. Since $SA=QR^{-1}$ where $Q$ has orthonormal columns, we have $SAR = Q$, so $\|ARx\|_2 \in [\frac{1}{1+\epsilon_{\ose}}, \frac{1}{1-\epsilon_{\ose}}] \cdot \|x\|_2$, that is, the singular values of $AR$ lie in $[1/(1+\epsilon_{\ose}),1/(1-\epsilon_{\ose})]$. 

Let $AR=U\Sigma V^\top$ be the SVD of $AR$ and $x^*$ denote the optimal solution to the regression $\min_{x\in \R^d} \|ARx-b\|_2$. 

Let us consider
\begin{align}\label{eq:AR_xt+1_x*}
    & ~ AR(x_{t+1}-x^*) \notag \\
    = & ~ AR (x_t+R^\top A^\top (b-ARx_t)-x^*) \notag \\
    = & ~ AR(x_t-x^*)+ARR^\top A^\top b-ARR^\top A^\top ARx_t \notag \\
    = & ~  AR(x_t-x^*)+ARR^\top A^\top ARx^*-ARR^\top A^\top ARx_t \notag \\
    = & ~ (AR-ARR^\top A^\top AR)(x_t-x^*) \notag \\
    = & ~ (U\Sigma V^\top - U\Sigma^3 V^\top) (x_t-x^*),
\end{align}
where the first step follows from the definition of $x_{t + 1}$ from Algorithm~\ref{alg:iter_regression}, the second step follows from simple algebra, the third step follows from the normal equations $(AR)^\top b = (AR)^\top AR x^*$, the fourth step follows from simple algebra, the last step follows from the SVD, $AR=U\Sigma V^\top$.

Therefore,
\begin{align*}
    \|AR(x_{t+1}-x^*)\|_2 
    = & ~ \|(U\Sigma V^\top - U\Sigma^3 V^\top) (x_t-x^*) \|_2 \\
    = & ~ \|(\Sigma-\Sigma^3) V^\top (x_t-x^*)\|_2 \\
    \leq & ~ O(\epsilon_{\ose})\cdot \|V^\top (x_t-x^*)\|_2 \\
    \leq & ~ (1+\epsilon_{\ose}) \cdot O(\epsilon_{\ose}) \cdot \|\Sigma V^\top (x_t-x^*)\|_2 \\
    = & ~ O(\epsilon_{\ose})\cdot \|\Sigma V^\top (x_t-x^*)\|_2 \\
    = & ~ O(\epsilon_{\ose})\cdot \|U\Sigma V^\top (x_t-x^*)\|_2 \\
    = & ~ O(\epsilon_{\ose})\cdot \|AR(x_t-x^*)\|_2,
\end{align*}
where the first step follows from Eq.~\eqref{eq:AR_xt+1_x*}, the second step follows from $U^\top U = I$, the third step follows from $\|AB\| \leq \|A\| \cdot \|B\|$, the fourth step follows from $\sigma_{\min}(\Sigma) \geq 1/(1+\epsilon_{\ose})$, the fifth step follows from $\epsilon_{\ose} \in (0,0.1)$, the sixth step follows from $U^\top U = I$, and the last step follows from the SVD, $AR=U\Sigma V^\top$.

This means the error shrinks by a factor of $O(\epsilon_{\ose})$ per iteration. After $T=O(\log(1/\epsilon_1))$ iterations, we have
\begin{align}\label{eq:init_T_gap}
    \|AR(x_T-x^*)\|_2 \leq & ~ O(\epsilon_1)\cdot \|AR(x_0-x^*)\|_2,
\end{align}
and recall for initial solution $x_0$, since $S$ is a subspace embedding for the column span of $[A,\, b]$, we have
\begin{align*}
    \|ARx_0-b\|_2 \leq & ~ \frac{1}{1-\epsilon_{\ose}} \|SARx_0-Sb\|_2 \\
    \leq & ~ \frac{1}{1-\epsilon_{\ose}} \|SARx^*-Sb\|_2 \\
    \leq & ~ \frac{1+\epsilon_{\ose}}{1-\epsilon_{\ose}}\cdot \|ARx^*-b\|_2,
\end{align*}
where the second step holds because $x_0$ minimizes $\|SARx-Sb\|_2$. As $\epsilon_{\ose} = 0.01$, the above equation implies that
\begin{align}\label{eq:crude_sketch}
\|ARx_0-b\|_2^2  - \|ARx^*-b\|_2^2 \leq O(\epsilon_{\ose}) \|ARx^*-b\|_2^2.
\end{align}

We can wrap up the proof as follows:
\begin{align*}
    & ~\|ARx_T-b\|_2^2 \\
    = & ~ \|AR(x_T-x^*)\|_2^2+\|ARx^*-b\|_2^2 \\
    \leq & ~ O(\epsilon_1^2)\cdot \|AR(x_0-x^*)\|_2^2+\|ARx^*-b\|_2^2 \\
    = & ~ O(\epsilon_1^2)\cdot (\|ARx_0-b\|_2^2-\|ARx^*-b\|_2^2) + \|ARx^*-b\|_2^2 \\
    \leq & ~ O(\epsilon_1^2)\cdot ( O(\epsilon_{\ose}  )  \|ARx^*-b\|_2^2 ) + \|ARx^*-b\|_2^2 \\
    = & ~ (1+O(\epsilon_1^2))\cdot \|ARx^*-b\|_2^2,
\end{align*}
where the first step follows from the Pythagorean theorem, the second step follows from Eq.~\eqref{eq:init_T_gap}, the third step follows from the Pythagorean theorem again, the fourth step follows from Eq.~\eqref{eq:crude_sketch}, and the fifth step follows from $\epsilon_{\ose}\leq 1$.

It remains to show the runtime. Applying $S$ to $A$ takes $O(nd\log n)$ time, the QR decomposition takes 
$O(m_{\mathrm{sk}} d^2)=O(d^3\log^2(n/\delta_{\ose}))$ time. 

Inverting the $d\times d$ triangular matrix $R^{-1}$ takes $O(d^3)$ time. To solve for $x_0$, we need to multiply $SA$ with $R$ in $O(m_{\mathrm{sk}} d^2)$ time and the solve takes $O(m_{\mathrm{sk}} d^2)$ time as well. To implement each iteration, we multiply from right to left which takes $O(nd)$ time. Putting things together gives the desired runtime.
\end{proof}

\subsection{Backward Error of Regression}
\label{sub:preli:forward_error}

We show how to translate the error from $\| A x' - b\|_2$ to $\| x' - x_{\mathrm{OPT}} \|_2$ via high-accuracy solver. In the low accuracy sketch-and-solve paradigm, {\sf SRHT} provides a stronger $\ell_\infty$ guarantee, but for high accuracy solver, we only attain an $\ell_2$ backward error bound.

\begin{lemma}[Backward error of regression]
\label{lem:forward_error}
Given a matrix $A\in \R^{n\times d}$ of full column rank $d$, a vector $b\in \R^n$. Suppose there exists a vector $x'\in \R^d$ such that 
\begin{align*}
\| A x' - b \|_2 \leq (1+\epsilon_1) \min_{x \in \R^d} \| A x - b \|_2. 
\end{align*}
Let $x_{\OPT}$ denote the exact solution to the regression problem, then it holds that 
\begin{align*}
    \|x' - x_{\rm OPT}\|_2 \leq & ~ O(\sqrt{\epsilon_1}) \cdot \frac{1}{\sigma_{\min}(A)} \cdot \|Ax_{\rm OPT}-b\|_2.
\end{align*}
\end{lemma}

\begin{proof}

Note that 
\begin{align*}
    \|Ax'-Ax_{\rm OPT}\|_2=\|Ax'-b-(Ax_{\rm OPT}-b)\|_2,
\end{align*}
so we can perform the following decomposition:
\begin{align}\label{eq:A_x'_xopt}
    \|A(x'-x_{\rm OPT})\|_2^2 
    = & ~ \|Ax'-b-(Ax_{\rm OPT}-b)\|_2^2 \notag\\
    = & ~ \|Ax'-b\|_2^2-\|Ax_{\rm OPT}-b\|_2^2 \notag\\
    \leq & ~ (1+\epsilon_1)^2 \|Ax_{\rm OPT}-b\|_2^2 - \|Ax_{\rm OPT}-b\|_2^2 \notag\\
    \leq & ~ 4 \epsilon_1 \cdot \|Ax_{\rm OPT}-b\|_2^2,
\end{align}
where the first step follows from simple algebra, the second step follows from the Pythagorean theorem, the third step follows from the assumption of Lemma~\ref{lem:forward_error}, and the fourth step follows from simple algebra. 

We expand more and explain why the second step follows from Pythagorean theorem. Ultimately, we will show that $Ax_{\rm OPT}-b$ is orthogonal to any vector in the column space of $A$. Recall that $x_{\rm OPT}$ can be written in terms of the solution to the normal equation:
\begin{align*}
    x_{\rm OPT} = & ~ (A^\top A)^\dagger A^\top b,
\end{align*}
we use the pseudo-inverse so that the computation below also covers rank-deficient $A$. Under the full column rank assumption of the lemma, $(A^\top A)^\dagger = (A^\top A)^{-1}$. Let $A=U\Sigma V^\top$ be the SVD of $A$, then the above can be written as
\begin{align*}
    x_{\rm OPT} = & ~ (V \Sigma^2 V^\top)^\dagger V \Sigma U^\top b \\
    = & ~ V (\Sigma^{\dagger})^2 V^\top V \Sigma U^\top b \\
    = & ~ V (\Sigma^{\dagger})^2 \Sigma U^\top b \\
    = & ~ V \Sigma^\dagger U^\top b,
\end{align*}
where we crucially use the fact that for a diagonal matrix $\Sigma$, the pseudo-inverse is by taking the inverse of nonzero diagonal entries, and keeping zero for remaining zero diagonal entries. We can therefore express $Ax_{\rm OPT}-b$ as 
\begin{align*}
    Ax_{\rm OPT}-b = & ~ U\Sigma V^\top V \Sigma^\dagger U^\top b - b \\
    = & ~ (U I_{{\rm rank}(A)} U^\top - I)b,
\end{align*}
where $I_{{\rm rank}(A)}$ is the matrix whose first ${\rm rank}(A)$ diagonal entries are 1, and remaining diagonal entries are 0. Finally, we compute $A^\top (Ax_{\rm OPT}-b)$:
\begin{align*}
    A^\top (Ax_{\rm OPT}-b) = & ~ V\Sigma U^\top (U I_{{\rm rank}(A)} U^\top - I)b \\
    = & ~ (V\Sigma U^\top - V\Sigma U^\top)b \\
    = & ~ (A^\top-A^\top)b \\
    = & ~ {\bf 0}_n,
\end{align*}
where for the second step, we use $U^\top U=I$, $\Sigma I_{{\rm rank}(A)}=\Sigma$ since $\Sigma$ only has its first ${\rm rank}(A)$ diagonal entries being nonzero. This justifies the use of Pythagorean theorem.

We can express the SVD of $A^\dagger$ in a similar manner:
\begin{align*}
    A^\dagger = & ~ V\Sigma^\dagger U^\top,
\end{align*}
therefore, we must have
\begin{align*}
    A^\dagger A = & ~ V\Sigma^\dagger U^\top U\Sigma V^\top \\
    = & ~ VI_{{\rm rank}(A)} V^\top
\end{align*}
and it acts as a contracting mapping, in the sense that for any $y\in \R^d$,
\begin{align*}
    \|A^\dagger Ay\|_2 = & ~ \|VI_{{\rm rank}(A)} V^\top y\|_2 \\
    = & ~ \|I_{{\rm rank}(A)} V^\top y \|_2 \\
    \leq & ~ \|V^\top y\|_2 \\
    \leq & ~ \|V^\top\| \|y\|_2 \\
    \leq & ~ \|y\|_2,
\end{align*}
where the third step is due to we only need to compute the $\ell_2$ norm on a sub-vector, and the last step is by $\|V^\top\|\leq 1$.

Therefore, we have
\begin{align*}
    \|x'-x_{\rm OPT}\|_2 
    \leq & ~ \| A^\dagger A ( x' - x_{\rm OPT} ) \|_2 \\
    \leq & ~ \|  A ( x' - x_{\rm OPT} ) \|_2 \cdot \| A^\dagger \| \\
    \leq & ~ 2\sqrt{\epsilon_1} \cdot \|Ax_{\rm OPT}-b\|_2 \cdot \|A^\dagger\| \\
    = & ~ \frac{2 \sqrt{\epsilon_1} }{\sigma_{\min}(A)}\cdot \|Ax_{\rm OPT}-b\|_2,
\end{align*}
where the first step follows from $A^\dagger A = V I_{{\rm rank}(A)} V^\top = I_d$, since $A$ has full column rank, the second step follows from $\|ABx\|_2 \leq \|Ax\|_2 \|B\|$, the third step follows from Eq.~\eqref{eq:A_x'_xopt}, and the last step follows from $\|A^\dagger\| = \frac{1}{\sigma_{\min}(A)}$. For rank-deficient $A$, the contraction above goes the other way and the first step fails. 
\end{proof}

\subsection{Reducing Weighted Linear Regression to Linear Regression}
\label{sub:app_sketch:reducing_weighted_linear_regression}

Solving the regression in matrix completion involves computing the Hadamard product with a binary matrix, we further generalize it to a nonnegative weight matrix, and show it's equivalent up to a rescaling.

\begin{lemma}[Reducing weighted linear regression to linear regression]
\label{lem:reduce_weighted_regression_to_standard_regression}
Given a matrix $A \in \R^{n \times d}$, a vector $b \in \R^n$, and a weight vector $w \in \R_{\geq 0}^n$. 

Let $\epsilon_1 \in (0,0.1)$ and $\delta_1 \in (0,0.1)$. 

Suppose there exists a regression solver that runs in time ${\cal T}(n,d,\epsilon_1,\delta_1)$ and outputs a vector $x'\in \R^d$ such that
\begin{align*}
    \|Ax'-b\|_2 \leq & ~ (1+\epsilon_1) \min_{x\in \R^d} \|Ax-b\|_2
\end{align*}
with probability at least $1-\delta_1$, then there is an algorithm that runs in \begin{align*}
O(\nnz(A))+{\cal T}(n, d, \epsilon_1, \delta_1)
\end{align*} 
time outputs $x' \in \R^d$ such that
\begin{align*}
\| A x' - b \|_w \leq (1+\epsilon_1) \min_{x \in \R^d} \| A x - b \|_w
\end{align*}
holds with probability $1-\delta_1$.
\end{lemma}
\begin{proof}
Recall that for a vector $z\in \R^n$,
$
    \|z\|_w^2=\sum_{i=1}^n w_i z_i^2,
$ 
we can define the following simple transformation on the regression problem: let $\wt b_i := \sqrt{w_i} b_i$ and $\wt A_{i, :} := \sqrt{w_i} A_{i,:}$.
 
We can then solve the regression on the transformed problem 
$
    \min_{x\in \R^d} \|\wt Ax-\wt b\|_2.
$

Let $x'\in \R^d$ be the solution such that
\begin{align*}
    \|\wt Ax'-\wt b\|_2 \leq & ~ (1+\epsilon_1) \min_{x\in \R^d} \|\wt Ax-\wt b\|_2,
\end{align*}
then for any $y\in \R^d$
\begin{align*}
    \|\wt Ay-\wt b\|_2^2 
    = & ~ \sum_{i=1}^n (\wt A_{i,:}^\top y-\wt b_i)^2 \\
    = & ~ \sum_{i=1}^n (\sqrt{w_i} A_{i,:}^\top y-\sqrt{w_i} b_i)^2 \\
    = & ~ \sum_{i=1}^n (\sqrt{w_i} (A_{i,:}^\top y - b_i))^2 \\
    = & ~ \sum_{i=1}^n w_i (A_{i,:}^\top y-b_i)^2 \\
    = & ~ \|Ay-b\|_w^2,
\end{align*}
where the first step follows from the definition of $\|\cdot\|_2^2$, the second step follows from the definition of $\wt A_{i,:}$ and $\wt b_i$, the third step follows from simple algebra, the fourth step follows from $(xy)^2 = x^2 y^2$, and the last step follows from the definition of $\|Ay-b\|_w^2$. This finishes the proof.
\end{proof}

\subsection{Solving Weighted Multiple Response Regression in Weight Sparsity Time}
\label{sub:high_accuracy_solver:main_regression}

The main result of this section is an algorithm, together with a weighted multiple response analysis that runs in time proportional to the sparsity of weight matrix $W$. 

 \begin{algorithm}[!ht]\caption{Fast, high precision solver for weighted multiple response regression. 
 We will use this algorithm in Algorithm~\ref{alg:alternating_minimization_completion_formal} (the formal version) and Algorithm~\ref{alg:alternating_minimization_completion} (the informal version).
 }\label{alg:multiple_regression}
\begin{algorithmic}[1]
\Procedure{FastMultReg}{$A \in \R^{m \times n}, B \in \R^{m \times k},W \in \R^{m \times n},m,n,k, \epsilon_0, \delta_0$} \Comment{Lemma~\ref{lem:main_regression}}
    \State \Comment{$A_i$ is the $i$-th column of $A$}
    \State \Comment{$W_i$ is the $i$-th column of $W$}
    \State \Comment{$D_{W_i}$ is a diagonal matrix where we put $W_i$ on diagonal, other locations are zero} 
    \State $\epsilon_1 \gets \epsilon_0^2/ (4 (n \kappa_0)^{2C})$ \Comment{$C$ and $\kappa_0$ as in Lemma~\ref{lem:main_regression} and its proof, and any upper bound on $\kappa_0$ may be used}
    \State $\delta_1 \gets \delta_0 / \poly(n)$
    \For{$i=1 \to n$}
        \State $S_i \gets \{ j \in [m] : W_{j,i} = 1 \}$ \Comment{the observed rows of column $i$}
        \State $X_i \gets \textsc{HighPrecisionReg}(B_{S_i,*} \in \R^{|S_i| \times k}, (A_i)_{S_i} \in \R^{|S_i|}, |S_i|, k, \epsilon_1,\delta_1)$ \Comment{Algorithm~\ref{alg:iter_regression}}
    \EndFor
    \State \Return $X$ \Comment{$X \in \R^{n \times k}$}
\EndProcedure
\end{algorithmic}
\end{algorithm}
 
\begin{lemma}\label{lem:main_regression}
For any accuracy parameter $\epsilon_0 \in (0,0.1)$, and failure probability $\delta_0 \in ( 0, 0.1 )$. We do not assume $m \leq n$. In the time bound, read $\log n$ and $\log(n/\delta_0)$ as $\log(m+n)$ and $\log((m+n)/\delta_0)$. Given matrix $M \in \R^{m \times n}$, $Y \in \R^{m \times k}$ and $W \in \{0,1\}^{m \times n}$, assume that $D_{W_i} Y$ has full column rank for every $i \in [n]$, and let $\kappa_0 := \max\{1, \sigma_{\max}(M) / \min_{i \in [n]} \sigma_{\min}(D_{W_i} Y)\}$ and $X = \arg\min_{X \in \R^{n \times k}} \| W \circ (M - Y X^\top) \|_F$. There exists an algorithm (Algorithm~\ref{alg:multiple_regression}) that takes time
\begin{align*}
O( (\|W\|_0 k\log n+nk^3 \log^2(n/\delta_0)) \cdot \log(n \kappa_0  /\epsilon_0)  )
\end{align*}
and returns a matrix $Z \in \R^{n \times k}$ such that  
\begin{itemize}
\item $\| Z  - X \| \leq \epsilon_0$
\end{itemize}
holds with probability $1-\delta_0$.
\end{lemma}

\begin{proof}
We can rewrite multiple regression into $n$ linear regression in the following way, 
\begin{align*}
     \min_{X \in \R^{n \times k}} \| M - YX^\top \|_W^2 
    =  \sum_{i = 1}^{n} \min_{X_{i, :} \in \R^{k}} \| D_{\sqrt{W_i}} Y X_{i, :} - D_{\sqrt{W_i}} M_{:, i} \|_2^2.
\end{align*}

Consider the $i$-th linear regression. We define
\begin{itemize}
    \item  $A := D_{\sqrt{W_i}} Y \in \R^{m \times k}$.
    \item $x := X_{i, :} \in \R^{k}$.
    \item $b := D_{\sqrt{W_i}} M_{ :,i} \in \R^{m}$. 
\end{itemize}

Recall that $Y \in \R^{m \times k}$.

For each $i \in [n]$, let $X_{i,:} \in \R^k$ denote $i$-th row of $X \in \R^{n \times k}$.
For each $i \in [n]$, let $M_{:,i} \in \R^m$ denote $i$-th column of $M \in \R^{m \times n}$.
We define $\OPT_i$
\begin{align*}
\OPT_i : = \min_{X_{i,:}\in \R^k} \| \underbrace{ D_{\sqrt{W_i}} }_{ m \times m } \underbrace{ Y }_{m \times k}  \underbrace{ X_{i,:} }_{k} - \underbrace{ D_{\sqrt{W_i}} }_{m \times m} \underbrace{ M_{:,i} }_{m} \|_2.
\end{align*}
Here $D_{\sqrt{W_i}} \in \R^{m \times m}$ is a diagonal matrix where diagonal entries are from vector $\sqrt{W_i}$ (where $W_i$ is the $i$-th column of $W\in \{0,1\}^{m \times n}$. $\sqrt{W_i}$ is entry-wise sqaure root of $W_i$.).

The trivial solution of regression is choosing $X_{i,:} \in \R^k$ to be an all zero vector (length-$k$). In this case $\OPT_i \leq \| M_{:,i}\|_2$.

Let $C > 1$ be a sufficiently large constant (The running time of the algorithm is linear in $C$). 
Let $z \in \R^k$ be the $i$-th row of the output $Z$, and let $S_i := \{ j \in [m] : W_{j,i} = 1 \}$. Since $W$ is binary, deleting the zero rows of $D_{\sqrt{W_i}} Y$ and $D_{\sqrt{W_i}} M_{:,i}$ does not change the $i$-th regression. So the call of Algorithm~\ref{alg:iter_regression} on $(Y_{S_i,*}, (M_{:,i})_{S_i})$ returns a $(1+\epsilon_1)$-approximate solution of this regression. Moreover, $\sigma_{\min}(Y_{S_i,*}) = \sigma_{\min}(D_{\sqrt{W_i}} Y) = \sigma_{\min}(D_{W_i} Y) > 0$. Using backward error Lemma~\ref{lem:forward_error}, we know that
\begin{align}\label{eq:||z-x||}
\| z - x \|_2 
\leq & ~ 2\sqrt{\epsilon_1} \cdot \frac{\OPT_i}{\sigma_{\min}(D_{\sqrt{W_i}} Y)} \notag\\
\leq & ~ \epsilon_0 \cdot \frac{1}{ (n \kappa_0)^C } \cdot \frac{ \sigma_{\max}(M) }{ \min_{j \in [n]} \sigma_{\min}(D_{\sqrt{W_j}} Y) } \notag\\
\leq & ~ \epsilon_0 \cdot \frac{1}{ n^C \kappa_0^{C-1} } \notag\\
\leq & ~ \epsilon_0 / \sqrt{n} ,
\end{align}
where the first step follows from Lemma~\ref{lem:forward_error}, the second step follows from $\epsilon_1 = \epsilon_0^2/(4(n \kappa_0)^{2C})$ and $\OPT_i \leq \| M_{:,i} \|_2 \leq \sigma_{\max}(M)$, the third step follows from the definition of $\kappa_0$ (see Lemma~\ref{lem:main_regression}), and the last step follows from $\kappa_0 \geq 1$ and $C \geq 1$. Summing over the $n$ rows, we obtain
\begin{align*}
\| Z - X \| \leq \| Z - X \|_F = \Big( \sum_{i=1}^n \| Z_{i,:} - X_{i,:} \|_2^2 \Big)^{1/2} \leq \epsilon_0 .
\end{align*}
By a union bound over the $n$ calls, each failing with probability at most $\delta_1 = \delta_0/\poly(n)$, this holds with probability at least $1 - \delta_0$.

For the running time, the $i$-th call has $|S_i|$ rows, so by Lemma~\ref{lem:alg_takes_time} it takes $O((|S_i| k \log n + k^3 \log^2(n/\delta_1)) \log(1/\epsilon_1))$ time. Forming its input takes $O(|S_i| k)$ time. Summing over $i \in [n]$ and using $\sum_{i=1}^n |S_i| = \|W\|_0$, $\log(n/\delta_1) = O(\log(n/\delta_0))$ and $\log(1/\epsilon_1) = O(\log(n \kappa_0/\epsilon_0))$ gives the claimed running time.
 \end{proof}

%% file: init.tex
\section{Initialization Conditions}
\label{sec:init_conditions}

In this section, we deal with init conditions. In Section~\ref{sub:init_conditions:bound_U_phi}, we provide the lemma which bound the distance between $U_\phi$ and $U_*$. In Section~\ref{sub:init_conditions:U_phi_close_U_*}, we show that if $U_\phi$ is close to $U_*$, then $U_0$ is close to $U_*$. In Section~\ref{sub:init_conditions:U_0_def}, we give the formal definition of $U_0$, $U_\tau$, and $U_\phi$. In Section~\ref{sub:init_conditions:main}, we combine the previous lemmas to form a new result. 

\subsection{Bounding the Distance Between \texorpdfstring{$U_\phi$}{} and \texorpdfstring{$U_*$}{}}
\label{sub:init_conditions:bound_U_phi}

We first bound the spectral norm of the sampling error $\frac{1}{q}P_{\Omega'}(M)-M$. The proof is a direct application of the matrix Bernstein inequality.

\begin{lemma}[Spectral norm of the sampling error]
\label{lem:sampled_spectral}
Let $M \in \R^{m \times n}$ be a rank-$k$, $(\mu,k)$-incoherent matrix (Definition~\ref{def:mu_k_incoherent}) with largest singular value $\sigma_1^*$, where $m \leq n$ and $n \geq 2$. Let $\eta \in (0,1]$, and let $\Omega' \subseteq [m]\times[n]$ contain each entry independently with probability $q$, where
\begin{align*}
    q \geq c_0 \cdot \frac{\mu^2 k \log n}{\eta^2 m}
\end{align*}
for a large enough absolute constant $c_0$. Then with probability at least $1-n^{-10}$,
\begin{align*}
    \Big\| \frac{1}{q} P_{\Omega'}(M) - M \Big\| \leq \eta \, \sigma_1^*.
\end{align*}
\end{lemma}

\begin{proof}
Let $M = U_* \Sigma_* V_*^\top$ be the SVD of $M$, and let $u_i$ and $v_j$ be the $i$-th row of $U_*$ and the $j$-th row of $V_*$. Let $\delta_{i,j} \in \{0,1\}$ indicate the event $(i,j) \in \Omega'$. Then
\begin{align*}
    \frac{1}{q} P_{\Omega'}(M) - M = \sum_{i \in [m]} \sum_{j \in [n]} Z_{i,j},
    \quad \text{where} \quad
    Z_{i,j} := \Big( \frac{\delta_{i,j}}{q} - 1 \Big) M_{i,j} \, e_i e_j^\top
\end{align*}
are independent $m \times n$ random matrices with $\E[Z_{i,j}] = 0$. By incoherence, $|M_{i,j}| = |u_i^\top \Sigma_* v_j| \leq \|u_i\|_2 \, \sigma_1^* \, \|v_j\|_2 \leq \mu^2 k \sigma_1^*/\sqrt{mn}$. Since $q \leq 1$, we have $|\delta_{i,j}/q - 1| \leq 1/q$, and so
\begin{align*}
    \|Z_{i,j}\| \leq L := \frac{\mu^2 k \sigma_1^*}{q \sqrt{mn}}.
\end{align*}
For the variance, $\E[(\delta_{i,j}/q-1)^2] = (1-q)/q \leq 1/q$. The matrix $\sum_{i,j} \E[Z_{i,j} Z_{i,j}^\top]$ is diagonal with $i$-th diagonal entry $\frac{1-q}{q} \|M_{i,*}\|_2^2$, and $\|M_{i,*}\|_2 = \|u_i^\top \Sigma_* V_*^\top\|_2 \leq \|u_i\|_2 \, \sigma_1^*$. Hence
\begin{align*}
    \Big\| \sum_{i,j} \E[Z_{i,j} Z_{i,j}^\top] \Big\| \leq \frac{1}{q} \max_{i \in [m]} \|u_i\|_2^2 (\sigma_1^*)^2 \leq \frac{\mu^2 k (\sigma_1^*)^2}{q m}.
\end{align*}
In the same way, $\|\sum_{i,j} \E[Z_{i,j}^\top Z_{i,j}]\| \leq \frac{1}{q}\max_{j\in[n]} \|v_j\|_2^2 (\sigma_1^*)^2 \leq \mu^2 k (\sigma_1^*)^2/(qn)$. As $m \leq n$, both are at most $\sigma^2 := \mu^2 k (\sigma_1^*)^2/(qm)$. The matrix Bernstein inequality~\cite[Theorem 1.6]{t12} gives, for $s := \eta \sigma_1^*$,
\begin{align*}
    \Pr\Big[ \Big\| \frac{1}{q} P_{\Omega'}(M) - M \Big\| \geq s \Big] \leq (m+n) \exp\Big( - \frac{s^2/2}{\sigma^2 + L s/3} \Big).
\end{align*}
Here $s^2/\sigma^2 = \eta^2 q m/(\mu^2 k) \geq c_0 \log n$, and $s/L = \eta q \sqrt{mn}/(\mu^2 k) \geq s^2/\sigma^2$ because $\eta \leq 1$ and $\sqrt{mn} \geq m$. So $L s \leq \sigma^2$, and the exponent is at least $\frac{3}{8} s^2/\sigma^2 \geq \frac{3 c_0}{8} \log n$, which is at least $12 \log n$ once $c_0 \geq 32$. The probability above is then at most $(m+n) n^{-12} \leq 2 n^{-11} \leq n^{-10}$.
\end{proof}

\begin{remark}
\label{rem:kom09}
The next lemma can be viewed as a variation of Lemma C.1 in~\cite{jns13}. The proof there cites Theorem 3.1 of the journal version of~\cite{kom09} for the bound $\|M - M_k\| \leq C \sqrt{k} \, \|M\|_F / (\sqrt{p} \, (mn)^{1/4})$ with a constant $C$, where $M_k$ is the best rank-$k$ approximation of $\frac{1}{p}P_\Omega(M)$. The corresponding result of~\cite{kom09} (Theorem III.1 there) does not give this bound in the form needed here, for three reasons. First, it is stated for a \emph{trimmed} copy of the sampled matrix, in which the rows and the columns with more than twice the average number of samples are set to zero. The algorithms of this paper do not trim the sampled matrix. The clipping in Algorithm~\ref{alg:init} acts on the rows of $U_\phi$ after the SVD. Second, its constant depends on the aspect ratio $m/n$. The theorem also treats the incoherence parameters and the singular values, after a normalization, as fixed. So, read with an absolute constant $C$, the bound above leaves the dependence on $\mu$, $k$ and $n/m$ unaccounted for. Third, its failure probability is only shown to tend to zero as $n$ grows. Lemma~\ref{lem:sampled_spectral} applies to the untrimmed matrix, keeps the dependence on $\mu$ and $k$ explicit, and fails with probability at most $n^{-10}$. It costs a $\log n$ factor in the sampling rate, which Definition~\ref{def:sampling_probability} already contains. We use it in place of the bound of~\cite{kom09}.
\end{remark}

\begin{lemma}
\label{lem:step_three}
Let $\Omega_0$ contain each entry of $[m]\times[n]$ independently with probability $p_t := p/(2T+1)$, where $p$ is as in Definition~\ref{def:sampling_probability} with a large enough constant. Write $\wt M := \frac{1}{p_t}P_{\Omega_0}(M)$, and let $\wt M = U\Sigma V^\top$ be its SVD, where $U\in \R^{m\times m}, \Sigma\in \R^{m\times m}$ and $V\in \R^{n\times m}$.
Let $U_\phi\in \R^{m\times k}$ be the $k$ columns of $U$ corresponding to the top-$k$ left singular vectors of $\wt M$ (scaling by $\frac{1}{p}$ instead of $\frac{1}{p_t}$, as the algorithms do, does not change $U_\phi$).
     Then with probability at least $1-n^{-10}$ the following hold.
     \begin{enumerate}
         \item $\|\wt M - M\| \leq \frac{\sigma_k^*}{4 \cdot 10^4 k}$.
         \item $\sigma_{k+1}(\wt M) \leq \frac{1}{2} \sigma_k(\wt M)$, where $\sigma_{k+1}(\wt M) := 0$ if $k = m$. In particular, $\sigma_{k+1}(\wt M) < \sigma_k(\wt M)$, so the top-$k$ left singular subspace of $\wt M$ is unique.
         \item $\dist(U_\phi, U_*) \leq \frac{1}{2 \cdot 10^4 k}$.
     \end{enumerate}
\end{lemma}

\begin{proof}

\textbf{Part 1.} By Definitions~\ref{def:sampling_probability} and~\ref{def:delta_2k}, $p = 10^4 C \kappa^4 \mu^2 k^{4.5} \log n \log(k\|M\|_F/\epsilon)/m$, where $\kappa = \sigma_1^*/\sigma_k^*$ and $C$ is the constant of Definition~\ref{def:sampling_probability}. Let $\eta := \frac{1}{4\cdot 10^4 k \kappa} \in (0,1]$. Since $2T+1 = O(\log(k\|M\|_F/\epsilon))$ and $\kappa, k \geq 1$, we have
\begin{align*}
    p_t \geq c_0 \cdot 16 \cdot 10^8 \cdot \frac{\mu^2 k^3 \kappa^2 \log n}{m} = c_0 \cdot \frac{\mu^2 k \log n}{\eta^2 m}
\end{align*}
when $C$ is large enough, where $c_0$ is the absolute constant of Lemma~\ref{lem:sampled_spectral}. That lemma applies to $M$ because its hypotheses hold. The matrix has rank $k$ and is $(\mu,k)$-incoherent by Assumption~\ref{assump:matrix_completion}. Moreover, $m \leq n$ by the convention of this paper, and $n \geq 2$ because $p > 0$ forces $\log n > 0$. With $q = p_t$ and $\Omega' = \Omega_0$, the lemma gives $\|\wt M - M\| \leq \eta \sigma_1^* = \frac{\sigma_k^*}{4 \cdot 10^4 k}$ with probability at least $1 - n^{-10}$. Parts 2 and 3 follow from Part 1 without further randomness.

\textbf{Part 2.} As $M$ has rank $k$, $\sigma_k(M) = \sigma_k^*$ and $\sigma_{k+1}(M) = 0$. Weyl's inequality (Lemma~\ref{lem:weyl}, applied to the transposes) and Part 1 give
\begin{align*}
    \sigma_k(\wt M) \geq \sigma_k^* - \|\wt M - M\| \geq \frac{1}{2} \sigma_k^*
    \quad \text{and} \quad
    \sigma_{k+1}(\wt M) \leq \|\wt M - M\| \leq \frac{1}{4} \sigma_k^* \leq \frac{1}{2} \sigma_k(\wt M).
\end{align*}
As $\sigma_k(\wt M) > 0$, this also gives $\sigma_{k+1}(\wt M) < \sigma_k(\wt M)$.

\textbf{Part 3.} Let $M_{\phi} : = U_{\phi} \Sigma_{\phi} V_{\phi}^\top$ be the best rank-$k$ approximation of $\wt M$, where $\Sigma_{\phi}$ holds its top $k$ singular values and the columns of $V_{\phi}$ are the matching right singular vectors, so $V_{\phi} \in \R^{n \times k}$. Then $\|\wt M - M_\phi\| = \sigma_{k+1}(\wt M) \leq \|\wt M - M\|$ by the bound on $\sigma_{k+1}(\wt M)$ in the proof of Part 2, and so
    \begin{align}\label{eq:p_sqrt_mn}
        \|M - M_{\phi}\|
        \leq \|M - \wt M\| + \|\wt M - M_\phi\|
        \leq 2 \|\wt M - M\|
        \leq \frac{1}{2 \cdot 10^4 k} \cdot \sigma_k^*.
    \end{align}

    We also have 
    \begin{align}\label{eq:sigma_k*}
        \|M - M_{\phi}\|^2 
        = & ~ \|U_* \Sigma_* V_*^\top - U_{\phi} \Sigma_{\phi} V_{\phi}^\top\|^2 \notag\\
        = & ~ \|U_* \Sigma_* V_*^\top - U_{\phi} U_{\phi}^\top U_* \Sigma_* (V_*)^\top + U_{\phi} U_{\phi}^\top U_* \Sigma_* V_*^\top - U_{\phi} \Sigma_{\phi} V_{\phi}^\top\|^2 \notag\\
        = & ~ \|(I - U_{\phi} U_{\phi}^\top) U_* \Sigma_* V_*^\top + U_{\phi} (U_{\phi}^\top U_* \Sigma_* V_*^\top - \Sigma_{\phi} V_{\phi}^\top)\|^2 \notag\\
        \geq & ~\|(I - U_{\phi} U_{\phi}^\top) U_* \Sigma_* V_*^\top\|^2 \notag\\
        = & ~ \| U_{\phi,\bot} U_{\phi,\bot}^\top U_* \Sigma_* V_*^\top \|^2 \notag\\
        = & ~ \| U_{\phi,\bot}^\top U_* \Sigma_* \|^2 \notag\\
        \geq & ~ \| U_{\phi,\bot}^\top U_* \|^2 \cdot (\sigma_{\min}( \Sigma_* ) )^2 \notag \\
        \geq & ~ (\sigma_k^*)^2 \|U_{\phi,\bot}^\top U_*\|^2,
    \end{align}
   
    where the first step follows from the SVD of $M$ (see Definition~\ref{def:mu_k_incoherent}) and the definition of $M_\phi$, the second step follows from adding and subtracting the same thing, the third step follows from simple algebra, the fourth step follows from $\| A + B \| \geq \| A \|$ if $B^\top A = 0$ (Fact~\ref{fac:norm_inequality}), the fifth step follows from $U_{\phi} U_{\phi}^\top + U_{\phi,\bot} U_{\phi,\bot}^\top =I$, the sixth step follows from applying Part 2 of Fact~\ref{fac:U_shrink_spectral_norm} twice (one for $U_{\phi,\bot}$ and one for $V_*$), the seventh step follows from $\| A B \| \geq \| A \| \cdot \sigma_{\min}(B)$ (Fact~\ref{fac:norm_inequality}), and the last step follows from $\sigma_{\min}( \Sigma_* ) = \sigma_k^*$.

    Combining Eq.~\eqref{eq:p_sqrt_mn} and Eq.~\eqref{eq:sigma_k*}, we get:
    \begin{align*}
        \|U_{\phi, \bot}^\top U_*\|_2
        \leq & ~ \frac{1}{\sigma_k^*} \cdot \| M - M_{\phi} \| \\
        \leq & ~ \frac{1}{2 \cdot 10^4 k}.
    \end{align*}
    By Definition~\ref{def:dist}, $\dist(U_\phi, U_*) = \|U_{\phi, \bot}^\top U_*\|$, which proves Part 3.
\end{proof}

\subsection{\texorpdfstring{Closeness Between $U_{\phi}$}{} and \texorpdfstring{$U_*$}{} To Closeness Between \texorpdfstring{$U_0$}{} and \texorpdfstring{$U_*$}{}}
\label{sub:init_conditions:U_phi_close_U_*}

We first define $\tau$, and then provide a lemma to show that whenever $U_{\phi}$ is close to $U_*$, $U_0$ is close to $U_*$.

\begin{definition}\label{def:tau}
We define $\tau$ as follows:
\begin{align*}
\tau : = 2 \frac{ \sqrt{k} }{ \sqrt{m} } \cdot \mu.
\end{align*}
\end{definition}

The initialization condition proof is very standard in the literature, we follow from \cite{jns13}.
\begin{lemma}
\label{lem:step_four}

    Let $U_* \in \R^{m \times k}$ is incoherent with parameter $\mu$.
    
    We define 
    \begin{align*}
         \epsilon_{\phi} : = \frac{1}{10^4 k}
    \end{align*}
    Let $U_{\phi}$ be an orthonormal column matrix such that 
    \begin{align*}
        \dist(U_{\phi}, U_*) \leq \epsilon_{\phi}.
    \end{align*}
     Let $U_{\tau}$ be obtained from $U_{\phi} \in \R^{m \times k}$ by setting all rows with norm greater than $\tau$ to zero. 
    Let $U_0$ 
    be an orthonormal basis of $U_{\tau}$. 
    
    Then, we have
    \begin{itemize}
        \item Part 1. $\dist (U_0, U_*) \leq 1/2$ 
        \item Part 2. $U_0$ is incoherent with parameter $4 \mu \sqrt{k}$.
    \end{itemize}
\end{lemma}
\begin{remark}\label{rem:base_of_induction}
For the convenience of matching later induction proof, when we use this lemma for the base in induction, we write $\wh{U}_0$ to denote $U_{\tau}$.
\end{remark}
\begin{proof}

{\bf Proof of Part 1.} For simplicity, in the proof, we use $\epsilon$ to denote $\epsilon_{\phi}$.

    Let $S := \{ j \in [m] : \|(U_{\phi})_{j,*}\|_2 > \tau \}$ be the set of zeroed rows. Then $U_{\ov{\tau}} := U_{\phi} - U_{\tau}$ equals $U_{\phi}$ on the rows in $S$ and is zero elsewhere. Let $Z := U_* U_*^\top U_{\phi}$. For every $j \in [m]$, $\|Z_{j,*}\|_2 \leq \|(U_*)_{j,*}\|_2 \cdot \|U_*^\top U_{\phi}\| \leq \mu\sqrt{k}/\sqrt{m} = \tau/2$ by the incoherence of $U_*$, so for $j \in S$,
    \begin{align*}
        \|(U_{\phi})_{j,*} - Z_{j,*}\|_2 \geq \|(U_{\phi})_{j,*}\|_2 - \frac{\tau}{2} \geq \frac{1}{2}\|(U_{\phi})_{j,*}\|_2 .
    \end{align*}

    We can show that 
    \begin{align}\label{eq:spectral_norm_U_ov_tau}
       \| U_{\ov{\tau}} \|^2 \leq & ~ \| U_{\ov{\tau}} \|_F^2 \notag\\
       = & ~ \sum_{j \in S} \| (U_{\phi})_{j,*} \|_2^2 \notag\\
       \leq & ~ 4 \sum_{j \in S} \| (U_{\phi})_{j,*} - Z_{j,*} \|_2^2 \notag\\
       \leq & ~ 4 \| (I - U_* U_*^\top) U_{\phi} \|_F^2 \notag\\
       \leq & ~ 4 k \, \| U_{*,\bot}^\top U_{\phi} \|^2 \notag\\
       \leq & ~ 4 \epsilon^2 k,
    \end{align}
    where the first step follows from $\| \cdot \| \leq \| \cdot \|_F$, the second step follows from the definition of $S$, the third step follows from the bound just shown, the fourth step follows from $U_{\phi} - Z = (I - U_* U_*^\top) U_{\phi}$ and the fact that the rows in $S$ are some of its rows, the fifth step follows from $(I - U_* U_*^\top) U_{\phi} = U_{*,\bot} U_{*,\bot}^\top U_{\phi}$, which has rank at most $k$, and the last step follows from $\dist(U_{\phi}, U_*) = \|U_{*,\bot}^\top U_{\phi}\| \leq \epsilon$ (Definition~\ref{def:dist}).
    
    Let 
    \begin{align}\label{eq:qr_decomposition}
        U_{\tau} = U_0 \Lambda^{-1}
    \end{align}
    be the QR decomposition.

    Then, for any $u_{*, \bot}\in\mathrm{span} (U_{*,\bot})$ with $\|u_{*,\bot}\|_2 = 1$, we have:
    \begin{align}\label{eq:u*_bot}
        \|u_{*,\bot}^\top U_0\|_2 
        = & ~ \|u_{*,\bot}^\top U_{\tau} \Lambda\|_2 \notag\\
        \leq & ~ \|u_{*,\bot}^\top U_{\tau }\|_2 \|\Lambda\|,
\end{align}
   where the first step follows from Eq.~\eqref{eq:qr_decomposition} and the second step follows from $\|Ax\|_2 \leq \|A\| \cdot \|x\|_2$.

For the first term in the above equation
\begin{align}\label{eq:3_sqrt_kd}
       \|u_{*,\bot}^\top U_{\tau}\|_2 
       \leq & ~ \|u_{*,\bot}^\top U_{\phi}\|_2 + \|u_{*,\bot}^\top U_{\ov{\tau}}\|_2 \notag\\
       \leq & ~ \epsilon +  \|u_{*,\bot}^\top U_{\ov{\tau}}\|_2 \notag\\
        \leq & ~ \epsilon + \|U_{\ov{\tau}}\| \notag\\
        \leq & ~ \epsilon + 2\sqrt{k} \epsilon \notag\\
        \leq & ~ 3\sqrt{k} \epsilon ,
    \end{align}
  where the first step follows from triangle inequality, the second step follows from $\|u_{*,\bot}^\top U_{\phi}\|_2 \leq \|U_{*,\bot}^\top U_{\phi}\| = \dist(U_{\phi}, U_*) \leq \epsilon$ for the unit vector $u_{*,\bot}$, the third step follows from the definition of $\| \cdot \|$ and $\|u_{*,\bot}\|_2 = 1$, the fourth step follows from Eq.~\eqref{eq:spectral_norm_U_ov_tau}, and the last step follows from $k \geq 1$.
   
    We now bound $\|\Lambda\|$ as follow:
    \begin{align}\label{eq:4/3}
        \|\Lambda\|^2
        = & ~ \frac{1}{\sigma_{\min} (\Lambda^{-1})^2} \notag\\
        = & ~ \frac{1}{\sigma_{\min} (U_0 \Lambda^{-1})^2} \notag\\
        = & ~ \frac{1}{\sigma_{\min} (U_{\tau})^2} \notag\\
        = & ~ \frac{1}{1 - \|U_{\ov{\tau}}\|^2} \notag\\
        \leq & ~ \frac{1}{1 - 2\sqrt{k}\epsilon} \notag\\
        \leq & ~ 2,
    \end{align}
    where the first step follows from Fact~\ref{fac:norm_inequality}, the second step follows from the fact that $U_0$ forms an orthonormal basis (see Part 2 of Fact~\ref{fac:U_shrink_spectral_norm}), the third step follows from the QR decomposition of $U_\tau$ (Eq.~\eqref{eq:qr_decomposition}), the fourth step follows from $U_\tau^\top U_\tau = U_\phi^\top U_\phi - U_{\ov\tau}^\top U_{\ov\tau} = I_k - U_{\ov\tau}^\top U_{\ov\tau}$ (as $U_\tau$ and $U_{\ov\tau}$ have disjoint row supports), the fifth step follows from Eq.~\eqref{eq:spectral_norm_U_ov_tau}, and the last step follows from using the fact that $\epsilon < \frac{1}{100 k}$.

    Thus, we have:
    \begin{align*}
        \|u_{*,\bot}^\top U_0\|_2
        \leq & ~ \|u_{*,\bot}^\top U_{\tau }\|_2 \|\Lambda\|\\
        \leq & ~ 3\sqrt{k} \epsilon \cdot \|\Lambda\|\\
        \leq & ~ 3\sqrt{k} \epsilon \cdot 2\\
        = & ~ 6 \sqrt{k} \epsilon \\
        \leq & ~ \frac{1}{10},
    \end{align*}
    where the first step follows from Eq.~\eqref{eq:u*_bot}, the second step follows from Eq.~\eqref{eq:3_sqrt_kd}, the third step follows from Eq.~\eqref{eq:4/3},  
    the fourth step follows from simple algebra, and the last step follows from the fact that $\epsilon \leq \frac{1}{10^4 k}$.
    
    This proves the first part of the lemma.

    {\bf Proof of Part 2.}
    The incoherence of $U_0$ (see incoherence in Definition~\ref{def:U_is_mu}) is
    \begin{align*}
         \frac{\sqrt{m}}{\sqrt{k}} \max_{i\in [m]} \|e_i^\top U_0\|_2
        \leq & ~ \frac{\sqrt{m}}{\sqrt{k}} \max_{i\in [m]} \|e_i^\top U_{\tau} \Lambda\|_2\\
        \leq & ~ \frac{\sqrt{m}}{\sqrt{k}} \max_{i\in [m]} \|e_i^\top U_{\tau}\|_2 \|\Lambda\|\\
        \leq & ~ \frac{2\sqrt{m}}{\sqrt{k}} \max_{i\in [m]} \|e_i^\top U_{\tau}\|_2\\
        \leq & ~ \frac{2\sqrt{m}}{\sqrt{k}} \cdot \tau \\
        = & ~ \frac{2\sqrt{m}}{\sqrt{k}} \cdot 2 \mu \sqrt{k}/\sqrt{m} \\
        = & ~ 4\mu,
    \end{align*}
    where the first step follows from the definition of incoherence, the second step follows from $\|Ax\|_2 \leq \|A\| \cdot \|x\|_2$, the third step follows from Eq~\eqref{eq:4/3}, the fourth step follows from $U_{\tau}$ is truncating at $\tau$, and the fifth step follows from $\tau = 2 \mu \sqrt{k}/\sqrt{m}$ (see Definition~\ref{def:tau}).
\end{proof}

\subsection{Definitions for \texorpdfstring{$U_0$, $U_{\tau}$, $U_{\phi}$}{}}
\label{sub:init_conditions:U_0_def}

We provide the definitions for $U_0$, $U_{\tau}$, and $U_{\phi}$. We also include a procedure that clips rows with large norm then perform Gram-Schmidt.

\begin{definition}\label{def:U_0}

Let $\frac{1}{p}P_{\Omega_0}(M)=U\Sigma V^\top$ be the SVD of the matrix $\frac{1}{p}P_{\Omega_0}(M)$, where $U\in \R^{m\times m}, \Sigma\in \R^{m\times n}$ and $V\in \R^{n\times n}$. 

Let $\tau:= \frac{2\mu\sqrt{k}}{\sqrt{m}}$.

Let $U_{\phi}\in \R^{m\times k}$ be the $k$ columns of $U$ corresponding to the top-$k$ left singular vectors of $\frac{1}{p}P_{\Omega_0}(M)$. By Part 2 of Lemma~\ref{lem:step_three}, this subspace is unique with high probability. The algorithms compute it approximately, as the proof of Theorem~\ref{thm:main} explains.

We define $U_{\tau}\in \R^{m\times k}$ as follows:

\begin{align*}
    U_{\tau, i,*} := & ~\begin{cases}
    U_{\phi,i,*}, & \text{if $\|U_{\phi,i,*} \|_2 \leq \tau$}, \\
    0, & \text{otherwise}.
    \end{cases}
\end{align*}

Finally, we define $U_0 \in \R^{m\times k}$ to be the matrix formed from performing Gram-Schmidt on $U_{\tau}$, i.e., columns of $U_0$ are orthonormal.

To make convenient of base case of induction proof, we call $U_{\tau}$ to be $\wh{U}_0$.
\end{definition}

\begin{algorithm}[!ht]
\caption{Clipping and Gram-Schmidt}
\label{alg:init}
\begin{algorithmic}[1]
\Procedure{Init}{$U_{\phi}\in \R^{m\times k}$, $\mu\geq 1$}
\State $\tau \gets \frac{2\mu\sqrt{k}}{\sqrt{m}}$
\For{$i=1\to m$}
\State $U_{\tau,i,*}=\begin{cases}
U_{\phi,i,*}, & \text{if $\|U_{\phi,i,*}\|_2\leq \tau$,} \\
0, & \text{otherwise.}
\end{cases}$
\EndFor
\State $(Q, R)\gets \textsc{QR}(U_{\tau})$
\State $U_0\gets Q$
\State \Return $U_0, U_{\tau}$
\EndProcedure
\end{algorithmic}
\end{algorithm}

\subsection{Initialization Condition: Main Result}
\label{sub:init_conditions:main}

We provide a Lemma which bounds $\dist(U_0, U_*)$ and shows that $U_0$ is incoherent with parameter $4\mu\sqrt{k}$.
 
\begin{lemma} 
\label{lem:main_init}
    Let $M \in \R^{m \times n}$ be a $(\mu,k)$-incoherent matrix (see Definition~\ref{def:mu_k_incoherent}). 
    
    Let $\Omega \subset [m] \times [n]$ and $p \in (0,1)$ be  defined as Definition~\ref{def:sampling_probability}, with a large enough constant there.
    
    Let $U_0 \in \R^{m \times k}$ be defined as Definition~\ref{def:U_0}, computed from the sample $\Omega_0$ of Algorithm~\ref{alg:alternating_minimization_completion_formal}, which contains each entry independently with probability $p_t = p/(2T+1)$. 
    
    Then, we have the following properties
    \begin{itemize}
        \item $\dist(U_0, U_*) \leq \frac{1}{2}$ and
        \item $U_0$ is incoherent with parameter $4\mu\sqrt{k}$.
    \end{itemize}
    holds with probability at least $1-1/\poly(n)$.
\end{lemma}
\begin{proof}

From Part 3 of Lemma~\ref{lem:step_three}, applied to the sample $\Omega_0$ (which has rate $p_t$), we see that $U_{\phi}$ satisfies
\begin{align*}
\dist (U_{\phi}, U_*) \leq \frac{1}{2 \cdot 10^4 k} \leq \frac{1}{10^4 k}.
\end{align*}

Using Lemma~\ref{lem:step_four}, we finish the proof. Lemma~\ref{lem:step_four} applies to any orthonormal $U_\phi$ with $\dist(U_\phi, U_*) \leq \frac{1}{10^4 k}$. By the running-time part of the proof of Theorem~\ref{thm:main}, the approximation of $U_\phi$ that Algorithm~\ref{alg:alternating_minimization_completion_formal} computes satisfies this bound, so the same conclusion holds for it.
\end{proof}

%% file: induction.tex
\section{Main Induction Hypothesis}
\label{sec:main_induction}

In this section, we provide our main induction hypothesis: we inductively bound $\dist(\wh V_{t+1}, V^*), \dist(\wh U_{t+1}, U^*)$ and the incoherence of $U_{t+1}, V_{t+1}$. We start by presenting our formal algorithm.

\begin{algorithm}[!ht]
\caption{Alternating minimization for matrix completion, formal version of Algorithm~\ref{alg:alternating_minimization_completion}.
} \label{alg:alternating_minimization_completion_formal} 
\begin{algorithmic}[1]
    \Procedure{FastMatrixCompletion}{$\Omega \subset [m] \times [n]$, $P_\Omega(M), k, \epsilon, \delta, \kappa, \|M\|_F$}  \Comment{Theorem~\ref{thm:main}} 
    \State Let $\epsilon \in (0,1)$ denote the final accuracy of algorithm.
    \State Let $\delta \in (0,1)$ denote the final failure probability.
    \State Let $\kappa$ and $\|M\|_F$ denote the condition number of $M$ and its Frobenius norm, or upper bounds on them.
    \State $T \gets \Theta(\log(k\|M\|_F/\epsilon))$
    \State \Comment{Draw independent $\Omega_0, \cdots, \Omega_{2T} \subseteq \Omega$, each containing every entry independently with probability $p_t := p/(2T+1)$}  
    \State \Comment{This is possible since $1-(1-p_t)^{2T+1} \leq p$, and the sets may overlap}  
    \State $U_{\phi} = \mathrm{SVD}(\frac{1}{p} P_{\Omega_0} (M), k)$ \Comment{$U_{\phi} \in \R^{m \times k}$, computed approximately by subspace iteration (proof of Theorem~\ref{thm:main})} 
    
    \State $U_0, \wh{U}_0 \gets \textsc{Init}(U_\phi)$ \Comment{Algorithm~\ref{alg:init}}
    
    \State $\epsilon_0 \gets \epsilon / (C' n^3 \kappa^5)$ \Comment{$C'$ is a large enough absolute constant}
    \State $\delta_0 \gets \delta/\poly(n, T)$
    \For{$t = 0, \cdots, T-1$}  
        \State $U_t \gets$ orthonormal factor of $\wh U_t$ (QR) \Comment{for $t = 0$, this is the $U_0$ returned by \textsc{Init}}
        \State $\wh{V}_{t + 1} \gets \textsc{FastMultReg}(M \in \R^{m \times n},U_t \in \R^{m \times k}, \Omega_{2t+1},m,n,k,\epsilon_0, \delta_0) $ \State \Comment{Alg.~\ref{alg:multiple_regression} , Lem.~\ref{lem:main_regression}. Here $\Omega_{2t+1}$ can be viewed as $m \times n$ weight matrix.} 
        \State $V_{t+1} \gets$ orthonormal factor of $\wh V_{t+1}$ (QR)
        \State $\wh{U}_{t + 1} \gets \textsc{FastMultReg}(M^\top \in \R^{n \times m} ,V_{t+1} \in \R^{n \times k},\Omega_{2t+2}^\top ,n,m,k,\epsilon_0, \delta_0) $ \State \Comment{Alg.~\ref{alg:multiple_regression} , Lem.~\ref{lem:main_regression}. Here $\Omega_{2t+2}^\top$ can be viewed as $n \times m$ weight matrix.}   
    \EndFor  
    \State $X \gets \wh{U}_T V_T^\top$ \Comment{$V_T$ is the orthonormal factor of $\wh V_T$}
    \State \Return $X $ \Comment{$X \in \R^{m \times n}$}
     
\EndProcedure
\end{algorithmic}
\end{algorithm}

\begin{lemma}[Induction hypothesis]\label{lem:induction}
Let $M= U_* \Sigma_* V_*^\top$ be defined as Definition~\ref{def:mu_k_incoherent}. Let $U_t$, $\wh{V}_{t+1}$, $V_{t+1}$ and $\wh{U}_{t+1}$ be the orthonormal factors and exact regression solutions of Definition~\ref{def:update_rule}, where the regression for $\wh{V}_{t+1}$ uses the samples $\Omega_{2t+1}$ and the regression for $\wh{U}_{t+1}$ uses $\Omega_{2t+2}$. Each of these samples contains every entry independently with probability $p_t = p/(2T+1)$, with $p$ as in Definition~\ref{def:sampling_probability}, and is independent of the fixed factor of its regression. For all $t \in [T]$, we have the following results.

Part 1. If $U_t$ is $\mu_2$ incoherent and $\dist(U_t, U_*) \leq 1/2$, then we have
    \begin{itemize}
        \item $\dist(\wh{V}_{t + 1}, V_*) 
        \leq   \frac{1}{4}\dist(U_t, U_*)$ 
    \end{itemize}
    
Part 2. If $U_t$ is $\mu_2$ incoherent and $\dist(U_t, U_*) \leq 1/2$, then we have
\begin{itemize}
    \item $V_{t+1}$ is $(\mu_2/2)$ incoherent 
\end{itemize}
    
Part 3. If $V_{t+1}$ is $\mu_2$ incoherent and $\dist(V_{t+1}, V_*)\leq 1/2$, then we have
    \begin{itemize}
        \item $\dist(\wh{U}_{t + 1}, U_*) 
        \leq  \frac{1}{4}\dist(V_{t + 1}, V_*)$
    \end{itemize}

Part 4. If $V_{t+1}$ is $\mu_2$ incoherent and $\dist(V_{t+1}, V_*)\leq 1/2$, then we have
\begin{itemize}
    \item $U_{t+1}$ is $(\mu_2/2)$ incoherent.
\end{itemize}
The above results hold with probability at least $1 - 1/\poly(n)$.
\end{lemma}
\begin{proof}

{\bf Proof of Part 1.} To prove this part, we need to use Lemma~\ref{lem:norm_of_F_Sigma^*^-1}, and Lemma~\ref{lem:norm_of_Sigma^*_R^-1}. To use Lemma~\ref{lem:norm_of_F_Sigma^*^-1}, we need the condition that $U_t$ is $\mu_2$ incoherent and independent of $\Omega_{2t+1}$. To use Lemma~\ref{lem:norm_of_Sigma^*_R^-1}, we need two conditions: one is that $U_t$ be a $\mu_2$-incoherent and the other is $\dist(U_t,U_*) \leq 1/2$.

Since $\wh{V}_{t+1}$ and $V_{t+1}$ span the same subspace, $\dist(\wh{V}_{t+1}, V_*) = \dist(V_{t+1}, V_*)$. We can show that
\begin{align}\label{eq:dist_V*_Vt+1}
\dist(V_*, V_{t+1})
= & ~ \|  V_{*,\bot}^\top V_{t+1} \| \notag \\
= & ~ \|  V_{*,\bot}^\top ( V_* \Sigma_* U_*^\top U_t - F) \cdot R_{t+1}^{-1} \| \notag \\
= & ~ \|  V_{*,\bot}^\top  V_* \Sigma_* U_*^\top U_t R_{t+1}^{-1} - V_{*,\bot}^\top F  R_{t+1}^{-1} \| \notag \\
= & ~ \|  V_{*,\bot}^\top F \cdot R_{t+1}^{-1} \| \notag\\
\leq & ~ \| F \| \cdot \| R_{t+1}^{-1} \|,
\end{align}
where the first step follows from the definition of distance (Definition~\ref{def:dist}), the second step follows from the definition of $V_{t+1}$ ($V_{t+1} = ( V_* \Sigma_* U_*^\top U_t - F) \cdot R_{t+1}^{-1}$, see Definition~\ref{def:update_of_V_wh_V}), the third step follows from simple algebra, the fourth step follows from $V_{*,\bot}^\top V_* = 0$, and the fifth step follows from Fact~\ref{fac:U_shrink_spectral_norm} and $\| A \cdot B \| \leq \|A \| \cdot \| B \|$ (Fact~\ref{fac:norm_inequality}).

By Eq.~\eqref{eq:dist_V*_Vt+1}, we can upper bound
\begin{align*}
\dist(V_*, V_{t+1} )
\leq & ~ \| F \| \cdot \| R_{t+1}^{-1} \| \\
\leq & ~ 2 \delta_{2k} \cdot k \cdot \sigma_1^* \cdot \dist(U_t,U_*) \cdot \| R_{t+1}^{-1} \| \\
\leq & ~ 2 \delta_{2k} \cdot k \cdot \sigma_1^* \cdot \dist(U_t,U_*) \cdot 2 /\sigma_k^* \\
\leq & ~ \frac{1}{4} \dist(U_t,U_*),
\end{align*}
where the second step follows from Lemma~\ref{lem:norm_of_F_Sigma^*^-1}, the third step follows from Part 2 of Lemma~\ref{lem:norm_of_Sigma^*_R^-1} ($\sigma_{\min}(R_{t+1}) \geq \sigma_k^*/2$), the last step follows from Definition~\ref{def:delta_2k} ($4 \delta_{2k} k \sigma_1^*/\sigma_k^* = 0.04$).

{\bf Proof of Part 2.} In this proof, we need to use Lemma~\ref{lem:norm_of_Sigma^*_R^-1}, Lemma~\ref{lem:bound_B^-1},
Lemma~\ref{lem:bound_C_j}. To use Lemma~\ref{lem:norm_of_Sigma^*_R^-1}, we need the condition that $U_t$ is $\mu_2$ incoherent and the other is $\dist(U_t,U_*) \leq 1/2$. To use Lemma~\ref{lem:bound_B^-1}, we need the condition that $U_t$ is $\mu_2$ incoherent. To use Lemma~\ref{lem:bound_C_j}, we need the condition that $U_t$ is $\mu_2$ incoherent.
  
    For each $j \in [n]$, according to Definition~\ref{def:update_of_V_wh_V}, we have
    \begin{align*}
        V_{t+1,j} = R_{t+1}^{-1} ( D_j - B_j^{-1} ( B_j D_j - C_j) ) \Sigma_* V_{*,j}.
    \end{align*}
    Therefore,
    \begin{align}\label{eq:rewrite_V_t+1_j}
        \| V_{t+1,j} \|_2 \leq \frac{ \sigma_1^* }{ \sigma_{\min} (R_{t+1}) } \cdot \| V_{*,j} \|_2 \cdot  ( \| D_j \| + \| B_j^{-1} ( B_j D_j - C_j ) \| ) 
    \end{align}
    For the first term of Eq.~\eqref{eq:rewrite_V_t+1_j}, we have 
    \begin{align}\label{eq:sigma1_over_sigma_min}
        \frac{ \sigma_1 }{ \sigma_{\min} (R_{t+1}) } \leq 2 \sigma_1^*/\sigma_k^*
    \end{align}
    where it follows from Part 2 of Lemma~\ref{lem:norm_of_Sigma^*_R^-1}.

    For the second term of Eq.~\eqref{eq:rewrite_V_t+1_j}, we have
    \begin{align}\label{eq:v*j2}
        \| V_{*,j} \|_2 \leq \mu \sqrt{k} / \sqrt{n}
    \end{align}    

    For the third term of Eq.~\eqref{eq:rewrite_V_t+1_j}, we have
    \begin{align}\label{eq:dj_and_bj-1}
        \| D_j \| + \| B_j^{-1} ( B_j D_j - C_j ) \| 
        \leq & ~ \| D_j \| + \| B_j^{-1} \| \cdot \| B_j D_j - C_j \| \notag \\
        \leq & ~ 1 +  \| B_j^{-1} \| \cdot \| B_j D_j - C_j \| \notag \\
        \leq & ~ 1 + 2 \cdot \| B_j D_j - C_j \| \notag \\
        \leq & ~ 1 + 2 \cdot ( \| B_j D_j \| + \| C_j \| ) \notag \\
        \leq & ~ 1 + 2 \cdot ( 2 + 2 ) \notag \\
        \leq & ~ 10,
    \end{align}
    where the first step follows from $\| A B \| \leq \| A \| \cdot \| B \|$, the second step follows from $\| D_j \| \leq 1$, the third step follows from Lemma~\ref{lem:bound_B^-1}, the fourth step follows from $\| A +  B \| \leq \| A \| + \| B \|$, the fifth step follows from Part 3 of Lemma~\ref{lem:bound_B^-1} and $\|D_j\| \leq 1$ (for $\| B_j D_j \| \leq 2$) and from Lemma~\ref{lem:bound_C_j} (for $\| C_j \| \leq 2$), and the last step follows from simple algebra. 
    
    Combining Eq.~\eqref{eq:rewrite_V_t+1_j}, Eq.~\eqref{eq:sigma1_over_sigma_min}, Eq.~\eqref{eq:v*j2}, and Eq.~\eqref{eq:dj_and_bj-1}, we have
    \begin{align*}
        \| V_{t+1,j} \|_2 \leq & ~ 2 \frac{\sigma_1^*}{ \sigma_k^* } \cdot \mu \frac{\sqrt{k}}{\sqrt{n}} \cdot 10 \\
        \leq & ~ \frac{\mu_2}{2} \cdot \frac{\sqrt{k}}{\sqrt{n}},
    \end{align*}
    where the second step follows from Definition~\ref{def:mu_2} (since $20 \kappa \mu \leq 20 \kappa \mu \sqrt{k} = \mu_2/2$).

{\bf Proof of Part 3 and Part 4.} By a symmetric argument (exchanging the roles of $U$ and $V$, of $m$ and $n$, and using the samples $\Omega_{2t+2}$, which are independent of $V_{t+1}$), we can also prove
\begin{align*}
\dist (U_{t+1} , U_*) \leq \frac{1}{4} \dist (V_{t+1}, V_*).
\end{align*}
and $U_{t+1}$ is $(\mu_2/2)$ incoherent.
\end{proof}

We are now ready to prove the main theorem of this paper (Theorem~\ref{thm:main})

\begin{proof}[Proof of Theorem~\ref{thm:main}]
We analyze Algorithm~\ref{alg:alternating_minimization_completion_formal}.

\textbf{Correctness.} Let $p_t := p/(2T+1)$ and assume $\epsilon < k\|M\|_F$ (otherwise $\log(k\|M\|_F/\epsilon) \leq 0$ and the sample bound is vacuous). Let $U_t$ be the orthonormal factor of the algorithm's $\wh U_t$, where $U_0$ comes from Algorithm~\ref{alg:init}, and let $V_{t+1}$ be that of its $\wh V_{t+1}$. Write $d_{2t} := \dist(U_t, U_*)$ and $d_{2t+1} := \dist(V_{t+1}, V_*)$. The algorithm takes $\epsilon_0 = \epsilon/(C' n^3 \kappa^5)$ for a large constant $C'$. Since $\epsilon < k\|M\|_F \leq k^{1.5}\sigma_1^*$ and $k \leq n$, we have $\epsilon_0 \leq \sigma_k^*/(C' n^{1.5}\kappa^4)$.

By Lemma~\ref{lem:main_init}, $d_0 \leq 1/2$ and $U_0$ is $4\mu\sqrt{k}$-incoherent, hence $\mu_2$-incoherent. Suppose $d_{2t} \leq 1/2$ and $U_t$ is $\mu_2$-incoherent. Since $U_t$ depends only on $\Omega_0, \ldots, \Omega_{2t}$, Lemma~\ref{lem:induction} for the samples $\Omega_{2t+1}$ shows that the exact minimizer $X$ of $\|P_{\Omega_{2t+1}}(U_t V^\top - M)\|_F$ satisfies $\dist(X, V_*) \leq d_{2t}/4$ and that its orthonormal factor is $(\mu_2/2)$-incoherent. Lemmas~\ref{lem:norm_of_Sigma^*_R^-1} and~\ref{lem:norm_of_F_Sigma^*^-1} give $\sigma_{\min}(X) \geq \sigma_k^*/2$ and $\sigma_{\max}(X) \leq \sigma_1^* + \|F\| \leq 1.1\sigma_1^*$, so $\kappa(X) \leq 2.2\kappa$ and $\epsilon_0 \leq 0.1\sigma_{\min}(X)$. By Lemma~\ref{lem:bound_B^-1}, for every $j \in [n]$ we have $\sigma_{\min}(D_{W_j} U_t)^2 = p_t \sigma_{\min}(B_j) \geq p_t/2 > 0$, where $W$ is the $0/1$ matrix of $\Omega_{2t+1}$. So the full column rank hypothesis of Lemma~\ref{lem:main_regression} holds, and its $\kappa_0$ satisfies $\kappa_0 \leq \max\{1, \sigma_1^*\sqrt{2/p_t}\}$. Since $\|\wh V_{t+1} - X\| \leq \epsilon_0$ (Lemma~\ref{lem:main_regression}), Lemma~\ref{lem:induction_control_eps_sketch} (with $n$ in place of $m$) gives $d_{2t+1} \leq d_{2t}/4 + \epsilon_{\rm sk}$ and shows that $V_{t+1}$ is $(\mu_2/2 + \epsilon_{\rm sk})$-incoherent. Here $\epsilon_{\rm sk} = O(\epsilon_0 n \kappa^4/\sigma_k^*) = O(1/(C'\sqrt{n})) \leq 1/8 \leq \mu_2/2$ for $C'$ large. The same argument with $U$ and $V$ exchanged handles $\wh U_{t+1}$. Hence $d_s \leq 4^{-s}/2 + \frac{4}{3}\epsilon_{\rm sk}$ for all $s \leq 2T$. For $1 \leq s \leq 2T$ this is at most $1/8 + 1/6 < 1/2$, and $d_0 \leq 1/2$ by Lemma~\ref{lem:main_init}. So $d_s \leq 1/2$ for all $s \leq 2T$, and all incoherence parameters stay at most $\mu_2$.

Let $X'$ be the exact minimizer of $\|P_{\Omega_{2T}}(U V_T^\top - M)\|_F$. By Definition~\ref{def:update_of_V_wh_V} with $U$ and $V$ exchanged, $X' = U_*\Sigma_*V_*^\top V_T - F'$, and Lemma~\ref{lem:norm_of_F_Sigma^*^-1} gives $\|F'\|_F \leq 2\delta_{2k}k\sigma_1^* d_{2T-1} \leq \sigma_1^* d_{2T-1}$. Since $M - X' V_T^\top = U_*\Sigma_*V_*^\top(I - V_TV_T^\top) + F' V_T^\top$ and $U_*\Sigma_*V_*^\top(I - V_TV_T^\top)$ has rank at most $k$ and spectral norm at most $\sigma_1^* d_{2T-1}$, we get
\begin{align*}
\|M - \wh U_T V_T^\top\|_F \leq \|U_*\Sigma_*V_*^\top(I - V_TV_T^\top)\|_F + \|F'\|_F + \sqrt{k}\|\wh U_T - X'\| \leq 2\sqrt{k}\sigma_1^* d_{2T-1} + \sqrt{k}\epsilon_0 .
\end{align*}
Using $d_{2T-1} \leq 4^{-(2T-1)}/2 + \frac{4}{3}\epsilon_{\rm sk}$, choose $T = \Theta(\log(k\|M\|_F/\epsilon))$ so that $\sqrt{k}\sigma_1^* 4^{-(2T-1)} \leq \epsilon/3$. Moreover, $\frac{8}{3}\sqrt{k}\sigma_1^*\epsilon_{\rm sk} = O(\epsilon\sqrt{k}/(C'n^2)) \leq \epsilon/3$ and $\sqrt{k}\epsilon_0 \leq \epsilon/3$. Then the right-hand side is at most $\epsilon$. A union bound over the $2T+1$ rounds bounds the failure probability.

\textbf{Running time.} The $2T$ calls of Lemma~\ref{lem:main_regression} take $\wt O(|\Omega| k)$ time in total, since $\sum_t |\Omega_t| = O(|\Omega|)$ with high probability and $\log(n\kappa_0/\epsilon_0) = O(\log(n\kappa k\|M\|_F/\epsilon))$ by the bound on $\kappa_0$ above. The QR steps take $O((m+n)k^2)$ time per iteration. The subspace $U_\phi$ is computed approximately, by block subspace iteration from a random start. Let $\wt M := \frac{1}{p_t}P_{\Omega_0}(M)$ and $A := \frac{1}{p^2}P_{\Omega_0}(M)P_{\Omega_0}(M)^\top = (p_t/p)^2\, \wt M \wt M^\top$. The iteration starts from an orthonormal basis $X_0 \in \R^{m\times k}$ of a uniformly random $k$-dimensional subspace. Each round multiplies the current basis by $A$, using one product with the sparse matrix $P_{\Omega_0}(M)^\top$ and one with $P_{\Omega_0}(M)$, and orthonormalizes the result by QR. A round takes $O(|\Omega_0| k + mk^2)$ time. The top-$k$ eigenvectors of $A$ span the top-$k$ left singular subspace of $\wt M$, and the eigenvalues of $A$ are the squared singular values of $\wt M$ up to a common factor. By Part 2 of Lemma~\ref{lem:step_three}, which holds with probability at least $1-n^{-10}$, the $(k+1)$-st eigenvalue of $A$ is therefore at most $1/4$ times the $k$-th. Part 2 rests on the spectral bound $\|\wt M - M\| \leq \frac{\sigma_k^*}{4\cdot 10^4 k}$ of Part 1, which Lemma~\ref{lem:sampled_spectral} proves with the matrix Bernstein inequality (Remark~\ref{rem:kom09} explains why we do not take it from~\cite{kom09}). We apply the convergence bound for the noisy power method~\cite[Corollary 1.1]{hp14} to $A$ with no noise, block size $k$, accuracy $\frac{1}{2\cdot 10^4 k}$ and failure parameter $\gamma = n^{c}$ for a large enough constant $c$. That paper calls this parameter $\tau$, which is unrelated to the threshold $\tau$ of Definition~\ref{def:tau}. For this accuracy, the round count of the bound is $O(\frac{\lambda_k}{\lambda_k - \lambda_{k+1}} \log(m \gamma k))$, where $\lambda_k$ and $\lambda_{k+1}$ are the $k$-th and $(k+1)$-st eigenvalues of $A$. As $\lambda_k \geq 4 \lambda_{k+1}$, the factor $\frac{\lambda_k}{\lambda_k - \lambda_{k+1}}$ is at most $4/3$. So after $O(\log(m \gamma k)) = O(\log n)$ rounds, the computed basis is within distance $\frac{1}{2\cdot 10^4 k}$ of the exact top-$k$ left singular subspace of $\wt M$, except with probability $\gamma^{-\Omega(1)} + e^{-\Omega(m)} \leq 1/\poly(n)$. Here $e^{-\Omega(m)} \leq 1/\poly(n)$ because $p_t \leq 1$ forces $m = \Omega(\log n)$ by Definition~\ref{def:sampling_probability}. By Part 3 of Lemma~\ref{lem:step_three}, the exact subspace is within $\frac{1}{2\cdot 10^4 k}$ of $U_*$. By the triangle inequality for $\dist$, the computed $U_\phi$ then satisfies $\dist(U_\phi,U_*)\leq \frac{1}{10^4 k}$. This bound is all that Lemma~\ref{lem:step_four} needs, and hence all that Lemma~\ref{lem:main_init} needs. \textsc{Init} takes $O(mk^2)$ time.
\end{proof}

%% file: rankk.tex
\section{General Case Update Rules and Notations}
\label{sec:update_rules_and_notations}

In this section, we review various properties of rank-$k$ updates due to our algorithm. In Section~\ref{sub:update_rules_and_notations:update}, we provide the formal definition of the updated rule. In Section~\ref{sub:update_rules_and_notations:error}, we give the definition of the error matrix and analyze its properties. In Section~\ref{sub:update_rules_and_notations:update_rule}, we focus on providing the definition for the update rule for $V$.

\subsection{Update Rule}
\label{sub:update_rules_and_notations:update}

To begin with, we formally define the updated rule as follows.

\begin{definition}[Update rule]\label{def:update_rule}
 
We define several updated rules here:

\begin{itemize}
        \item Part 1. We define the QR factorization $\wh{U}_t \in \R^{m \times k}$ as $\wh{U}_t 
        :=  U_tR_{t, U} $.
        \begin{itemize}
            \item Here $U_t \in \R^{m \times k}$ and $R_{t,U} \in \R^{k \times k}$
        \end{itemize}
        \item Part 2. We define $\wh{V}_{t + 1} \in \R^{n \times k}$ as follows $\wh{V}_{t + 1} 
        :=  \arg\min_{\wh{V} \in \R^{n \times k}} \|P_\Omega(U_t\wh{V}^\top) - P_\Omega(M)\|_F^2$.
        \item Part 3. We define the QR decomposition of $\wh{V}_{t+1} \in \R^{n \times k}$ as follows $\wh{V}_{t + 1} 
        :=  V_{t + 1} R_{t + 1, V}  $.
        \begin{itemize}
            \item Here $V_{t+1} \in \R^{n \times k}$ and $R_{t+1,V} \in \R^{k \times k}$
        \end{itemize}
        \item Part 4. We define $\wh{U}_{t+1} \in \R^{m \times k}$ as follows $
        \wh{U}_{t + 1} 
       : =  \arg\min_{\wh{U} \in \R^{m \times k}} \|P_\Omega(\wh{U}V_{t + 1}^\top) - P_\Omega(M)\|_F^2$.
\end{itemize}

\end{definition}

The update rule uses exact regression solutions, whereas our algorithm (Algorithm~\ref{alg:alternating_minimization_completion_formal}) computes them approximately, with the same QR steps. We start by analyzing this QR-based update. The proof of Theorem~\ref{thm:main} then connects it to our algorithm through Lemma~\ref{lem:induction_control_eps_sketch}.

\begin{definition}[$\mu_2$]\label{def:mu_2}
    We define $\mu_2$ as follows
    \begin{align*}
    \mu_2 := 40 \cdot \frac{ \sigma_1^*}{\sigma_k^*} \cdot \sqrt{k} \cdot \mu.
    \end{align*}
\end{definition}

\subsection{Error Matrix}
\label{sub:update_rules_and_notations:error}

Next, we provide several definitions related to the error matrix $F$. For most of the notations, we follow~\cite{jns13} (Note that those definitions are implicitly presented on page 16 in \cite{jns13}). Note that we use $U_*,V_*$ to denote the optimal low rank factor instead of $U^*,V^*$ to ease notations.
\begin{definition}[Error matrix]\label{def:error_matrix}
Let $p_t := p/(2T+1)$ be the per-round sampling rate, and let $\Omega := \Omega_{2t+1}$ be the sample used to compute $\wh{V}_{t+1}$. It contains every entry independently with probability $p_t$ and is independent of $U_t$. For each $j\in [n]$, we use $v_{*,j}$ to denote the $j$-th row of $V_* \in \R^{n \times k}$, and $\Omega_{*,j}$ is as in Definition~\ref{def:Omega_*_j}.

    We use $u_{t,i}$ to denote the $i$-th row of $U_t \in \R^{m \times k}$.

    We use $u_{*,i}$ to denote the $i$-th row of $U_* \in \R^{m \times k}$.
    
    \begin{itemize}
        \item For each $j \in [n]$, we define $B_j \in \R^{k \times k}$ as follows
    \begin{align*}
        B_j:=\frac{1}{p_t} \sum_{i \in \Omega_{*,j}} u_{t,i} u_{t,i}^\top .
    \end{align*}
        \item For each $j \in [n]$, we define $C_j \in \R^{k \times k}$ as follows
        \begin{align*}
            C_j := \frac{1}{p_t} \sum_{i \in \Omega_{*,j}} u_{t,i} u_{*,i}^\top 
        \end{align*}
        \item For each $j \in [n]$, we define $D_j \in \R^{k \times k}$ as follows
        \begin{align*}
           D_{j} := U_t^\top U_*
        \end{align*} 
        \item We write $B, C, D \in \R^{nk \times nk}$ for the block diagonal matrices whose $j$-th diagonal blocks are $B_j$, $C_j$ and $D_j$, respectively.
    \end{itemize}
    
    For each $j \in [n]$, we define $f_j \in \R^k$ as follows
    \begin{align*}
        f_j := B_j^{-1}(B_j D_j - C_j) \Sigma_* v_{*,j},
    \end{align*}
    and we define $F \in \R^{n \times k}$ to be the matrix whose $j$-th row is $f_j^\top$.
\end{definition}

The least-squares problem defining $\wh{V}_{t+1}$ (Part 2 of Definition~\ref{def:update_rule}) separates over the rows of $\wh{V}_{t+1}$. Since $M_{i,j} = u_{*,i}^\top \Sigma_* v_{*,j}$, the normal equations for the $j$-th row give $B_j \wh{V}_{t+1,j} = C_j \Sigma_* v_{*,j}$ (the factor $1/p_t$ cancels), so when $B_j$ is invertible,
\begin{align*}
    \wh{V}_{t+1,j} = B_j^{-1} C_j \Sigma_* v_{*,j} = D_j \Sigma_* v_{*,j} - B_j^{-1}(B_j D_j - C_j)\Sigma_* v_{*,j} = D_j \Sigma_* v_{*,j} - f_j .
\end{align*}
As $D_j \Sigma_* v_{*,j}$ is the $j$-th row of $V_* \Sigma_* U_*^\top U_t$, written as a column vector, this gives the formula for $\wh{V}_{t+1}$ in Definition~\ref{def:update_of_V_wh_V}.

\subsection{Update Rule for \texorpdfstring{$V$}{}}
\label{sub:update_rules_and_notations:update_rule}

Now, we define the update rule for $V$.

\begin{definition}\label{def:update_of_V_wh_V}

We provide the definition of $\wh{V}_{t + 1} \in \R^{n \times k}$ and $V_{t + 1} \in \R^{n \times k}$.

\begin{itemize}
    \item We define $\wh{V}_{t+1} \in \R^{n \times k}$ as follows
    \begin{align*}
        \wh{V}_{t + 1} := V_* \Sigma_* U_{*}^{\top} U_t - F
    \end{align*}
    Note that the first term can be treated as power-method update, and the second term is the error term. Here $F \in \R^{n \times k}$ is the error matrix defined in Definition~\ref{def:error_matrix}.
    \item We define $V_{t+1} \in \R^{n \times k}$ as follows
    \begin{align*}
        V_{t + 1} := \wh{V}_{t + 1} R_{t + 1}^{-1}.
    \end{align*}
    Here $R_{t+1} \in \R^{k \times k}$ is a upper-triangular matrix obtained using QR-decomposition of $\wh{V}_{t + 1} \in \R^{n \times k}$.
\end{itemize}
The above two definitions imply that
\begin{align*}
V_{t+1} = (  V_* \Sigma_* U_{*}^{\top} U_t - F ) R_{t+1}^{-1}.
\end{align*}
\end{definition}

\section{Technical Lemmas for Distance Shrinkage}
\label{sec:distance_shrinking}

In this section, we prove a collection of technical lemmas that will facilitate us to eventually conclude our induction.

In Section~\ref{sub:distance_shrinking:Sigma}, we bound $\| F \|$ by distance. In Section~\ref{sub:distance_shrinking:BDC}, we bound $( \sum_{j} \| (B_j D_j - C_j) \Sigma_* v_{*,j} \|_2^2 )^{1/2}$ by distance. In Section~\ref{sub:distance_shrinking_and_incoherent_preserving:Sigma}, we upper bound $\| \Sigma^* R_{t+1}^{-1} \|$. In Section~\ref{sub:incoherent_preserving:B^-1}, we upper bound $\| B^{-1} \|$. In Section~\ref{sub:incoherent_preserving:Cj}, we upper bound $\| C_j \|$.

\subsection{Upper Bounding \texorpdfstring{$\| F \|$}{} by Distance}\label{sub:distance_shrinking:Sigma}

In this section, we upper bound $\| F \|$ by distance between $U_t$ and $U_*$. We generally follow the approach of~\cite{jns13}.

\begin{lemma} 
\label{lem:norm_of_F_Sigma^*^-1}
     Let $F$ be the error matrix defined by Definition~\ref{def:error_matrix}. Let $U_t$ be a $\mu_2$-incoherent matrix with orthonormal columns and $\dist(U_t, U_*) \leq 1/2$. Let $M$ be a $(\mu,k)$-incoherent matrix (see Definition~\ref{def:mu_k_incoherent}). Let $p \in (0,1)$ be defined as Definition~\ref{def:sampling_probability}, and let $\Omega$ contain every entry independently with probability $p_t = p/(2T+1)$, independently of $U_t$. Let $\delta_{2k} \in (0,0.01]$ be defined as Definition~\ref{def:delta_2k}.
     
     Then, we have 
     \begin{align*}
         \|F\| \leq \|F\|_F \leq 2 \delta_{2k} \cdot k \cdot \sigma_1^* \cdot \dist(U_t, U_*)
     \end{align*}
     holds with probability least $1-1/\poly(n)$.
\end{lemma}
\begin{proof}
By Definition~\ref{def:error_matrix}, the $j$-th row of $F$ is $f_j^\top$ with $f_j = B_j^{-1}(B_jD_j - C_j)\Sigma_* v_{*,j}$. Hence
\begin{align*}
\| F \|_F^2 
= & ~ \sum_{j=1}^n \| f_j \|_2^2 \\
\leq & ~ \max_{j \in [n]} \| B_j^{-1} \|^2 \cdot \sum_{j=1}^n \| (B_jD_j - C_j)\Sigma_* v_{*,j} \|_2^2 \\
\leq & ~ 4 \cdot ( \delta_{2k} \cdot k \cdot \sigma_1^* \cdot \dist(U_t,U_*) )^2,
\end{align*}
where the second step follows from $\| A x \|_2 \leq \| A \| \cdot \| x \|_2$, and the third step follows from Lemma~\ref{lem:bound_B^-1} and Lemma~\ref{lem:bound_BD_minus_C_dot_v^*}. Finally, $\|F\| \leq \|F\|_F$ (Fact~\ref{fac:norm_inequality}).
\end{proof}

\subsection{Upper Bounding \texorpdfstring{$( \sum_{j} \| (B_j D_j - C_j) \Sigma_* v_{*,j} \|_2^2 )^{1/2}$}{the BD-C Term} by Distance}

\label{sub:distance_shrinking:BDC}

Now, we upper bound $( \sum_{j} \| (B_j D_j - C_j) \Sigma_* v_{*,j} \|_2^2 )^{1/2}$. We follow similar ideas in literature \cite{jns13}, but we use the matrix Bernstein inequality~\cite{t12} in place of Lemma~\ref{lem:modified_of_lem6.1}.

\begin{lemma} 
\label{lem:bound_BD_minus_C_dot_v^*}
    Let $M \in \R^{m \times n}$ be a $(\mu,k)$-incoherent matrix (see Definition~\ref{def:mu_k_incoherent}). Let $p \in (0,1)$ be defined as Definition~\ref{def:sampling_probability}, and let $\Omega \subset [m] \times [n]$ contain every entry independently with probability $p_t = p/(2T+1)$. Let $U_t \in \R^{m \times k}$ be a $\mu_2$ incoherent matrix with orthonormal columns that is fixed and independent of $\Omega$. Let $B_j, C_j, D_j$ be as in Definition~\ref{def:error_matrix}. For each $j \in [n]$, let $v_{*,j}$ denote the $j$-th row of $V_* \in \R^{n \times k}$.

   Then, we have:
   \begin{align*}
       \Big( \sum_{j=1}^n \| (B_j D_j - C_j) \Sigma_* v_{*,j} \|_2^2 \Big)^{1/2} \leq \delta_{2k} \cdot k \cdot \sigma_1^* \cdot \dist(U_t, U_*),
   \end{align*}
   holds with probability at least $1 -1/\poly(n)$.
\end{lemma}
\begin{proof}
    Write $d := \dist(U_t,U_*)$ and $G := (U_tU_t^\top - I)U_* \in \R^{m \times k}$ with $i$-th row $g_i^\top$. Then $\|G\| = d$, $\|g_i\|_2 \leq d$ and $\|G\|_F \leq \sqrt{k}\, d$. For $j\in [n], i \in [m] $, we define $H_{j,i} \in \R^{k \times k}$ as follows
    \begin{align}\label{eq:h_ji}
        H_{j,i} := u_{t,i} u_{t,i}^\top U_t^\top U_* - u_{t,i} u_{*,i}^\top = u_{t,i} g_i^\top .
    \end{align}
    By the definition of $B_j$, $C_j$, and $D_j$ (see Definition~\ref{def:error_matrix}), $B_j D_j - C_j = \frac{1}{p_t} \sum_{i \in \Omega_{*,j}} H_{j,i}$. Moreover,
    \begin{align}\label{eq:Hij}
        \sum_{i=1}^m H_{j,i} = U_t^\top G = U_t^\top U_t U_t^\top U_* - U_t^\top U_* = 0.
    \end{align}
    Let $y_j := \Sigma_* v_{*,j}$, so that $\|y_j\|_2 \leq \sigma_1^* \mu \sqrt{k}/\sqrt{n}$ by the incoherence of $V_*$. Let $\delta_{i,j} := \mathbf{1}[(i,j) \in \Omega]$. By Eq.~\eqref{eq:Hij},
    \begin{align*}
        (B_jD_j - C_j)y_j = \sum_{i=1}^m \frac{\delta_{i,j}-p_t}{p_t} u_{t,i} (g_i^\top y_j).
    \end{align*}
    Stacking these vectors over $j \in [n]$ gives a vector $z \in \R^{nk}$ with $\|z\|_2^2 = \sum_{j=1}^n \| (B_j D_j - C_j) y_j \|_2^2$, and $z = \sum_{i,j} Z_{i,j}$, where 
    \begin{align*}
        Z_{i,j} := \frac{\delta_{i,j}-p_t}{p_t} (e_j \otimes u_{t,i}) (g_i^\top y_j) \in \R^{nk}
    \end{align*}
    are independent with mean zero, because $U_t$ does not depend on $\Omega$. By the incoherence of $U_t$ and of $V_*$,
    \begin{align*}
        \|Z_{i,j}\|_2 \leq L := \frac{\mu \mu_2 k \sigma_1^* d}{p_t \sqrt{mn}}
        \quad \text{and} \quad
        \sum_{i,j} \E [ \|Z_{i,j}\|_2^2 ] \leq \frac{\mu_2^2 k}{p_t m} \|G \Sigma_* V_*^\top\|_F^2 \leq \sigma^2 := \frac{\mu_2^2 k^2 (\sigma_1^*)^2 d^2}{p_t m}.
    \end{align*}
    The matrix Bernstein inequality~\cite[Theorem 1.6]{t12}, applied to $nk \times 1$ matrices, gives 
    \begin{align*}
        \Pr[\|z\|_2 \geq s] \leq (nk+1) \exp\Big(-\frac{s^2/2}{\sigma^2 + Ls/3}\Big).
    \end{align*}
    If $d = 0$, then $z = 0$ and there is nothing to prove. Otherwise, take $s := \delta_{2k} k \sigma_1^* d$. Then $s^2/\sigma^2 = \delta_{2k}^2 p_t m/\mu_2^2$, and $s/L = \delta_{2k} p_t \sqrt{mn}/(\mu \mu_2) \geq \delta_{2k}^2 p_t m/\mu_2^2$ because $\delta_{2k} < 1$, $\mu \leq \mu_2$ and $n \geq m$. So the exponent is at least $\delta_{2k}^2 p_t m/(4\mu_2^2)$. By Definitions~\ref{def:sampling_probability}, \ref{def:delta_2k} and~\ref{def:mu_2}, $\delta_{2k}^2 p m/\mu_2^2 = \frac{C}{1600} k^{1.5} \log n \log(k\|M\|_F/\epsilon)$, and since $2T+1 = O(\log(k\|M\|_F/\epsilon))$, we get $\delta_{2k}^2 p_t m/\mu_2^2 \geq 40 \log n$ for $C$ large enough. Hence $\|z\|_2 \leq s$ with probability at least $1-(nk+1)n^{-10}$.
\end{proof}

\subsection{ Upper Bounding \texorpdfstring{$\| \Sigma^* R_{t+1}^{-1} \|$}{} by Condition Number}
\label{sub:distance_shrinking_and_incoherent_preserving:Sigma}

To bound $\|\Sigma^* R_{t + 1}^{-1}\|$, we show that $\|\Sigma^* R_{t + 1}^{-1}\| \leq  2 \cdot \sigma_1^*/\sigma_k^*$. Similarly, we also bound $\frac{\sigma_1^*}{ \sigma_{\min} (R_{t+1}) }$ by showing $\frac{\sigma_1^*}{ \sigma_{\min} (R_{t+1}) } \leq 2 \cdot \sigma_1^*/ \sigma_k^*$.

\begin{lemma} 
\label{lem:norm_of_Sigma^*_R^-1}
    Let $R_{t + 1}$ be the upper-triangular matrix obtained by QR decomposition of $\wh{V}_{t + 1}$ (see Definition~\ref{def:update_of_V_wh_V}). Let $U_t$ be a $\mu_2$-incoherent matrix with orthonormal columns and $\dist(U_t,U_*) \leq 1/2$. Let $M$ be a $(\mu,k)$-incoherent matrix (see Definition~\ref{def:mu_k_incoherent}). Let $\Omega$ contain every entry independently with probability $p_t = p/(2T+1)$, independently of $U_t$, where $p$ is as in Definition~\ref{def:sampling_probability}. Let $\delta_{2k} \in (0,0.01]$ be defined as Definition~\ref{def:delta_2k}.

    Then, we have
    \begin{itemize}
        \item Part 1.
        \begin{align*}
        \|\Sigma^* R_{t + 1}^{-1}\| \leq  2 \cdot \sigma_1^*/\sigma_k^*.
    \end{align*}
    \item Part 2.
    \begin{align*}
        \frac{\sigma_1^*}{ \sigma_{\min} (R_{t+1}) } \leq 2 \cdot \sigma_1^*/ \sigma_k^*
    \end{align*}
    \end{itemize}
\end{lemma}
\begin{proof}

Note that
\begin{align*}
\| \Sigma^* R_{t+1}^{-1} \| \leq \frac{ \sigma_1^* }{ \sigma_{\min} (R_{t+1}) }
\end{align*}

We can show that
\begin{align*}
\min_{z: \| z \|_2 =1} \| V_* \Sigma_* U_*^\top U_t z \|_2^2 = & ~ \min_{z: \| z \|_2 =1} \| \Sigma_* U_*^\top U_t z \|_2^2 \\
= & ~ \sigma_{\min}^2 ( \Sigma_* U_*^\top U_t z  ) \\
\geq & ~ \sigma_{\min}^2 ( \Sigma_* )  \cdot \sigma_{\min}^2 ( U_*^\top U_t   )  \\
= & ~ (\sigma_k^*)^2 \cdot \sigma_{\min} ^2 (U_*^\top U_t ) \\
\geq & ~ (\sigma_k^*)^2 \cdot ( 1- \| U_{*,\bot}^\top U_t \|^2 ),
\end{align*}
where the first step follows from Part 2 of Fact~\ref{fac:U_shrink_spectral_norm} that $V$ is a matrix of SVD that has an orthonormal basis, the second step follows from definition of $\sigma_{\min}$, the third step follows from Fact~\ref{fac:norm_inequality}, the fourth step follows from $\sigma_{\min} ( \Sigma_* ) = \sigma_k^*$, and the last step follows from $\cos^2 \theta + \sin^2\theta=1$ (see Lemma~\ref{lem:sin^2_and_cos^2_is_1}).

Now, we have
\begin{align*}
\sigma_{\min} (R_{t+1})
= & ~ \min_{z: \| z \|_2 = 1} \| R_{t+1} z \|_2 \\
= & ~ \min_{z: \| z \|_2 = 1} \| V_{t+1} R_{t+1} z \|_2 \\
= & ~ \min_{z: \| z \|_2 = 1} \| V_* \Sigma_* U_*^\top U_t z -  F z \|_2 \\
\geq & ~ \min_{z: \| z \|_2 = 1} \| V_* \Sigma_* U_*^\top U_t z \|_2 - \|  F z \|_2 \\
\geq & ~ \min_{z: \| z \|_2 = 1} \| V_* \Sigma_* U_*^\top U_t z \|_2 - \|  F \| \\
\geq & ~ \sigma_k^* \cdot \sqrt{1- \| U_{*,\bot}^\top U_t \|_2^2} - \| F \| \\
= & ~ \sigma_k^* \cdot \sqrt{1- \dist(U_*,U_t)^2} - \| F \| \\
\geq & ~ \sigma_k^* \cdot \sqrt{1- \dist(U_*,U_t)^2} - 2 \sigma_1^* \delta_{2k} k \dist(U_t,U_*),
\end{align*}
where the first step follows from the definition of $\sigma_{\min}$, the second step follows from Part 2 of Fact~\ref{fac:U_shrink_spectral_norm}, the third step follows from Definition~\ref{def:update_of_V_wh_V}, the fourth step follows from triangle inequality, the fifth step follows from the definition of $\| \cdot \|$, the sixth step follows from the previous display, the seventh step follows from the definition of $\dist(U_*,U_t)$ (see Definition~\ref{def:dist}), and the last step follows from Lemma~\ref{lem:norm_of_F_Sigma^*^-1}.

We can obtain that
\begin{align*}
        \|\Sigma^* R_{t + 1}^{-1}\| \leq 
        \frac{\sigma_1^*/\sigma_k^*}{\sqrt{1 - \dist(U_t, U_*)^2 } - 2 \cdot (\sigma_1^*/\sigma_k^*) \delta_{2k}k\dist(U_t, U_*) }.
    \end{align*} 
For convenient, we define $x:=\dist(U_t,U_*)$.

We define $y := 2 (\sigma_1^*/\sigma_k^*) \delta_{2k} k$.

Then we now $x \in [0,1/2]$ and $y \in (0,0.1)$.

Then we obtain that
\begin{align*}
\| \Sigma_* R_{t+1}^{-1} \| \leq \frac{\sigma_1^*/\sigma_k^*}{ \sqrt{1-x^2} -  y \cdot x }.
\end{align*}
Using Part 1 of Fact~\ref{fac:x_vs_sqrt_1_minus_x^2}, we know that 
\begin{align*}
\sqrt{1-x^2} -  y \cdot x \geq 1/2.
\end{align*}

Thus, we have 
\begin{align*}
\|\Sigma_* R_{t + 1}^{-1}\| \leq 2 \sigma_1^*/\sigma_k^*.
\end{align*}
\end{proof}

\subsection{Upper Bounding \texorpdfstring{$\| B^{-1} \|$}{} by Constant}
\label{sub:incoherent_preserving:B^-1}

We analyze $\| B_j^{-1} \|$, for all $j \in [n]$, and $ \|B^{-1}\|$ and show both of them are bounded by 2.
We prove a variation of Lemma C.6 in \cite{jns13}.
\begin{lemma} 
\label{lem:bound_B^-1}
    Let $M$ be a $(\mu,k)$-incoherent matrix (see Definition~\ref{def:mu_k_incoherent}). Let $p$ be defined as Definition~\ref{def:sampling_probability}, and let $\Omega$ contain every entry independently with probability $p_t = p/(2T+1)$. For each $j \in [n]$, we define $B_j \in \R^{k \times k}$ as follows
    \begin{align*}
    B_j:=\frac{1}{p_t} \sum_{i \in \Omega_{*,j}} U_{t,i} U_{t,i}^\top
    \end{align*}
    
   Let $U_t$ be $\mu_2$ incoherent with orthonormal columns, and independent of $\Omega$. Then, we have:
   \begin{itemize}
        \item Part 1. For all $j \in [n]$, 
        \begin{align*}
            \| B_j^{-1} \| \leq 2
        \end{align*}
        \item Part 2. 
        \begin{align*}
       \|B^{-1}\| \leq 2 
   \end{align*}
        \item Part 3. For all $j \in [n]$, 
        \begin{align*}
            \| B_j \| \leq 1 + \delta_{2k} \leq 2
        \end{align*}
   \end{itemize}
   
   all three events succeed with a probability at least $1 - 1/\poly(n)$
\end{lemma}
\begin{proof}
 
    We have: 
    \begin{align*}
        \|B^{-1}\|_2 
        = & ~ \frac{1}{\sigma_{\min} (B)} \\
        = & ~ \frac{1}{\min_{x,\| x\|_2=1} \| B x \|_2 } \\
        = & ~ \frac{1}{\min_{x,\|x\|_2 = 1 } x^\top Bx},
    \end{align*}
    where $x\in \R^{nk}$ and the first step follows from Fact~\ref{fac:norm_inequality}, the second step follows from definition of $\sigma_{\min}$, and the last step follows from $B$ is a symmetric matrix. 
    
    Let $x = \vect{(X)}$, i.e., $x_p$ is the $p$-th column of $X$ and $x_j$ is the $j$-th row of $X$. Let $B_{j} \in \R^{k \times k}$. Let $x_j \in \R^k$. Now, $\forall x \in \R^{nk}$,
    \begin{align*}
        x^\top B x 
        = & ~ \sum_{j \in [n]} x_j^\top B_j x_j \\
        \geq & ~ \min_{j \in [n]} \sigma_{\min} (B_j),
    \end{align*}
    where the first step follows from $B$ is a diagonal block matrix, and the second step follows from the definition of $\sigma_{\min}$ for a symmetric matrix. 
    
    Results would follow using the bound on $\sigma_{\min}(B_j)$, $\forall j \in [n]$ that we show below
    
    \noindent
    {\bf Lower bound on $\sigma_{\min} (B_j)$}: 
    
    In the rest of this proof, we write $p$ for $p_t$ to lighten notation and set $\delta_{i,j} := \mathbf{1}[(i,j) \in \Omega]$.

    Consider any $w \in \R^k$ such that $\|w\|_2 = 1$. 
    
    We have:
    \begin{align*}
        Z 
        = & ~ w^\top B_j w \\
        = & ~ w^\top ( \frac{1}{p} \sum_{i \in \Omega_{*,j}} U_{t,i} U_{t,i}^\top ) w \\
        = & ~ \frac{1}{p} \sum_{i \in \Omega_{*,j} } \langle w, U_{t,i} \rangle^2\\
        = & ~ \frac{1}{p} \sum_{i\in [m]} \delta_{i,j} \cdot \langle w, U_{t,i} \rangle^2,
    \end{align*}
    where the first step follows from the definition of $Z$, the second step follows from the definition of $B$, the third step follows from the definition of inner product, and the last step follows from definition of $\delta_{i,j}$.
    
    Note that, 
    \begin{align}\label{eq:EZ_1}
        \E[Z] 
        = & ~ \E[ \frac{1}{p} \sum_{i\in [m]} \delta_{i,j} \cdot \langle w, U_{t,i} \rangle^2] \notag\\
        = & ~ \frac{1}{p}  \sum_{i\in [m]}  \E[ \delta_{i,j} ] \cdot \langle w, U_{t,i} \rangle^2 \notag\\
        = & ~ \frac{1}{p}  \sum_{i\in [m]} p \cdot \langle w, U_{t,i} \rangle^2 \notag\\
        = & ~ \sum_{i\in [m]} \langle w, U_{t,i} \rangle^2 \notag\\
        = & ~ \sum_{i \in [m]} w^\top U_{t,i} U_{t,i}^\top w \notag\\
        = & ~ w^\top U_t^\top U_t w \notag\\
        = & ~ w^\top w \notag\\
        = & ~ 1,
    \end{align}
    where the first step follows from the definition of $Z$, the second step follows from $U$ being orthogonal, the third step follows from $\E[ \delta_{i,j} ] = p$, the fourth step follows from simple algebra, the fifth step follows from the definition of the inner product, the sixth step follows from the fact that $U_t$ is a $m\times k$ matrix, the seventh step follows from $U_t U_t^\top = I_k$, and the last step follows from $\|w\|_2 = 1$.

    For variance, we have
    \begin{align}\label{eq:frac_mu12_k_mp}
     \E [Z^2]
     = & ~ \E[ (\frac{1}{p}\sum_{i\in [m]} \delta_{i,j} \cdot \langle w, U_{t,i} \rangle^2)^2] \notag\\
     = & ~ \E[ \frac{1}{p^2}(\sum_{i\in [m]} \delta_{i,j} \cdot \langle w, U_{t,i} \rangle^2)^2] \notag\\
     = & ~ \frac{1}{p^2} \E[ (\sum_{i\in [m]} \delta_{i,j} \cdot \langle w, U_{t,i} \rangle^2)^2] \notag\\
     = & ~ \frac{1}{p^2} ( \sum_{i \in [m]} \E[ \delta_{i,j}^2 ] (\langle w, U_{t,i} \rangle^2)^2 + \sum_{i_1 \neq i_2} \E[\delta_{i_1,j}] \langle w, U_{t,i_1} \rangle^2 \E[\delta_{i_2,j} ]\langle w, U_{t,i_2} \rangle^2 ) \notag\\
     = & ~ \frac{1}{p^2} ( \sum_{i \in [m]} p \langle w, U_{t,i} \rangle^4 + \sum_{i_1 \neq i_2} p^2 \langle w, U_{t,i_1} \rangle^2 \langle w, U_{t,i_2} \rangle^2 ) \notag\\
     = & ~ \frac{1}{p} ( \sum_{i \in [m]} \langle w, U_{t,i} \rangle^4 - p\sum_{i \in [m]} \langle w, U_{t,i} \rangle^4 + p\sum_{i \in [m]} \langle w, U_{t,i} \rangle^4 + p\sum_{i_1 \neq i_2} \langle w, U_{t,i_1} \rangle^2 \langle w, U_{t,i_2} \rangle^2 ) \notag\\
     = & ~ \frac{1}{p} \sum_{i \in [m]} \langle w, U_{t,i} \rangle^4 - \sum_{i \in [m]} \langle w, U_{t,i} \rangle^4 + \sum_{i \in [m]} \langle w, U_{t,i} \rangle^4 + \sum_{i_1 \neq i_2} \langle w, U_{t,i_1} \rangle^2 \langle w, U_{t,i_2} \rangle^2 \notag\\
     = & ~ (\frac{1}{p} - 1) \sum_{i \in [m]} \langle w, U_{t,i} \rangle^4 + (\sum_{i \in [m]} \langle w, U_{t,i} \rangle^2)^2 \notag\\
     = & ~ (\frac{1}{p} - 1) \sum_{i \in [m]} \langle w, U_{t,i} \rangle^4 + (\E[Z])^2 \notag\\
     \leq & ~ \frac{1}{p} \sum_{i=1}^m \langle w, U_{t,i} \rangle^4 + (\E[Z])^2  \notag\\
     = & ~ \frac{1}{p}\frac{\mu_2^2 k }{m} + (\E[Z])^2  
     \end{align}
     where the first step follows from the definition of $Z$, the second step follows from $(ab)^2 = a^2b^2$, the third step follows from $\E[ca] = c\E[a]$, where $c$ is a constant, the fourth step follows from $\E[a + b] = \E[a] + \E[b]$, the fifth step follows from $\E[\delta_{i,j}^2] = p$, the sixth step follows from adding and subtracting the same thing, the seventh step follows from simple algebra, the eighth step follows from simple algebra, the ninth step follows from Eq.~\eqref{eq:EZ_1}, the tenth step follows from $(\frac{1}{p}-1)\leq \frac{1}{p}$,  
     the last step follows from Claim~\ref{cla:sum_Ux_4th_power}. 

     Then, we can compute
     \begin{align*}
        \var[Z] 
        = & ~ \E[Z^2] - (\E[Z])^2\\
        = & ~ \frac{\mu_2^2 k }{mp} + (\E[Z])^2 - (\E[Z])^2 \\
        = & ~ \frac{\mu_2^2 k }{mp},
     \end{align*}
     where the first step follows from the definition of variance, the second step follows from Eq.~\eqref{eq:frac_mu12_k_mp}, and the third step follows from simple algebra.

    Similarly, 
    \begin{align*}
    \max_{i \in [m]} |\langle w, U_{t,i} \rangle^2| \leq \frac{\mu_2^2 k }{m}.
    \end{align*}

    The computation above shows that $\E[B_j] = I_k$ and bounds the variance of $w^\top B_j w$ for one fixed $w$. A union bound over the infinitely many $w \in \R^k$ is not possible, so we apply the matrix Bernstein inequality~\cite[Theorem 1.6]{t12} to the sum
    \begin{align*}
        B_j - I_k = \sum_{i=1}^m \frac{\delta_{i,j}-p}{p} U_{t,i}U_{t,i}^\top
    \end{align*}
    of mean zero, symmetric $k \times k$ matrices, which are independent because $U_t$ does not depend on $\Omega$. Each summand has spectral norm at most $\|U_{t,i}\|_2^2/p \leq \mu_2^2 k/(pm)$, and
    \begin{align*}
        \sum_{i=1}^m \E\Big[ \Big(\frac{\delta_{i,j}-p}{p}\Big)^2 \Big] \|U_{t,i}\|_2^2 \, U_{t,i}U_{t,i}^\top \preceq \frac{\mu_2^2 k}{pm} U_t^\top U_t = \frac{\mu_2^2 k}{pm} I_k .
    \end{align*}
    Hence
    \begin{align*}
        \Pr[ \|B_j - I_k\| \geq \delta_{2k}] \leq 2k \exp\Big(-\frac{\delta_{2k}^2/2}{(1 + \delta_{2k}/3)\mu_2^2 k/(pm)}\Big) \leq 2k\exp\Big(-\frac{\delta_{2k}^2 pm}{3\mu_2^2 k}\Big).
    \end{align*}
    Recall that $p$ stands for $p_t = p/(2T+1)$ here, where $p$ in the formula of Definition~\ref{def:sampling_probability} is the total rate. By Definitions~\ref{def:sampling_probability}, \ref{def:delta_2k} and~\ref{def:mu_2}, and since $2T+1 = O(\log(k\|M\|_F/\epsilon))$, we have $\delta_{2k}^2 p_t m/(\mu_2^2 k) = \frac{C}{1600}k^{0.5}\log n\log(k\|M\|_F/\epsilon)/(2T+1) \geq 40\log n$ for $C$ large enough. Taking a union bound over $j \in [n]$, with probability at least $1 - 2kn^{-12}$, for all $j \in [n]$,
    \begin{align*}
        \| B_j - I_k \| \leq \delta_{2k}, \quad \text{hence} \quad \sigma_{\min}(B_j) \geq 1 - \delta_{2k} \quad \text{and} \quad \| B_j \| \leq 1 + \delta_{2k}.
    \end{align*}
     That is, $\forall j \in [n]$, 
     \begin{align*}
        \sigma_{\min}(B_j) \geq (1 - \delta_{2k}) \geq 0.5,
     \end{align*}
     which proves Part 1 and Part 3.
\end{proof}

\subsection{Upper Bounding \texorpdfstring{$\| C_j \|$}{} by \texorpdfstring{$1+\delta_{2k}$}{}}
\label{sub:incoherent_preserving:Cj}

Now, we analyze $\| C_j \|$ and show that it is bounded by $1+ \delta_{2k}$ for all $j \in [n]$. We prove a variation of Lemma C.7 in \cite{jns13}.

\begin{lemma} 
\label{lem:bound_C_j}

Let $M$ be a $(\mu,k)$-incoherent matrix (see Definition~\ref{def:mu_k_incoherent}). Let $p$ be defined as Definition~\ref{def:sampling_probability}, and let $\Omega$ contain every entry independently with probability $p_t = p/(2T+1)$. Let $U_t$ be $\mu_2$ incoherent with orthonormal columns, and independent of $\Omega$. For each $j \in [n]$, we define $C_j \in \R^{k \times k}$ as follows
\begin{align*}
C_j := \frac{1}{p_t} \sum_{i \in \Omega_{*,j}} U_{t,i} U_{*,i}^\top
\end{align*}
Then, with probability at least $1 - 1/\poly(n)$, we have: for all $j \in [n]$
\begin{align*}
\| C_j \| \leq 1+ \delta_{2k}.
\end{align*}
\end{lemma}
\begin{proof}
In this proof, we write $p$ for $p_t$ to lighten notation and set $\delta_{i,j} := \mathbf{1}[(i,j) \in \Omega]$.

Let $x \in \R^k$ and $y \in \R^k$ be two arbitrary unit vectors. We define $Z := x^\top C_j y$. Then, 
\begin{align*}
Z 
= & ~ x^\top C_j y \\ 
= & ~ \frac{1}{p} \sum_{i \in \Omega_{*,j} } x^\top U_{t,i} \cdot y^\top U_{*,i} \\
= & ~ \frac{1}{p} \sum_{i=1}^m \delta_{i,j} x^\top U_{t,i} \cdot y^\top U_{*,i},
\end{align*}
where the first step follows from the definition of $Z$, the second step follows from the the notation $\Omega_{*,j} \subset [m]$ being defined in Definition~\ref{def:Omega_*_j}, and the third step follows from definition of $\delta_{i,j}$ which is $1$ if $i \in \Omega_{*,j}$ and $0$ if $i \notin \Omega_{*,j}$.

Note that,
\begin{align}\label{eq:xtop_Uttop_U*y}
\E[Z] 
= & ~ \E[\frac{1}{p} \sum_{i=1}^m \delta_{i,j} x^\top U_{t,i} \cdot y^\top U_{*,i}] \notag\\
= & ~ \frac{1}{p} \sum_{i=1}^m \E[\delta_{i,j}] x^\top U_{t,i} \cdot y^\top U_{*,i} \notag\\
= & ~ \frac{1}{p} \sum_{i=1}^m p x^\top U_{t,i} \cdot y^\top U_{*,i} \notag\\
= & ~ \sum_{i=1}^m x^\top U_{t,i} \cdot y^\top U_{*,i} \notag\\
= & ~ x^\top U_t^\top U_* y,
\end{align}
where the first step follows from the definition of $Z$, the second step follows from the linearity of expectation ($\E[a + b] = \E[a] + \E[b]$), the third step follows from $\E[\delta_{i,j}] = p$, the fourth step follows from simple algebra, and the last step follows from the fact that $U$ is a $m \times k$ matrix.

We can compute the second moment
\begin{align}\label{eq:mu2k_mp}
\E[Z^2] 
= & ~ \E[(\frac{1}{p} \sum_{i=1}^m \delta_{i,j} x^\top U_{t,i} \cdot y^\top U_{*,i})^2] \notag\\
= & ~ \frac{1}{p^2} \E[(\sum_{i=1}^m \delta_{i,j} x^\top U_{t,i} \cdot y^\top U_{*,i})^2] \notag\\
= & ~ \frac{1}{p^2} ( \sum_{i = 1}^m \E[ \delta_{i,j}^2 ] (x^\top U_{t,i} \cdot y^\top U_{*,i})^2+ \sum_{i_1 \neq i_2} \E[\delta_{i_1,j}] x^\top U_{t,i_1} \cdot y^\top U_{*,i_1} \E[\delta_{i_2,j} ]x^\top U_{t,i_2} \cdot y^\top U_{*,i_2}) \notag\\
= & ~ \frac{1}{p^2} ( \sum_{i = 1}^m p (x^\top U_{t,i} \cdot y^\top U_{*,i})^2+ \sum_{i_1 \neq i_2} p^2 x^\top U_{t,i_1} \cdot y^\top U_{*,i_1} x^\top U_{t,i_2} \cdot y^\top U_{*,i_2}) \notag\\
= & ~ \frac{1}{p} ( \sum_{i = 1}^m (x^\top U_{t,i} \cdot y^\top U_{*,i})^2+ p\sum_{i_1 \neq i_2} x^\top U_{t,i_1} \cdot y^\top U_{*,i_1} x^\top U_{t,i_2} \cdot y^\top U_{*,i_2}) \notag\\
= & ~ \frac{1}{p} ( \sum_{i = 1}^m (x^\top U_{t,i} \cdot y^\top U_{*,i})^2 - p\sum_{i = 1}^m (x^\top U_{t,i} \cdot y^\top U_{*,i})^2 + p\sum_{i = 1}^m (x^\top U_{t,i} \cdot y^\top U_{*,i})^2 \notag\\
+ & ~ p\sum_{i_1 \neq i_2} x^\top U_{t,i_1} \cdot y^\top U_{*,i_1} x^\top U_{t,i_2} \cdot y^\top U_{*,i_2}) \notag\\
= & ~ (\frac{1}{p} - 1) \sum_{i = 1}^m (x^\top U_{t,i} \cdot y^\top U_{*,i})^2 + (\sum_{i = 1}^m x^\top U_{t,i} \cdot y^\top U_{*,i})^2  \notag\\
\leq & ~ \frac{1}{p} \sum_{i=1}^m ( x^\top U_{t,i} )^2 \cdot (y^\top U_{*,i})^2 + (\sum_{i = 1}^m x^\top U_{t,i} \cdot y^\top U_{*,i})^2 \notag\\
= & ~ \frac{1}{p} \sum_{i=1}^m ( x^\top U_{t,i} )^2 \cdot (y^\top U_{*,i})^2 + (\E[Z])^2 \notag\\
= & ~ \frac{\mu^2 k}{ mp} + (\E[Z])^2,
\end{align}
where the first step follows from the definition of $Z$, the second step follows from $\E[ca] = c\E[a]$, where $c$ is a constant, the third step follows from $\E[a + b] = \E[a] + \E[b]$, the fourth step follows from $\E[\delta_{i,j}^2] = p$, the fifth step follows from simple algebra, the sixth step follows from adding and subtracting the same thing, the seventh step follows from simple algebra, the eighth step follows from $(xy)^2 = x^2 y^2$, the ninth step follows Eq.~\eqref{eq:xtop_Uttop_U*y}, the 10th step follows from Claim~\ref{cla:sum_Ux_4th_power}.

Then, we have
\begin{align*}
\var[Z] 
= & ~ \E[Z^2] - (\E[Z])^2\\
= & ~ \frac{\mu^2 k}{ mp} + (\E[Z])^2 - (\E[Z])^2\\
= & ~ \frac{\mu^2 k}{ mp},
\end{align*}
where the first step follows from the definition of variance, the second step follows from Eq.~\eqref{eq:mu2k_mp}, and the last step follows from simple algebra.

We have
\begin{align*}
\max_{i \in [m]} | x^\top U_{t,i} \cdot y^\top U_{*,i} | \leq \frac{\mu \mu_2 k}{m} \leq \frac{\mu_2^2 k}{m}
\end{align*}

The computation above is for fixed unit vectors $x, y$. A union bound over all of them is not possible, so instead we apply the matrix Bernstein inequality~\cite[Theorem 1.6]{t12} to the sum
\begin{align*}
C_j - U_t^\top U_* = \sum_{i=1}^m \frac{\delta_{i,j}-p}{p} U_{t,i} U_{*,i}^\top
\end{align*}
of mean zero $k \times k$ matrices, which are independent because $U_t$ does not depend on $\Omega$. Each summand has spectral norm at most $\|U_{t,i}\|_2 \|U_{*,i}\|_2/p \leq \mu\mu_2 k/(pm) \leq \mu_2^2 k/(pm)$. Moreover $\| \sum_i \E[(\frac{\delta_{i,j}-p}{p})^2] \|U_{*,i}\|_2^2 U_{t,i}U_{t,i}^\top \| \leq \frac{\mu^2 k}{pm}$ and $\| \sum_i \E[(\frac{\delta_{i,j}-p}{p})^2] \|U_{t,i}\|_2^2 U_{*,i}U_{*,i}^\top \| \leq \frac{\mu_2^2 k}{pm}$, so the variance parameter is at most $\mu_2^2 k/(pm)$. As in the proof of Lemma~\ref{lem:bound_B^-1}, this gives
\begin{align*}
\Pr[\|C_j - U_t^\top U_*\| \geq \delta_{2k}] \leq 2k\exp\Big(-\frac{\delta_{2k}^2 pm}{3\mu_2^2 k}\Big) \leq 2k n^{-13},
\end{align*}
and a union bound over $j \in [n]$ shows that, with probability at least $1 - 2kn^{-12}$, for all $j \in [n]$,
\begin{align*}
\|C_j\| \leq \|U_t^\top U_*\| + \|C_j - U_t^\top U_*\| \leq 1+\delta_{2k}.
\end{align*}
\end{proof}

%% file: stronger_core.tex
\section{A Perturbation Theory for Distances and Incoherence}
\label{sec:purterbation}

In this section, we present one of the main technical contributions of this paper --- a perturbation theory for distances and matrix incoherence. Our main technology is a reduction from incoherence to leverage score, then utilize different parametrizations of leverage score to obtain a bound on incoherence.

In Section~\ref{sub:purterbation:distance_by_spectral}, show to bound distance by the spectral norm. 
In Section~\ref{sub:purterbation:Z_X}, we present the perturbation related lemma that is from right to left low rank factor. In Section~\ref{sub:purterbation:leverage_score}, we assert a toll from the previous work: leverage score changes under row perturbation. In Section~\ref{sub:purterbation:X_Y}, we elucidate the perturbation related lemma that is from approximate to exact updates.

\subsection{Bound Distance by Spectral Norm}\label{sub:purterbation:distance_by_spectral}

The following lemma establishes a relationship between $\dist$ and the spectral norm.
\begin{lemma}[Bounded distance by spectral]\label{lem:bound_distance_by_spectral}
    Let $X, Y\in \R^{n\times k}$ be two matrices (not necessarily orthogonal).
    Suppose $\|X-Y\|^2 \le \sigma_{\min}(X) \sigma_{\min}(Y)$.
    Then
    \begin{align*}
        \dist(X, Y) \le 4 \|X-Y\| \sigma_{\min}(X)^{-0.5} \sigma_{\min}(Y)^{-0.5}.
    \end{align*}
\end{lemma}
\begin{proof}
    \begin{align}\label{eq:bound_cos}
        \cos \theta(X, Y) &= \min_{u\in \mathrm{span}(X)} \max_{v\in \mathrm{span}(Y)} \frac{u^\top v}{\|u\| \|v\|} \notag\\
        & \ge \min_{\substack{w\in \R^k \\ u=Xw, v=Yw}} \frac{u^\top v}{\|u\| \|v\|} \notag\\
        & = \min_{\substack{w\in \R^k \\ u=Xw, v=Yw}} \frac 12 \cdot \frac{\|u\|^2 + \|v\|^2 - \|u-v\|^2}{\|u\| \|v\|} \notag\\
        & \ge \min_{\substack{w\in \R^k \\ u=Xw, v=Yw}} 1- \frac 12 \cdot \frac{\|u-v\|^2}{\|u\| \|v\|} \notag\\
        & \ge \min_{\substack{w\in \R^k}} 1- \frac 12 \cdot \frac{\|X-Y\|^2 \|w\|^2}{\| u \| \| v \| } \notag\\
        & \ge \min_{\substack{w\in \R^k}} 1- \frac 12 \cdot \frac{\|X-Y\|^2 \|w\|^2}{\sigma_{\min}(X) \sigma_{\min}(Y) \|w\|^2} \notag\\
        & =  1- \frac 12 \|X-Y\|^2 \sigma_{\min}(X)^{-1} \sigma_{\min}(Y)^{-1},
    \end{align}
where the first step follows from the definition of $\cos \theta (X, Y)$, the second step follows from $\max_i X_i \geq X_i$, the third step follows from simple algebra, the fourth step follows from $a^2+b^2 \geq 2ab$, the fifth step follows from $\| u - v \|_2^2 = \| X w -  Y w \|_2^2 \leq \| X - Y \|^2 \| w \|_2^2$, the sixth step follows from $\|u \|_2 = \| X w \|_2 \geq \sigma_{\min}(X) \| w \|_2$, and the last step follows from canceling the term $\| w \|_2^2$.

    Thus, 
    \begin{align}\label{eq:bound_sin}
       \sin^2 \theta(X, Y) 
       = & ~ 1-\cos^2 \theta(X, Y) \notag\\
       \leq & ~ 1-\left(1- \frac 12 \|X-Y\|^2 \sigma_{\min}(X)^{-1} \sigma_{\min}(Y)^{-1}\right)^2 \notag\\
       \leq & ~ 1-\left(1- \|X-Y\|^2 \sigma_{\min}(X)^{-1} \sigma_{\min}(Y)^{-1}\right) \notag\\
       = & ~ \|X-Y\|^2 \sigma_{\min}(X)^{-1} \sigma_{\min}(Y)^{-1},
    \end{align}
    where the first step follows from $\cos^2 \theta (X, Y) + \sin^2 \theta (X, Y) = 1$ (see Lemma~\ref{lem:sin^2_and_cos^2_is_1}) and the second steep follows from Eq.~\eqref{eq:bound_cos}, the third step follows from $-(a - b)^2  = -a^2 - b^2 + 2ab \leq -a^2 + 2ab$, and the fourth step follows from simple algebra.

Note that $\dist(X,Y) = \sin \theta(X,Y)$, thus, we complete the proof.
\end{proof}

\subsection{From Right to Left Factor}
\label{sub:purterbation:Z_X}

Given an orthonormal basis $Z$ and a target matrix $M$, let $X := MZ$. We prove some relations on singular values of $X$ and $M$ given some conditions on $Z$.

\begin{lemma}
Let $M = U_* \Sigma_* V_*^\top$.

Suppose the matrix $Z \in \R^{n \times k}$ satisfies that
\begin{itemize}
    \item $Z \in \R^{n \times k}$ is an orthonormal basis
    \item $\dist(Z, V_*) \le \sqrt{1-\alpha^2}$
\end{itemize}
Let $X \in \R^{m \times k}$ be defined by 
\begin{align}\label{eq:define_M}
X := M Z. 
\end{align}
Then, we have
\begin{align*}
    \sigma_{\min}(X) \geq \alpha \cdot \sigma_{\min} (M)
\end{align*}
and
\begin{align*}
\sigma_{\max} (X) \leq \sigma_{\max}(M)
\end{align*}
\end{lemma}
\begin{proof}

We have
\begin{align}\label{eq:VZ}
\|V_{*,\bot}^\top Z\|
= & ~ \dist(Z, V_*)  \notag\\
\leq & \sqrt{1-\alpha^2}. 
\end{align}
where the first step follows from the definition of distance and the second step follows from the Lemma statement.

Using $\sin^{2} \theta + \cos^2 \theta = 1$ (Lemma~\ref{lem:sin^2_and_cos^2_is_1}), we have
\begin{align}\label{eq:VZ_is_1}
\sigma_{\min} (V_*^\top Z)^2 + \| V_{*,\bot}^\top Z \|^2 = 1.
\end{align}

We can show
\begin{align}\label{eq:sigma_min_V_*_Z}
\sigma_{\min}(V_*^\top Z) 
= & ~ \sqrt{1 - \| V_{*,\bot}^\top Z \|^2 } \notag \\
\geq & ~ \sqrt{1 - (1-\alpha^2) } \notag \\
= & ~ \alpha,
\end{align}
where the first step follows from Eq.~\eqref{eq:VZ_is_1}, the second step follows from Eq.~\eqref{eq:VZ}, and the third step follows from simple algebra.

We have
\begin{align*}
\sigma_{\min}(X) 
= & ~ \sigma_{\min} ( M \cdot (Z^\top)^\dagger ) \\
= & ~ \sigma_{\min} ( U_* \Sigma_* V_*^\top \cdot (Z^\top)^\dagger )  \\
= & ~ \sigma_{\min} ( \Sigma_* V_*^\top \cdot (Z^\top)^\dagger )  \\
= & ~ \sigma_{\min} ( \Sigma_* V_*^\top \cdot Z ) \\
\geq & ~ \sigma_{\min} ( \Sigma_*) \cdot \sigma_{\min} ( V_*^\top \cdot Z ) \\
= & ~ \sigma_{\min} (M) \cdot \sigma_{\min} ( V_*^\top \cdot Z )  \\
\geq & ~ \sigma_{\min} (M) \cdot \alpha,
\end{align*}
where the first step follows from Eq.~\eqref{eq:define_M}, the second step follows from $M= U_* \Sigma_* V_*^\top$, the third step follows from $U_*$ has orthonormal columns, the fourth step follows from $(Z^\top)^\dagger = Z$, the fifth step follows from part 11 of Fact~\ref{fac:norm_inequality}, the sixth step follows from definition of matrix $M \in \R^{m \times n}$, and the last step follows from Eq.~\eqref{eq:sigma_min_V_*_Z}.

Finally, $\sigma_{\max}(X) = \|MZ\| \leq \|M\| \cdot \|Z\| = \sigma_{\max}(M)$, since $Z$ has orthonormal columns.
\end{proof}

\subsection{Leverage Score Change Under Row Perturbations} 
\label{sub:purterbation:leverage_score}

We analyze the change to leverage scores when the rows are perturbed in a structured manner.

\begin{lemma}\label{lem:purturb_leverage_score}
Given two matrices $A \in \R^{m \times k}$ and $B \in \R^{m \times k}$. For each $i \in [m]$, let $a_i$ denote the $i$-th row of matrix $A$. For each $i \in [m]$, let $b_i$ denote the $i$-th row of matrix $B$.

If 
\begin{itemize} 
    \item $\| A - B \| \leq \epsilon_0$
    \item $\epsilon_0 \leq \frac{1}{2}\sigma_{\min}(A)$
\end{itemize}
then we have 
\begin{align*}
    | \|  (A^\top A)^{-1/2} a_i \|_2 - \| (B^\top B)^{-1/2} b_i \|_2 |  \leq 75 \epsilon_0 \sigma_{\min}(A)^{-1} \kappa^4(A).
\end{align*}
\end{lemma}
We do not divide by $\|(A^\top A)^{-1/2} a_i\|_2+\|(B^\top B)^{-1/2} b_i\|_2$: a lower bound on this sum through $\|a_i\|_2+\|b_i\|_2\geq 2\sigma_{\min}(A)-\epsilon_0$ fails when the $i$-th row of $A$ is short or zero.
\begin{proof}
The quantity $\|(A^\top A)^{-1/2} a_i\|_2$ is the square root of the $i$-th leverage score of $A$. It is the length of the orthogonal projection of the $i$-th standard basis vector onto the column space of $A$, and hence it depends on $A$ only through that space. So instead of comparing the two expressions algebraically, we compare the two column spaces. When $B$ is close to $A$ and $A$ is far from singular, the two column spaces are close. The projected length of a fixed unit vector can then change by at most the spectral-norm distance between the two projections. This argument gives a bound with no dependence on $\kappa(A)$.

\textbf{Both matrices have full column rank.} The statement uses $(A^\top A)^{-1/2}$ and $\sigma_{\min}(A)^{-1}$, so $A$ has rank $k$ (in particular $m \geq k$) and $\sigma_{\min}(A) = \sigma_k(A) > 0$. We have
\begin{align}\label{eq:lev_sigma_min_B}
    \sigma_{\min}(B)
    = & ~ \sigma_k(B) \notag \\
    \geq & ~ \sigma_k(A) - \|A - B\| \notag \\
    \geq & ~ \sigma_{\min}(A) - \epsilon_0 \notag \\
    \geq & ~ \frac{1}{2} \sigma_{\min}(A),
\end{align}
where the first step holds because $B$ has $k$ columns and $m \geq k$, the second step follows from Weyl's inequality (Lemma~\ref{lem:weyl}), the third step follows from $\sigma_k(A) = \sigma_{\min}(A)$ and $\|A - B\| \leq \epsilon_0$, and the last step follows from $\epsilon_0 \leq \frac{1}{2}\sigma_{\min}(A)$. Hence $B$ also has rank $k$, and $(B^\top B)^{-1/2}$ is well defined.

\textbf{Leverage scores as projected lengths.} Define
\begin{align*}
    P_A := A (A^\top A)^{-1} A^\top \in \R^{m \times m}
    \quad \text{and} \quad
    P_B := B (B^\top B)^{-1} B^\top \in \R^{m \times m}.
\end{align*}
Both matrices are symmetric, and $P_A^2 = A (A^\top A)^{-1} (A^\top A) (A^\top A)^{-1} A^\top = P_A$. Likewise $P_B^2 = P_B$. So $P_A$, $P_B$, $I_m - P_A$ and $I_m - P_B$ are orthogonal projections, and each has spectral norm at most $1$. Moreover $(I_m - P_A) A = A - A = 0$ and $(I_m - P_B) B = 0$.

Let $e_i \in \R^m$ denote the $i$-th standard basis vector, so that $a_i = A^\top e_i$. Then
\begin{align*}
    \|(A^\top A)^{-1/2} a_i\|_2^2
    = & ~ a_i^\top (A^\top A)^{-1} a_i \\
    = & ~ e_i^\top A (A^\top A)^{-1} A^\top e_i \\
    = & ~ e_i^\top P_A e_i \\
    = & ~ e_i^\top P_A^\top P_A e_i \\
    = & ~ \|P_A e_i\|_2^2,
\end{align*}
where the first step follows from the fact that $(A^\top A)^{-1/2}$ is symmetric and its square is $(A^\top A)^{-1}$, the second step follows from $a_i = A^\top e_i$, the third step follows from the definition of $P_A$, the fourth step follows from $P_A = P_A^2 = P_A^\top P_A$, and the last step follows from the definition of the $\ell_2$ norm. Both sides are nonnegative, so $\|(A^\top A)^{-1/2} a_i\|_2 = \|P_A e_i\|_2$. The same computation gives $\|(B^\top B)^{-1/2} b_i\|_2 = \|P_B e_i\|_2$. Therefore
\begin{align}\label{eq:lev_reverse_triangle}
    | \|(A^\top A)^{-1/2} a_i\|_2 - \|(B^\top B)^{-1/2} b_i\|_2 |
    = & ~ | \|P_A e_i\|_2 - \|P_B e_i\|_2 | \notag \\
    \leq & ~ \|(P_A - P_B) e_i\|_2 \notag \\
    \leq & ~ \|P_A - P_B\|,
\end{align}
where the second step follows from the reverse triangle inequality, and the last step follows from $\|e_i\|_2 = 1$.

\textbf{The two projections are close.} The two cross terms $P_A P_B$ cancel in
\begin{align*}
    P_A - P_B = P_A (I_m - P_B) - (I_m - P_A) P_B.
\end{align*}
Since $(I_m - P_A) A = 0$, we can replace $B$ by $B - A$ in the second term:
\begin{align*}
    (I_m - P_A) P_B
    = & ~ (I_m - P_A) B (B^\top B)^{-1} B^\top \\
    = & ~ (I_m - P_A) (B - A) (B^\top B)^{-1} B^\top.
\end{align*}
Let $B = U_B \Sigma_B V_B^\top$ be a thin singular value decomposition, where $U_B \in \R^{m \times k}$ has orthonormal columns, $V_B \in \R^{k \times k}$ is orthogonal and $\Sigma_B \in \R^{k \times k}$ is diagonal and invertible. Then $(B^\top B)^{-1} B^\top = V_B \Sigma_B^{-2} V_B^\top V_B \Sigma_B U_B^\top = V_B \Sigma_B^{-1} U_B^\top$, whose spectral norm is $\sigma_{\min}(B)^{-1}$. Using $\|I_m - P_A\| \leq 1$ and $\|B - A\| \leq \epsilon_0$, we obtain
\begin{align*}
    \|(I_m - P_A) P_B\|
    \leq & ~ \|I_m - P_A\| \cdot \|B - A\| \cdot \|(B^\top B)^{-1} B^\top\| \\
    \leq & ~ \epsilon_0 \sigma_{\min}(B)^{-1}.
\end{align*}
The same argument with the roles of $A$ and $B$ exchanged, now using $(I_m - P_B) B = 0$, gives $\|(I_m - P_B) P_A\| \leq \epsilon_0 \sigma_{\min}(A)^{-1}$. The matrix $P_A (I_m - P_B)$ is the transpose of $(I_m - P_B) P_A$, so the two have the same spectral norm. Therefore
\begin{align}\label{eq:lev_projection_gap}
    \|P_A - P_B\|
    \leq & ~ \|P_A (I_m - P_B)\| + \|(I_m - P_A) P_B\| \notag \\
    \leq & ~ \epsilon_0 \sigma_{\min}(A)^{-1} + \epsilon_0 \sigma_{\min}(B)^{-1} \notag \\
    \leq & ~ 3 \epsilon_0 \sigma_{\min}(A)^{-1},
\end{align}
where the first step follows from the triangle inequality, the second step follows from the two bounds above, and the last step follows from $\sigma_{\min}(B) \geq \frac{1}{2}\sigma_{\min}(A)$ (Eq.~\eqref{eq:lev_sigma_min_B}).

\textbf{Put things together.} We have
\begin{align*}
    | \|(A^\top A)^{-1/2} a_i\|_2 - \|(B^\top B)^{-1/2} b_i\|_2 |
    \leq & ~ 3 \epsilon_0 \sigma_{\min}(A)^{-1} \\
    \leq & ~ 75 \epsilon_0 \sigma_{\min}(A)^{-1} \kappa^4(A),
\end{align*}
where the first step follows from Eq.~\eqref{eq:lev_reverse_triangle} and Eq.~\eqref{eq:lev_projection_gap}, and the second step follows from $\kappa(A) = \sigma_1(A) / \sigma_k(A) \geq 1$. This completes the proof. The first bound is stronger than the stated one, since it does not depend on the condition number of $A$.
\end{proof}

\subsection{From Approximate to Exact Distance and Incoherence}
\label{sub:purterbation:X_Y}

Given a matrix $X$ that is incoherent, we show that if a matrix $Y$ is close to $X$ enough in spectral norm, then it also preserves the distance and incoherence of $Y$.

\begin{lemma}
\label{lem:equiv_leverage}
Let $X\in \R^{m\times k}$ and $X=U\Sigma V^\top$ be its SVD. Then
\begin{align*}
    \|u_i\|_2^2 =  x_i^\top (X^\top X)^{-1} x_i.
\end{align*}

where $u_i, x_i$ denote the $i$-th row of $U$ and $X$ respectively.
\end{lemma}

\begin{proof}
Let us form the projection matrix $X(X^\top X)^{-1}X^\top$, note that the quantity on the RHS is the $i$-th diagonal of the projection matrix.
\begin{align*}
    X(X^\top X)^{-1}X^\top = & ~ U\Sigma V^\top (V\Sigma^2 V^\top)^{-1} V\Sigma U^\top \\
    = & ~ U\Sigma V^\top V\Sigma^{-2} V^\top V\Sigma U^\top \\
    = & ~ UU^\top,
\end{align*}
where the first step follows from the Lemma statement: $X=U\Sigma V^\top$, the second step follows from simple algebra, the third step follows from simple algebra.

The $i$-th diagonal of $UU^\top$ is then $u_i^\top u_i=\|u_i\|_2^2$, as desired.
\end{proof}

\begin{lemma}\label{lem:induction_control_eps_sketch}
Let $X, Y\in \R^{m \times k}$. Let $\mu_0 \geq 1$. Let $\epsilon_0 \in (0,1)$.

Suppose that
\begin{itemize}
    \item $X$ is $\mu_0$ incoherent, i.e., $\| (U_{X})_i \|_2 \leq \mu_0 \sqrt{k}/\sqrt{m}$.
    \item Let $ \| Y - X \| \leq \epsilon_0$. 
    \item Let $\epsilon_0 \leq 0.1 \cdot  \sigma_{\min}(X)$ 
\end{itemize}

For each $i\in [m]$, let $X_i$ denote the $i$-th row of matrix $X \in \R^{m \times k}$. Let $\kappa(X):=\sigma_{\max}(X) / \sigma_{\min}(X)$ be the condition number of $X$. Let $\epsilon_{\mathrm{sk}} := \Theta(\epsilon_0 \cdot  (m \cdot \kappa^4(X) \cdot \sigma_{\min}(X)^{-1} ))$.

Then, we have
\begin{itemize}
    \item Part 1. $\dist(Y, U_*) \leq \dist(X,U^*) +\epsilon_{\mathrm{sk}}$. 
    \item Part 2. $Y$ is $\mu_0 + \epsilon_{\mathrm{sk}}$ incoherent, i.e., $\| (U_Y)_i \|_2 \leq (\mu_0+ \epsilon_{\mathrm{sk}}  ) \sqrt{k} / \sqrt{m}$.
\end{itemize}
\end{lemma}
\begin{remark}
In the proof of Theorem~\ref{thm:main}, we use this lemma once per half-step, with $X$ the exact regression solution and $Y$ the approximate one. There $\kappa(X) \leq 2.2\kappa(M)$ and $\sigma_{\min}(X) \geq \sigma_k^*/2$, where $M$ is the ground truth matrix in matrix completion problem. There we choose $\epsilon_0 = \epsilon/(C' n^3\kappa^5)$. Since the distances contract by a factor $1/4$ in each half-step, the additive errors $\epsilon_{\mathrm{sk}}$ do not accumulate. After $s$ half-steps, the distance is at most $4^{-s}/2 + \frac{4}{3}\epsilon_{\mathrm{sk}}$. Note that our running time blowup is only logarithmically in $1/\epsilon_0$, thus it's always fine to set $\epsilon_0$ to be $1/\poly(n,\kappa)$ smaller. Similarly for part 2, the exact solutions are $(\mu_2/2)$-incoherent and $\epsilon_{\mathrm{sk}} \leq \mu_2/2$, so over all the iterations, the incoherence parameter stays at most $\mu_2$.
\end{remark}
\begin{proof}

We prove two parts separately.

{\bf Proof of Part 1.} From lemma assumptions, we have
\begin{align*}
\| Y - X \| \leq \epsilon_0 \leq 0.1 \sigma_{\min}(X)
\end{align*}
Then, we know that
\begin{align}\label{eq:sigma_min_Y_at_least_sigma_min_X}
\sigma_{\min}(Y) \geq 0.5 \sigma_{\min}(X).
\end{align}

We have
\begin{align*}
\dist(Y,U_*)
\leq & ~ \dist(X,U_*) + \dist(Y, X) \\
\leq & ~ \dist(X,U_*) + 4\| Y - X \| \cdot \sigma_{\min}(X)^{-0.5} \cdot \sigma_{\min}(Y)^{-0.5} \\
\leq & ~ \dist(X,U_*) + 4\epsilon_0 \cdot \sigma_{\min}(X)^{-0.5} \sigma_{\min}(Y)^{-0.5}  \\
\leq & ~  \dist(X,U_*) + 6 \epsilon_0 \cdot \sigma_{\min}(X)^{-1} \\
\leq & ~ \dist(X,U_*) + \epsilon_{\rm sk},
\end{align*}
where the first step follows from triangle inequality, the second step follows from Lemma~\ref{lem:bound_distance_by_spectral} (its hypothesis holds since $\|X - Y\|^2 \leq \epsilon_0^2 \leq 0.01 \sigma_{\min}(X)^2 \leq \sigma_{\min}(X)\sigma_{\min}(Y)$), the third step follows from lemma statement $\| Y - X \| \leq \epsilon_0$, the fourth step follows from Eq.~\eqref{eq:sigma_min_Y_at_least_sigma_min_X} and $4\sqrt{2} \leq 6$, the last step follows from definition $\epsilon_{\mathrm{sk}}$ in the lemma statement.

{\bf Proof of Part 2.} By Lemma~\ref{lem:purturb_leverage_score} and Lemma~\ref{lem:equiv_leverage}, we know that for any $i\in [m]$
\begin{align*}
    |\|(U_X)_i\|_2 - \|(U_Y)_i\|_2| \leq & ~ 500 \cdot \epsilon_0 \cdot \sigma_{\min}(X)^{-1} \cdot \kappa^4(X).
\end{align*}

Since $X \in \R^{m \times k}$ is $\mu_0$ incoherent, by triangle inequality, we have
\begin{align*}
    \|(U_Y)_i\|_2 \leq & ~ \mu_0\sqrt{k}/\sqrt{m}+500 \cdot \epsilon_0 \cdot \sigma_{\min}(X)^{-1} \cdot \kappa^4(X) \\
    \leq & ~ (\mu_0 + \epsilon_{\mathrm{sk}}) \cdot \sqrt{k} / \sqrt{m}
\end{align*}
where the last step follows from definition of $\epsilon_{\mathrm{sk}}$ in the lemma statement.

Thus, we complete the proof.
\end{proof}